\documentclass[10pt]{article}

\usepackage[utf8]{inputenc} 
\usepackage[T1]{fontenc}

\usepackage{url}          
\usepackage{booktabs}      
\usepackage{amsfonts}       
\usepackage{nicefrac}       
\usepackage{microtype}      
\usepackage{fullpage}
\usepackage{mathtools}
\usepackage{xcolor}
\newcommand*{\colorboxed}{}
\def\colorboxed#1#{
	\colorboxedAux{#1}
}
\newcommand*{\colorboxedAux}[3]{
	\begingroup
	\colorlet{cb@saved}{.}
	\color#1{#2}
	\boxed{
		\color{cb@saved}
		#3
	}
	\endgroup
}

\usepackage{amssymb}
\usepackage{amsmath,amsthm}

\usepackage{enumitem}

\usepackage{algorithm}
\usepackage{algorithmic}

\usepackage{xcolor,colortbl}
\definecolor{DarkBlue}{RGB}{22,54,93}

\usepackage[colorlinks = true]{hyperref}        
\usepackage{cleveref}
\hypersetup{linkcolor =red, citecolor = blue}

\usepackage{natbib}
\usepackage{graphicx}
\usepackage{subfigure}
\usepackage{adjustbox}
\usepackage{amsmath}
\usepackage{amssymb}
\usepackage{mathtools}
\usepackage{amsthm}
\usepackage{threeparttable}

\theoremstyle{plain}
\newtheorem{theorem}{Theorem}[section]

\newtheorem{lemma}[theorem]{Lemma}
\newtheorem{corollary}[theorem]{Corollary}
\theoremstyle{definition}
\newtheorem{definition}[theorem]{Definition}
\newtheorem{assumption}[theorem]{Assumption}
\theoremstyle{remark}

\usepackage[textsize=tiny]{todonotes}

\allowdisplaybreaks[4]

\Crefname{assumption}{Assumption}{Assumptions}

\newcommand{\Yc}{\mathcal{Y}}
\newcommand{\Xc}{\mathcal{X}}

\newcommand{\Sc}{\mathcal{S}}
\newcommand{\Ac}{\mathcal{A}}

\newcommand{\Dc}{\mathcal{D}}
\newcommand{\Uc}{\mathcal{U}}
\newcommand{\Mc}{\mathcal{M}}

\newcommand{\Eb}{\mathbb{E}}
\newcommand{\Rb}{\mathbb{R}}
\newcommand{\Jb}{\mathbb{J}}
\newcommand{\Ib}{\mathbb{I}}
\newcommand{\Pb}{\mathbb{P}}

\newcommand{\norm}[1]{\left\lVert#1\right\rVert}
\newcommand{\RM}[1]{\left(\romannumeral#1\right)}

\newcommand{\sphi}{\phi^\star}

\newcommand{\hb}{\hat{b}}
\newcommand{\hP}{\hat{P}}
\newcommand{\hphi}{\hat{\phi}}
\newcommand{\hmu}{\hat{\mu}}
\newcommand{\hQ}{\hat{Q}}
\newcommand{\hU}{\hat{U}}
\newcommand{\hV}{\hat{V}}
\newcommand{\pp}{{p^\prime}}
\newcommand{\ps}{p^\star}
\newcommand{\sP}{P^\star}
\newcommand{\ph}{{h^\prime}}

\newcommand{\ts}{\tilde{s}}
\newcommand{\ta}{\tilde{a}}
\newcommand{\tDc}{\tilde{\Dc}}
\newcommand{\tpi}{\tilde{\pi}}

\newcommand{\tU}{\tilde{U}}
\newcommand{\tW}{\tilde{W}}

\newcommand{\tap}{\tilde{\alpha}}
\newcommand{\Jc}{\mathcal{J}}
\newcommand{\Ic}{\mathcal{I}}
\newcommand{\Hc}{\mathcal{H}}

\title{Provably Efficient Algorithm for Nonstationary Low-Rank MDPs}

\author
{
	Yuan Cheng\thanks{\small 
Integrative Sciences and Engineering Programme, NUS Graduate School, Singapore; e-mail: {\tt   yuan.cheng@u.nus.edu}}
	,~~~Jing Yang\thanks{\small School of Electrical Engineering and Computer Science, The Pennsylvania State University, University Park, PA 16802, USA; e-mail: {\tt  yangjing@psu.edu }}
	,~~~Yingbin Liang\thanks{\small Department of Electrical and Computer Engineering, The Ohio State University, OH 43210, USA; e-mail: {\tt   liang.889@osu.edu}}
}

\begin{document}

\maketitle

\begin{abstract}
  Reinforcement learning (RL) under changing environment models many real-world applications via nonstationary Markov Decision Processes (MDPs), and hence gains considerable interest. However, theoretical studies on nonstationary MDPs in the literature have mainly focused on tabular and linear (mixture) MDPs, which do not capture the nature of unknown representation in deep RL. In this paper, we make the first effort to investigate nonstationary RL under episodic low-rank MDPs, where both transition kernels and rewards may vary over time, and the low-rank model contains unknown representation in addition to the linear state embedding function. We first propose a parameter-dependent policy optimization algorithm called PORTAL,
and further improve PORTAL to its parameter-free version of Ada-PORTAL, which is able to tune its hyper-parameters adaptively without any prior knowledge of nonstationarity. 
For both algorithms, we provide upper bounds on the average dynamic suboptimality gap, which show that as long as the nonstationarity is not significantly large, PORTAL and Ada-PORTAL are sample-efficient and can achieve arbitrarily small average dynamic suboptimality gap with polynomial sample complexity.
\end{abstract}

\section{Introduction}
Reinforcement learning (RL) has gained significant success in real-world applications such as board games of Go and chess~\citep{silver2016mastering,silver2017mastering,silver2018general}, robotics~\citep{levine2016end,gu2017deep}, recommendation systems~\citep{DBLP:conf/aaai/ZhaoGZYLTL21} and autonomous driving~\citep{bojarski2016end,DBLP:conf/icra/Ma0KIF21}. Most theoretical studies on RL have been focused on a stationary environment and evaluated the performance of an algorithm by comparing against only one best fixed policy (i.e., \textit{static regret}). However, in practice, the environment is typically time-varying and {\em nonstationary}. As a result,  
the transition dynamics, rewards and consequently the optimal policy change over time.
	
There has been a line of research studies that investigated {\em nonstationary} RL. Specifically, \citet{DBLP:journals/corr/abs-1805-10066,DBLP:conf/icml/CheungSZ20,DBLP:conf/icml/MaoZZSB21} studied nonstationary tabular MDPs. To further overcome the curse of dimensionality, \citet{DBLP:conf/nips/FeiYWX20,DBLP:journals/corr/abs-2010-04244} proposed algorithms for nonstationary linear (mixture) MDPs and established upper bounds on the dynamic regret. 

In this paper, we significantly advance this line of research by investigating nonstationary RL under {\em low-rank} MDPs \citep{DBLP:conf/nips/AgarwalKKS20}, where the transition kernel of each MDP admits a decomposition into a representation function and a state-embedding function that map to low dimensional spaces. 
Compared with linear MDPs where the representation is known, the low-rank MDP model contains unknown representation, and is hence much more powerful to capture representation learning that occurs often in deep RL. Although there have been several recent studies on {static} low-rank MDPs~\citep{DBLP:conf/nips/AgarwalKKS20,DBLP:conf/iclr/UeharaZS22,modi2021model}, nonstationary low-rank MDPs remain unexlored, and are the focus of this paper.

To investigate nonstationary low-rank MDPs, several challenges arise. (a) All previous studies of nonstationary MDPs took on-policy exploration, such a strategy will have difficulty in providing sufficiently accurate model (as well as representation) learning for nonstationary low-rank MDPs. (b) Under low-rank MDPs, since both representation and state-embedding function change over time, it is more challenging to use history data collected under previous transition kernels for current use.



The \textbf{main contribution} of this paper lies in addressing above challenges and designing a provably efficient algorithm for nonstationary low-rank MDPs. We summarize our contributions as follows.
 
\begin{list}{$\bullet$}{\topsep=0.ex \leftmargin=0.15in \rightmargin=0.1in \itemsep =0.01in}
   
\item We propose a novel policy optimization algorithm with representation learning called PORTAL for nonstationary low-rank MDPs. PORTAL features new components, including off-policy exploration, data-transfer model learning, and target policy update with periodic restart. 
   
\item We theoretically characterize the average dynamic suboptimality gap ($\mathrm{Gap_{Ave}}$) of PORTAL, where $\mathrm{Gap_{Ave}}$ serves as a new metric that captures the performance of target policies with respect to the best policies at each instance in the nonstationary MDPs under off-policy exploration. We further show that with prior knowledge on the degree of nonstationarity, PORTAL can select hyper-parameters that minimize $\mathrm{Gap_{Ave}}$. If the nonstationarity is not significantly large, PORTAL enjoys a diminishing $\mathrm{Gap_{Ave}}$ with respect to the number of iterations $K$, indicating that PORTAL can achieve arbitrarily small $\mathrm{Gap_{Ave}}$ with polynomial sample complexity. 

Our analysis features a few new developments. (a) We provide a new MLE guarantee under nonstationary transition kernels that captures errors of using history data collected under different transition kernels for benefiting current model estimation. (b) We establish trajectory-wise uncertainty bound for estimation errors via a square-root $\ell_\infty$-norm of variation budgets. (c) We develop an error tracking technique via auxiliary anchor representation for convergence analysis.

\item Finally, we improve PORTAL to a {\em parameter-free} algorithm called Ada-PORTAL, which does not require prior knowledge on nonstationarity and is able to tune the hyper-parameters adaptively. We further characterize $\mathrm{Gap_{Ave}}$ of Ada-PORTAL as $\tilde{O}(K^{-\frac{1}{6}}(\Delta+1)^{\frac{1}{6}})$, where $\Delta$ captures the variation of the environment. Notably, based on PORTAL, we can also use the black-box method called MASTER in \citet{DBLP:conf/colt/WeiL21} to turn PORTAL into a parameter-free algorithm (called MASTER+PORTAL) with $\mathrm{Gap_{Ave}}$ of $\tilde{O}(K^{-\frac{1}{6}}\Delta^{\frac{1}{3}})$. Clearly, Ada-PORTAL performs better than MASTER+PORTAL when nonstationarity is not significantly small, i.e. $\Delta \geq \tilde{O}(1)$.  
\end{list}
To our best knowledge, this is the first study of nonstationary RL under low-rank MDPs.
	\section{Related Works}
	 Various works have studied nonstationary RL under tabular and linear MDPs, most of which can be divided into two lines: policy optimization methods and value-based methods.  
	
    \textbf{Nonstationary RL: Policy Optimization Methods.} 
    As a vast body of existing literature~\citep{DBLP:conf/icml/CaiYJW20,DBLP:conf/icml/ShaniE0M20,DBLP:conf/nips/AgarwalHKS20,DBLP:conf/icml/XuLL21} has proposed policy optimization methods attaining computational efficiency and sample efficiency simultaneously in {\em stationary} RL under various scenarios, only several papers investigated policy optimization algorithm in {\em nonstationary} environment. Assuming time-varying rewards and time-invariant transition kernels, \citet{DBLP:conf/nips/FeiYWX20} studied nonstationary RL under tabular MDPs. \citet{DBLP:journals/corr/abs-2110-08984} assumed both transition kernels and rewards change over episodes and studied nonstationary linear mixture MDPs. These policy optimization methods all assumed prior knowledge on nonstationarity.    
    
 \textbf{Nonstationary RL: Value-based Methods.} Assuming that both transition kernels and rewards are time-varying, several works have studied nonstationary RL under tabular and linear MDPs, most of which adopted Upper Confidence Bound (UCB) based algrithms.   \citet{DBLP:conf/icml/CheungSZ20} investigated tabular MDPs with infinite-horizon and proposed algorithms with both known variation budgets and unknown variation budgets. In addition, this work also proposed a Bandit-over-Reinforcement Learning (BORL) technique to deal with unknown variation budgets. \citet{DBLP:conf/icml/MaoZZSB21} proposed a model-free algorithm with sublinear dynamic regret bound. They then proved a lower bound for nonstationary tabular MDPs and showed that their regret bound is near min-max optimal.   \citet{DBLP:journals/corr/abs-2010-12870,DBLP:journals/corr/abs-2010-04244} considered nonstationary RL in linear MDPs and proposed algorithms achieving sublinear regret bounds with unknown variation budgets. 
  
Besides these two lines of researches, \citet{DBLP:conf/colt/WeiL21} proposed a black-box method that turns a RL algorithm with optimal regret in a (near-)stationary environment into another algorithm that can work in a nonstationary environment with sublinear dynamic regret without prior knowledge on nonstationarity. In this paper, we show that our algorithm Ada-PORTAL outperforms such a type of black-box method (taking PORTAL as subroutine) if nonstationarity is not significantly small. 

 
 
\textbf{Stationary RL under Low-rank MDPs.} Low-rank MDPs were first studied by \citet{DBLP:conf/nips/AgarwalKKS20} in a reward-free regime, and then \citet{DBLP:conf/iclr/UeharaZS22} studied low-rank MDPs for both online and offline RL with known rewards. \citet{cheng2023improved} studied reward-free RL under low-rank MDPs and improved the sample complexity of previous works. \citet{modi2021model} proposed a model-free algorithm MOFFLE under low-nonnegative-rank MDPs.
\citet{chengprovable,agarwal2022provable} studied multitask representation learning under low-rank MDPs, and further showed the benefit of representation learning to downstream RL tasks.


	
\section{Formulation}
\textbf{Notations:} We use $[K]$ to denote set $\{1,\dots,K\}$ for any $K\in \mathbb{N}$, use $\norm{x}_2$ to denote the $\ell_2$ norm of vector $x$,  use $\triangle(\Ac)$ to denote the probability simplex over set $\Ac$, use $\mathcal{U}(\Ac)$ to denote uniform sampling over $\Ac$, given $|\Ac|<\infty$, and use $\triangle(\Sc)$ to denote the set of all possible density distributions over set $\Sc$. Furthermore, for any symmetric positive definite matrix $\Sigma$, we let $\norm{x}_\Sigma : = \sqrt{x^\top \Sigma x}$. For distributions $p_1$ and $p_2$, we use $D_{KL}(p_1(\cdot)\|p_2(\cdot))$ to denote the KL divergence between $p_1$ and $p_2$. 

\subsection{Episodic MDPs and Low-rank Approximation}

An episodic MDP is denoted by a tuple $\Mc:=\left(\Sc,\Ac,H,P:=\{P_h\}_{h=1}^H,r:=\{r_h\}_{h=1}^H\right)$, where $\Sc$ is a possibly infinite state space, $\Ac$ is a finite action space with cardinality $A$, $H$ is the time horizon of each episode, $P_h(\cdot|\cdot,\cdot): \Sc \times \Ac \rightarrow \Delta(\Sc)$ denotes the transition kernel at each step $h$, and $r_h(\cdot,\cdot): \Sc \times \Ac \rightarrow [0,1]$ denotes the deterministic reward function at each step $h$. We further normalize the reward as $\sum_{h=1}^Hr_h \leq 1$.
A policy $\pi=\{\pi_h\}_{h \in [H]}$ is a set of mappings where $\pi_h: \Sc \rightarrow \Delta(\Ac)$. For any $(s,a)\in \Sc \times \Ac$, $\pi_h(a|s)$ denotes the probability of selecting action $a$ at state $s$ at step $h$. For any $(s,a)\in\Sc\times\Ac$, let $(s_h,a_h) \sim (P, \pi)$ denote that the state $s_h$ is sampled by executing policy $\pi$ to step $h$ under transition kernel $P$ and then action $a_h$ is sampled by $\pi_h(\cdot|s_h)$. 

Given any state $s \in \Sc$, the value function for a policy $\pi$ at step $h$ under an MDP $\Mc$ is defined as the expected value of the accumulative rewards as: $V_{h,P,r}^\pi(s)=\sum_{\ph=h}^H\Eb_{ (s_{\ph},a_\ph)\sim\left(P,\pi\right)}\left[r_\ph(s_{\ph},a_\ph)|s_h=s\right]$. Similarly, given any state-action pair $(s,a) \in \Sc\times\Ac$, the action-value function ($Q$-function) for a policy $\pi$ at step $h$ under an MDP $\Mc$ is defined as $Q_{h,P,r}^\pi(s,a)=r_h(s,a)+\sum_{\ph=h+1}^H\Eb_{ (s_{\ph},a_\ph)\sim\left(P,\pi\right)}\left[r_\ph(s_{\ph},a_\ph)|s_h=s,a_h=a\right]$.
Denote $(P_h f)(s,a):=\Eb_{s^\prime \sim P_h(\cdot|s,a)}[f(s^\prime)]$ for any function $f: \Sc \rightarrow \Rb$. Then we can write the action-value function as $Q_{h,P,r}^\pi(s,a)=r_h(s,a)+(P_hV^\pi_{h+1,P,r})(s,a)$. For any $k \in [K]$, without loss of generality, we assume the initial state $s_1$ to be fixed and identical, and we use $V_{P,r}^\pi$ to denote $V_{1,P,r}^\pi(s_1)$ for simplicity.

This paper focuses on low-rank MDPs \citep{jiang2017contextual,DBLP:conf/nips/AgarwalKKS20} defined as follows.
\begin{definition}[Low-rank MDPs]\label{definition: Low_rank}
A transition kernel $P_h^*: \Sc \times \Ac \rightarrow \triangle(\Sc)$ admits a low-rank decomposition with dimension $d \in \mathbb{N}$ if there exist a representation function $\phi_h^\star: \Sc \times \Ac \rightarrow \Rb^d$ and a state-embedding function $\mu_h^\star: \Sc \rightarrow \Rb^d$ such that 
\begin{align*}
\textstyle P_h^\star(s^\prime|s,a)=\left\langle\phi_h^\star(s,a),\mu_h^\star(s^\prime)\right\rangle, \quad \forall s, s^\prime \in \Sc, a \in \Ac.
\end{align*}
Without loss of generality, we assume $\norm{\phi_h^*(s,a)}_2\leq 1$ for all $(s,a)\in\mathcal{S}\times\mathcal{A}$ and for any function $g: \Sc \mapsto [0,1]$, $\norm{\int \mu^\star_h(s)g(s)ds }_2\leq\sqrt{d}$. An MDP is a low-rank MDP with dimension $d$ if for any $h \in [H]$, its transition kernel $P_h^*$ admits a low-rank decomposition with dimension $d$. Let $\phi^\star=\{\phi^\star_h\}_{h \in [H]}$ and $\mu^\star=\{\mu^\star_h\}_{h \in [H]}$ be the true representation and state-embedding functions.
\end{definition}


\subsection{Nonstationary Transition Kernels with Adversarial Rewards}

In this paper, we consider an episodic RL setting under changing environment, where both transition kernels and rewards vary over time and possibly in an adversarial fashion.


Specifically, suppose the RL system goes by {\em rounds}, \textcolor{blue}{where each round have a fixed number of episodes}, and the transition kernel and the reward remain the same in each round, and can change adversarially across rounds.
For each round, say round $k$, we denote the MDP as $\Mc^k=(\Sc,\Ac,H,P^k:=\{P_h^{\star,k}\}_{h=1}^H,r^k:=\{r^k_h\}_{h=1}^H)$, where $P^{\star,k}$ and $r^k$ denote the true transition kernel and the reward of round $k$. Further, $P^{\star,k}$  takes the low-rank decomposition as $P^{\star,k}=\langle\phi^{\star,k},\mu^{\star,k}\rangle$. Both the representation function $\phi^{\star,k}$ and the state embedding function $\mu^{\star,k}$ can change across rounds. Given the reward function $r^k$, there always exists an optimal policy $\pi^{\star,k}$ that yields the optimal value function $V_{P^{\star,k},r^k}^{\pi^{\star,k}}=\sup_\pi V_{P^{\star,k},r^k}^\pi$, abbreviated as $V_{P^{\star,k},r^k}^\star$. Clearly, the optimal policy also changes across rounds.
         
We assume the agent interacts with the nonstationary environment (i.e., the time-varying MDPs) over $K$ rounds in total without the knowledge of transition kernels $\{P^{k,\star}\}_{k=1}^K$. At the beginning of each round $k$, the environment changes to a possibly adversarial transition kernel unknown to the agent, picks a reward function $r^k$, which is revealed to the agent only at the end of round $k$, and outputs a fixed initial state $s_1$ for the agent to start the exploration of the environment for each episode. 
The agent is allowed to interact with MDPs via a few episodes with one or multiple {\em exploration} policies at her choice to take samples from the environment and then should output an target policy to be executed during the next round. 
Note that in our setting, the agent needs to decide exploration and target policies only based on the information in previous rounds, and hence exploration samples and the reward information of the current round help only towards future rounds.

\subsection{Learning Goal and Evaluation Metric}\label{subsec: performance metric}
In our setting, the agent seeks to find the optimal policy at each round $k$ (with only the information of previous rounds), 
where both transition kernels and rewards can change over rounds. Hence we define the following notion of \textit{average dynamic suboptimality gap} to measure the convergence of the target policy series to the optimal policy series. 
\begin{definition}[Average Dynamic Suboptimality Gap]
For $K$ rounds, and any policy set $\{\pi^k\}_{k \in [K]}$, the average dynamic suboptimality gap $(\mathrm{Gap_{Ave}})$ of the value functions over $K$ rounds 
is given as $\mathrm{Gap_{Ave}}(K)= \frac{1}{K}\sum_{k=1}^{K}[V^{\star}_{P^{\star,k},r^k}-V^{\pi^k}_{P^{\star,k},r^k}]$.
For any $\epsilon$, we say an algorithm is $\epsilon$-average suboptimal, if it outputs a policy set $\{\pi^k\}_{k \in [K]}$ satisfying $\mathrm{Gap_{Ave}}(K) \leq \epsilon$.
\end{definition}
$\mathrm{Gap_{Ave}}$ compares the agent's target policy to the optimal policy of each individual round in hindsight, which captures the dynamic nature of the environment. This is in stark contrast to the stationary setting where the comparison policy is a single fixed best policy over all rounds. This notion is similar to \textit{dynamic regret} used for nonstationary RL~\citep{DBLP:conf/nips/FeiYWX20,DBLP:journals/corr/abs-1805-10066}, where the only difference is that $\mathrm{Gap_{Ave}}$ evaluates the performance of target policies rather than the exploration policies. Hence, given any target accuracy $\epsilon \geq 0$, the agent is further interested in the statistical efficiency of the algorithm, i.e., using as few trajectories as possible to achieve $\epsilon$-average suboptimal.  



	\section{Policy Optimization Algorithm and Theoretical Guarantee}
	\subsection{Base Algorithm: PORTAL}
We propose a novel algorithm called PORTAL (\Cref{Alg: DPO}), which features three main steps.
Below we first summarize our main design ideas and then explain reasons behind these ideas as we further elaborate main steps of PORTAL. 

{\bf Summary of New Design Ideas:} PORTAL features the following main design ideas beyond previous studies on nonstationary RL under tabular and linear MDPs. (a) PORTAL features a specially designed {\em off-policy} exploration which turns out to be beneficial for nonstationary low-rank /MDP models rather than the typical {\em on-policy} exploration taken by previous studies of nonstationary tabular and linear MDP models. (b) PORTAL transfers history data collected under various different transition kernels for benefiting the estimation of the current model. (c) PORTAL updates target policies with periodic restart. As a comparison, previous work using periodic restart~\citep{DBLP:conf/nips/FeiYWX20} chooses the restart period $\tau$ based on a certain smooth visitation assumption. Here, we remove such an assumption and hence our choice of $\tau$ is applicable to more general model classes. 

\textbf{Step 1. Off-Policy Exploration for Data Collection:} We take {\em off-policy} exploration, which is beneficial for nonstationary low-rank MDPs than simply using the target policy for {\em on-policy} exploration taken by the previous studies on nonstationary tabular or linear (mixture) MDPs~\citep{DBLP:journals/corr/abs-2110-08984,DBLP:conf/nips/FeiYWX20,DBLP:journals/corr/abs-2010-04244}. To further explain,
we first note that under tabular or linear (mixture) MDPs studied in the previous work, the bonus term is able to serve as a {\em point-wise} uncertainty level of the estimation error for each state-action pair at any step $h$, so that for any step $h$, $\hat{Q}_h^k$ is a good optimistic estimation for $Q^{\pi}_{h,P^{\star,k},r^k}$. Hence it suffices to collect samples using the target policy. However, in low-rank MDPs, the bonus term $\hat{b}^k_h$ cannot serve as a point-wise uncertainty measure. For step $h \geq 2$, $\hat{Q}_h^k$ is not a good optimistic estimation for the true value function if the agent only uses target policy to collect data (i.e., for on-policy exploration). Hence, more samples and a novel off-policy exploration are required for a good estimation under low-rank MDPs. Specifically, as line 5 in \Cref{Alg: DPO}, at the beginning of each round $k$, for each step $h \in [H]$, the agent explores the environment by executing the exploration policy $\tpi^{k-1}$ to state $\tilde{s}^{k,h}_{h-1}$ and then taking two uniformly chosen actions, where $\tpi^{k-1}$ is determined in Step 2 of the previous round. 
	\begin{algorithm}
		\begin{algorithmic}[1]
			\caption{\textbf{PORTAL} (\textbf{P}olicy \textbf{O}ptimization with \textbf{R}epresen\textbf{TA}tion \textbf{L}earning under nonstationary MDPs)} \label{Alg: DPO}
			\STATE {\bf Input:} Rounds $K$, hyper-parameters $\tau,W$, regularizer $\lambda_{k,W}$, coefficient $\tap_{k,W}$, stepsize $\eta$ and models $ \{\Psi,  \Phi \}$.
			\STATE {\bf Initialization:} $\pi_0(\cdot|s)$ to be uniform; $\tilde{\mathcal{D}}_h^{(0,0)}=\emptyset$.
			\FOR{episode $k=1,\ldots,K$}
			\FOR{step $h=1,\ldots,H$}
			\STATE  Roll into $\ts_{h-1}^{(k,h)}$ using $\tilde{\pi}^{k-1}$, uniformly choose $\ta_{h-1}^{(k,h)},\ta_{h}^{(k,h)}$, and enter into $\ts_{h}^{(k,h)},\ts_{h+1}^{(k,h)}$.
			\STATE Update datasets 
			\vspace{-0.1in}
			\begin{align*}
			\tDc_{h-1}^{(k,h,W)}=\left\{\ts_{h-1}^{(i,h)},\ta_{h-1}^{(i,h)},\ts_{h}^{(i,h)}\right\}_{i=1 \lor k-W+1}^{k},\tDc_{h}^{(k,h,W)}=\left\{\ts_{h}^{(i,h)},\ta_{h}^{(i,h)},\ts_{h+1}^{(i,h)}\right\}_{i=1 \lor k-W+1}^{k}.
			\end{align*}
			\ENDFOR
			\STATE Receive full information rewards $r^k=\{r_h^k\}_{h \in [H]}$.
			\STATE Estimate transition kernel and update the exploration policy $\tpi^k$ for the next round via: 
			\vspace{-0.1in}
			\begin{align*}
			\textstyle \mathrm{E^2U}\left(k,\{\tDc_{h-1}^{(k,h,W)}\},\{\tDc_h^{(k,h,W)}\}\right).
			\end{align*}
			\vspace{-0.25in}
      		\FOR{step $h=1,\ldots,H$}
			\STATE Update $\hQ^k_h = Q_{h,\hP^k,r^k}^{\pi^k}$.
			\ENDFOR
                \IF{$k \mod \tau =1$}
			\STATE Set $\{\hQ_h^{k}\}_{h\in[H]}$ as zero functions and $\{\pi_h^{k}\}_{h \in [H]}$ as uniform distributions on $\Ac$.
			\ENDIF
   			\FOR{step $h=1,\ldots,H$} 
			\STATE  Update the target policy as in \Cref{Eq: Alg-PolicyUpdate}. 
			\ENDFOR    
			\ENDFOR
			\STATE {\bf Output:} $\{\pi^k\}_{k=1}^K$.
		\end{algorithmic}
	\end{algorithm}
 
  \textbf{Step 2. Data-Transfer Model Learning and E$^2$U:} 

In this step, we transfer history data collected under previous different transition kernels for benefiting the estimation of the current model. This is theoretically grounded by our result that the model estimation error can be decomposed into variation budgets plus a diminishing term as the estimation sample size increases, which justifies that the usage of data generated by mismatched distributions
within a certain window is beneficial for model learning as long as variation budgets is mild. Then, the estimated model will further facilitate the selection of future exploration policies accurately.

Specifically, the agent selects desirable samples only from the latest $W$ rounds following a \textit{forgetting rule}~\citep{garivier2011upper}. Since nonstationary low-rank MDPs (compared to tabular and linear MDPs) also have additional variations on representations over time, the choice of $W$ needs to incorporate such additional information. Then the agent passes these selected samples to a subroutine E$^2$U (see \Cref{Alg: MEPE}), in which the agent estimates the transition kernels via the maximum likelihood estimation (MLE). Next, the agent updates the empirical covariance matrix $\hat{U}^{k,W}$ and exploration-driven bonus $\hat{b}^k$ as in lines 4 and 5 in \Cref{Alg: MEPE}. We then define a {\em truncated value function} iteratively using the estimated transition kernel and the exploration-driven reward as follows:
\begin{align}
	&\textstyle\hat{Q}_{h,\hP^k,\hat{b}^k}^{\pi}(s_h,a_h) = \min\left\{1, \hat{b}^k_h(s_h,a_h) + \hP_h^k\hV_{h+1,\hP^k,\hat{b}^k}^{\pi}(s_h,a_h)\right\},\nonumber\\
	&\textstyle\hat{V}_{h,\hP^k,\hat{b}^k}^{\pi}(s_h) =  \mathop{\Eb}_{\pi}\left[\hat{Q}^{\pi}_{h,\hP^k,\hat{b}^k}(s_h,a_h)\right]\label{ineq: def of hV}.
\end{align}
Although the bonus term $\hat{b}^k_h$ cannot serve as a \textit{point-wise} uncertainty measure, the truncated value function $\hat{V}^{\pi}_{\hat{P}^k,\hat{b}^k}$ as the cumulative version of $\hat{b}^k_h$ can serve as a \textit{trajectory-wise} uncertainty measure, which can be used to determine future exploration policies. Intuitively, for any policy $\pi$, the model estimation error satisfies $\Eb_{(s,a)\sim (P^{\star,k},\pi)}[\|\hat{P}^k(\cdot|s,a)-P^{\star,k}(\cdot|s,a)\|_{TV}]\leq \hV_{\hP^k,\hb^k}^{\pi}+\Delta$, where the error term $\Delta$ captures the variation of both transition kernels and representations over time.

As a result, by selecting the policy that maximizes $\hV_{\hP^k,\hb^k}^{\pi}$ as the exploration policy as in line 7, the agent will explore the trajectories whose states and actions have not been estimated sufficiently well so far.    

	\begin{algorithm}[H]
		\begin{algorithmic}[1]
			\caption{\textbf{E$^2$U} (Model \textbf{E}stimation and \textbf{E}xploration Policy \textbf{U}pdate)}\label{Alg: MEPE}
			\STATE {\bf Input:} round index $k$, regularizer $\lambda_{k,W}$ and coefficient $\tap_{k,W}$, datasets $\{\tDc_{h-1}^{(k,h)}\}$,$\{\tDc_h^{(k,h)}\}$ and models $ \{\Psi,  \Phi \}$.
			\FOR{step $h=1,\ldots,H$}
			\STATE Learn the representation via MLE for step $h$: 
			\begin{align*}
				\textstyle\hP_h^k=(\hphi_h^k,\hmu_h^k)= \arg\max_{\phi \in \Phi, \mu \in \Psi} \Eb_{\tDc_h^{(k,h)}}\left[\log\langle\phi(s_h,a_h),\mu(s_{h+1})\rangle\right]. 
			\end{align*}
			
			\STATE Compute the empirical covariance matrix as
	      \vspace{-0.05in}
			\begin{align*}
				\textstyle\hU_h^{k,W} = \sum_{\tDc_h^{(k,h+1)}}\hphi^k_h(s_h,a_h)\hphi^k_h(s_h,a_h)^\top+\lambda_{k,W} I
			\end{align*}
			\vspace{-0.15in}
			\STATE Define exploration-driven bonus $\hat{b}_h^k(\cdot,\cdot) = \min\{{\alpha_{k,W}}\|\hphi_h^k(\cdot,\cdot)\|_{(\hat{U}_h^{k,W})^{-1}},1\}$.
			\ENDFOR
			\STATE Find exploration policy $\tpi^k = \arg \max_{\pi}\hV_{\hP^k,\hb^k}^\pi$, where $\hV_{\hP^k,\hb^k}^\pi$ is defined as in \Cref{ineq: def of hV}.
			\STATE {\bf Output:} Model $\hP^k$ and exploration policy $\{\tpi^k\}$.
		\end{algorithmic}
	\end{algorithm}

\textbf{Step 3: Target Policy Update with Periodic Restart:} The agent evaluates the target policy by computing the value function under the target policy and the estimated transition kernel. Then, due to the nonstationarity, the target policy is reset every $\tau$ rounds. Compared with the previous work using periodic restart~\citep{DBLP:conf/nips/FeiYWX20}, whose choice of $\tau$ is based on a certain smooth visitation assumption, we remove such an assumption and hence our choice of $\tau$ is applicable to more general model classes. Finally, the agent uses the estimated value function for target policy update for the next round $k+1$ via online mirror descend. The update step is inspired by the previous works~\citep{DBLP:conf/icml/CaiYJW20,DBLP:journals/corr/SchulmanWDRK17}. Specifically, for any given policy $\pi^0$ and MDP $\Mc$, define the following function w.r.t. policy $\pi$ :
	\begin{align*}
		 L^{\Mc,\pi_0}(\pi)=V_{P,r}^{\pi^0}+\sum_{h=1}^H\Eb_{s_h \sim (P,\pi^0)}\left[\left\langle Q_{h,P,r}^{\pi^0},\pi_h(\cdot|s_h)-\pi^0_h(\cdot|s_h)\right\rangle\right].
	\end{align*}
	$L^{\Mc,\pi_0}(\pi)$ can be regarded as a local linear approximation of $V_{P,r}^\pi$ at ``point'' $\pi^0$~\citep{DBLP:journals/corr/SchulmanWDRK17}. Consider the following optimization problem:
	\begin{align*}
		&\textstyle\pi^{k+1}= \arg\max_{\pi}L^{\Mc^{k},\pi^{k}}(\pi)-\frac{1}{\eta}\sum_{h\in[H]}\Eb_{ s_h\sim(P^{\star,k},\pi^{k})}\left[D_{KL}(\pi_h(\cdot|s_h)\|\pi^{k}_h(\cdot|s_h))\right].
	\end{align*}
	This can be regarded as a mirror descent step with KL divergence, where the KL divergence regularizes $\pi$ to be close to $\pi^{k}$. It further admits a closed-form solution: $\pi_h^{k+1}(\cdot|\cdot) \propto \pi_h^{k}(\cdot|\cdot) \cdot \exp\{\eta\cdot Q_{h,P^{\star,k},r^k}^{\pi^k}(\cdot,\cdot)\}$. We use the estimated version $\hQ_h^{k}$ to approximate $Q_{h,P^{\star,k},r^k}^{\pi^k}$ and get 
 \begin{align}
 \textstyle    \pi_h^{k+1}(\cdot|\cdot) \propto \pi_h^{k}(\cdot|\cdot) \cdot \exp\{\eta\cdot\hQ_h^{k}(\cdot,\cdot)\}. \label{Eq: Alg-PolicyUpdate}
 \end{align}
 \subsection{Technical Assumptions}
 Our analysis adopts the following standard assumptions on low-rank MDPs.
\begin{assumption}\label{assumption: realizability}(Realizability).
A learning agent can access to a model class {$\{(\Phi,\Psi)\}$} that contains the true model, 
namely, for any $h \in [H], k \in[K]$, $\phi_h^{\star,k} \in \Phi, \mu_h^{\star,k} \in \Psi$. 
\end{assumption}
While we assume cardinality of the model class to be finite for simplicity, extensions to infinite classes with bounded statistical complexity are not difficult \citep{DBLP:conf/colt/SunJKA019}. 
\begin{assumption}[Bounded Density]\label{assumption: bounded ratio}
Any model induced by $\Phi$ and $\Psi$ has bounded density, i.e. $\forall P=\langle\phi,\mu\rangle, \phi \in \Phi, \mu \in \Psi$, there exists a constant $B \geq 0$ such that $\max_{(s,a,s^\prime) \in \Sc\times\Ac\times\Sc}P(s^\prime|s,a)\leq B$.
	\end{assumption}
 \begin{assumption}[Reachability]\label{assumption: reachability}
For each round $k$ and step $h$, the true transition kernel $P_h^{\star,k}$ satisfies that for any $(s,a,s^\prime) \in \Sc \times \Ac \times \Sc$, $P_h^{\star,k}(s^\prime|s,a) \geq p_{\min}$.
	\end{assumption}

\textbf{Variation Budgets:} We next introduce several measures of nonstationarity of the environment: $\Delta^{P}=\sum_{k=1}^K\sum_{h=1}^H\max_{(s,a)\in \Sc \times \Ac}\|P_h^{\star,k+1}(\cdot|s,a)-P_h^{\star,k}(\cdot|s,a)\|_{TV}$,$\Delta^{\sqrt{P}} =\sum_{k=1}^K\sum_{h=1}^H\max_{(s,a)\in \Sc \times \Ac}\|P_h^{\star,k+1}(\cdot|s,a)-P_h^{\star,k}(\cdot|s,a)\|_{TV}^{1/2}$,$\Delta^{\phi}= \sum_{k=1}^K\sum_{h=1}^H\max_{(s,a)\in \Sc \times \Ac}\|\phi_h^{\star,k+1}(s,a)-\phi_h^{\star,k}(s,a)\|_2$,$\Delta^{\pi}= \sum_{k=1}^K\sum_{h=1}^H\max_{s \in \Sc}\|\pi_h^{\star,k}(\cdot|s)-\pi_h^{\star,k-1}(\cdot|s)\|_{TV}$.
        These notions are known as \textit{variation budgets} or \textit{path lengths} in the literature of online convex optimization~\citep{DBLP:journals/ior/BesbesGZ15,DBLP:journals/ftopt/Hazan16,DBLP:journals/jstsp/HallW15} and nonstationary RL~\citep{DBLP:conf/nips/FeiYWX20,DBLP:journals/corr/abs-2110-08984,DBLP:journals/corr/abs-2010-04244}. The regret of nonstationary RL naturally depends on these notions that capture the variations of MDP models over time. 
        
	\subsection{Theoretical Guarantee}\label{sec: thm1}

To present our theoretical result for PORTAL, we first discuss technical challenges in our analysis and the novel tools that we develop. 
Generally, large nonstationarity of environment can cause significant errors to MLE, empirical covariance and exploration-driven bonus design for low-rank models. Thus, different from static low-rank MDPs~\citep{DBLP:conf/nips/AgarwalKKS20,DBLP:conf/iclr/UeharaZS22}, we devise several new techniques in our analysis to capture the errors caused by nonstationarity 
which we summarize below. 
\begin{enumerate}[leftmargin=5mm,topsep=0mm,itemsep=1mm]
\item Characterizing nonstationary MLE guarantee. We provide a theoretical ground to support our design of leveraging history data collected under various different transition kernels in previous rounds for benefiting the estimation of the current model, which is somewhat surprising. Specifically, we establish an MLE guarantee of the model estimation error, which features a separation of variation budgets from a diminishing term as the estimation sample size $W$ increases. Such a result justifies the usage of data generated by mismatched distributions within a certain window as long as the variation budgets is mild. Such a separation cannot be shown directly. Instead, we bridge the bound of model estimation error and the expected value of the ratio of transition kernels via Hellinger distance, and the latter can be decomposed into the variation budgets and a diminishing term as the estimation sample size increases. 
\item Establishing trajectory-wise uncertainty for estimation error $\hV_{\hP^k,\hb^k}^\pi$. To this end, straightforward combination of our nonstationary MLE guarantee with previous techniques on low-rank MDPs would yield a coefficient $\tap_{k,W}$ that depends on local variation budgets. Instead, we convert the $\ell_\infty$-norm variation budgets $\Delta^P$ to square-root $\ell_\infty$-norm variation budget $\Delta^{\sqrt{P}}$. In this way, the coefficient no longer depends on the local variation budgets, and the estimation error can be upper bounded by $\hV_{\hP^k,\hb^k}^\pi$ plus an error term only depending on the square-root $\ell_\infty$-norm variation budgets. 

\item Error tracking via auxiliary anchor representation. In proving the convergence of average of $\hV_{\hP^k,\hb^k}^\pi$, standard elliptical potential based analysis cannot work, because the representation $\phi^{\star,k}$ in the elliptical potential $\norm{\phi^{\star,k}(s,a)}_{(U_{h,\phi}^{k,W})^{-1}}$ changes across rounds, where $U_{h,\phi}^{k,W}$ is the population version of $\hU_h^{k,W}$.
    To deal with this challenge, in our analysis, we divide the total $K$ rounds into blocks with equal length of $W$ rounds. Then for each block, we set an \textit{auxiliary anchor} representation. We keep track of the elliptical potential functions using the anchor representation within each block, and control the errors by using anchor representation via variation budgets. 
\end{enumerate}
The following theorem characterizes theoretical performance for PORTAL. 
\begin{theorem}\label{Thm1: average dynamic suboptimality gap with known variation}
$\{\Mc^k\}_{k=1}^K$ is set of low-rank MDPs with dimension $d$. Under \Cref{assumption: realizability,assumption: bounded ratio,assumption: reachability}, set $\tap_{k,W}=\tilde{O}\left(\sqrt{A+d^2}\right)$ and $\lambda_{k,W}=\tilde{O}(d)$. Let $\{\pi^k\}_{k=1}^{K}$ be the output of PORTAL in \Cref{Alg: DPO}. For any $\delta \in (0,1)$, with probability at least $1-\delta$, $\mathrm{Gap_{Ave}}(K)$ of PORTAL is at most
	\begin{align}
		\tilde{O}\Big(\underbrace{ {\sqrt{\frac{H^4d^2A}{W}\left(A+d^2\right)}+\sqrt{\frac{H^3dA}{K}\left(A+d^2\right)W^2\Delta^{{\phi}}}+\sqrt{\frac{H^2W^3A}{K^2}}\Delta^{\sqrt{P}}}}_{(I)} \quad +\underbrace{\frac{H}{\sqrt{\tau}} + \frac{H\tau}{K} (\Delta^P+\Delta^{\pi})}_{(II)}\Big). \label{Eq: Basic Bound}
	\end{align}
 \end{theorem}
We explain the upper bound in \Cref{Thm1: average dynamic suboptimality gap with known variation} as follows. The basic bound as \Cref{Eq: Basic Bound} in \Cref{Thm1: average dynamic suboptimality gap with known variation} contains two parts: the first part $(I)$ captures the estimation error for evaluating the target policy under the true environment via the estimated value function $\hat{Q}^k$ as in line 16 of \Cref{Alg: DPO}. Hence, part $(I)$ decreases with $K$ and increases with the nonstationarity of transition kernels and representation functions. Also, part $(I)$ is greatly affected by the window size $W$, which is determined by the dataset used to estimate the transition kernels. Typically, $W$ is tuned carefully based on the variation of environment. If the environment changes significantly, then the samples far in the past are obsolete and become not very informative for estimating transition kernels. The second part $(II)$ captures the approximation error arising in finding the optimal policy via the policy optimization method as in line 8 of \Cref{Alg: DPO}. Due to the nonstationarity of the environment, the optimal policy keeps changing across rounds, and hence the nonstationarity of optimal policy $\Delta^{\pi}$ affects the approximation error. Similarly to the window size $W$ in part $(I)$, the policy restart period $\tau$ can also be tuned carefully based on the variation of environment and the optimal policies.  
 \begin{corollary}\label{Coro: of Thm1}
Under the same conditions of \cref{Thm1: average dynamic suboptimality gap with known variation}, if the variation budgets are known, then we can select the hyper-parameters correspondingly to achieve optimality. Specially, if the nonstationarity of the environment is moderate, we have 
\begin{align*}
    \mathrm{Gap_{Ave}}(K) \leq \tilde{O}({H^{\frac{11}{6}}d^{\frac{5}{6}}A^{\frac{1}{2}}}\left(A+d^2\right)^{\frac{1}{2}}K^{-\frac{1}{6}}(\Delta^{\sqrt{P}}+\Delta^{{\phi}})^{\frac{1}{6}}+2HK^{-\frac{1}{3}}(\Delta^P+\Delta^\pi)^{\frac{1}{3}}).
\end{align*}
If the environment is near stationary, then the best $W$ and $\tau$ are $K$. The $\mathrm{Gap_{Ave}}$ reduces to $\tilde{O}(\sqrt{{H^4d^2A}(A+d^2)/K})$, which matches results of stationary environment~\citep{DBLP:conf/iclr/UeharaZS22}.
\end{corollary}
 Detailed discussions and proofs of \Cref{Thm1: average dynamic suboptimality gap with known variation,Coro: of Thm1} are provided in \Cref{Appd A,Appd B}, respectively.
	
 \section{Parameter-free Algorithm: Ada-PORTAL}
As shown in \Cref{Thm1: average dynamic suboptimality gap with known variation}, hyper-parameters $W$ and $\tau$ greatly affect the performance of \Cref{Alg: DPO}. With prior knowledge of variation budgets, the agent is able to optimize the performance as in \Cref{Coro: of Thm1}. However, in practice, variation budgets are unknown to the agent. To deal with this issue, in this section, we present a parameter-free algorithm called Ada-PORTAL in \Cref{Alg: ADPO}, which is able to tune hyper-parameters without knowing the variation budgets beforehand.
         
\Cref{Alg: ADPO} is inspired by the BORL method~\citep{DBLP:conf/icml/CheungSZ20}. The idea is to use \Cref{Alg: DPO} as a subroutine and treat the selection of the hyper-parameters such as $W$ and $\tau$ as a bandit problem.
\begin{algorithm}
        	\begin{algorithmic}[1]
        		\caption{\textbf{Ada-PORTAL} (\textbf{Ada}ptive \textbf{P}olicy \textbf{O}ptimization \textbf{R}epresen\textbf{TA}tion \textbf{L}earning)} \label{Alg: ADPO}
        		\STATE {\bf Input:} Confidence level $\delta$, number of episodes $K$, block length $M$, feasible set of window size $\Jc_W$ and policy restart period $\Jc_\tau$.
        		\STATE {\bf Initialization:} Initialize $\alpha, \beta, \gamma$ and $\left\{q_{l,1}\right\}_{l \in [J]}$ as in \Cref{Eq: sec4-3-1}.
        		\FOR{block $i=1,\ldots, \lceil K/M\rceil$}
        		\STATE Update the windows size selection distribution $\{u_{(k,l),i}\}_{(k,l)\in \Jc}$ as in \Cref{Eq: sec4-3-2}.
        		\STATE Sample $(k_i,l_i) \in \Jc$ from the updated distribution $\{u_{(k,l),i}\}_{(k,l) \in \Jc}$, then set $W_i=\lfloor M^{k_i/J_M}\rfloor$ and $\tau_i=\lfloor M_\tau^{l_i/J_\tau}\rfloor$.
        		\FOR{episode $k=(i-1)M+1,\ldots,\min\{iM,K\}$}
        		\STATE \textbf{Run} $\mathrm{PORTAL}$ with $W_i$ and $\tau_i$.   
        		\ENDFOR
        		\STATE Compute the total reward for block $i$ as $ R_i(W_i,\tau_i)=\sum_{k=(i-1)M+1}^{\min\{iM,K\}}V_1^k$, where $V_1^k$ is empirical value functions of target policy $\pi^k$, and update the estimated total reward of running different epoch sizes $\{q_{(k,l),i+1}\}_{(k,l) \in \Jc}$ according to \Cref{Eq: sec4-3-3}.
        		\ENDFOR 
        		\STATE {\bf Output:} $\{\pi^k\}_{k=1}^K$.
        	\end{algorithmic}
        \end{algorithm}
        Specifically, Ada-PORTAL divides the entire $K$ rounds into $\lceil K/M\rceil$ blocks with equal length of $M$ rounds. Then two sets $\mathcal{J}_W$ and $\mathcal{J}_\tau$ are specified (see later part of this section), from which the window size $W$ and the restart period $\tau$ for each block are drawn. 

        For each block $i \in \left[\lceil\frac{K}{M}\rceil\right]$, Ada-PORTAL treats each element of $\mathcal{J}_W\times\mathcal{J}_\tau$ as an arm and take it as a bandit problem to select the best arm for each block. In lines 4 and 5 of \Cref{Alg: ADPO}, a master algorithm is run to update parameters and select the desired arm, i.e., the window size $W_i$ and the restart period $\tau_i$. Here we choose EXP3-P~\citep{DBLP:journals/ftml/BubeckC12} as the master algorithm and discuss the details later. Then \Cref{Alg: DPO} is called with input $W_i$ and $\tau_i$ as a subroutine for the current block $i$. At the end of each block, the total reward of the current block is computed by summing up all the empirical value functions of the target policy of each episode within the block, which is then used to update the parameters for the next block.

We next set the feasible sets and block length of \Cref{Alg: DPO}. Since optimal $W$ and $\tau$ in PORTAL are chosen differently from previous works~\citep{DBLP:journals/corr/abs-2010-04244,DBLP:conf/aistats/CheungSZ19,DBLP:journals/corr/abs-2010-12870} on nonstationary MDPs due to the low-rank structure, the feasible set here that covers the optimal choices of $W$ and $\tau$ should also be set differently from those previous works using BORL.

 
 $M_W=d^{\frac{1}{3}}H^{\frac{1}{3}}K^{\frac{1}{3}}$, $M_\tau=K^{\frac{2}{3}}$, $M=d^{\frac{1}{3}}H^{\frac{1}{3}}K^{\frac{2}{3}}$. 
	$J_W=\lfloor\log (M_W)\rfloor$, $\Jc_W=\{M_W^0,\lfloor M_W^{\frac{1}{J_W}}\rfloor,\ldots,M_W\}$, $J_\tau=\lfloor\log (M_\tau)\rfloor$, $\Jc_\tau=\{M_\tau^0,\lfloor M_\tau^{\frac{1}{J_\tau}}\rfloor,\ldots,M_\tau\}$, $J=J_W \cdot J_\tau$.
 
	Then the parameters of EXP3-P are intialized as follows:
	\begin{align}
		 \textstyle \alpha = 0.95\sqrt{\frac{\ln J}{J \lceil K/M\rceil}}, \quad \beta = \sqrt{\frac{\ln J}{J \lceil K/M\rceil}}, \gamma = 1.05\sqrt{\frac{J\ln J}{\lceil K/M\rceil}}, \quad q_{(k,l),1}=0, (k,l)\in \Jc, \label{Eq: sec4-3-1}
	\end{align}
	where $\Jc=\left\{(k,l): k \in \left\{0,1,\ldots,J_W\right\}, l \in  \left\{0,1,\ldots,J_\tau\right\}\right\}$.
	The parameter updating rule is as follows. For any $(k,l) \in \Jc, i \in \lceil K/M\rceil$,
	\begin{align}
		\textstyle u_{(k,l),i}=(1-\gamma)\frac{\exp(\alpha q_{(k,l),i})}{\sum_{(k,l) \in \Jc}\exp(\alpha q_{(k,l),i})}+\frac{\gamma}{J}\label{Eq: sec4-3-2},
	\end{align}
	where $u_{(k,l),i}$ is a probability over $\Jc$. From $u_{(k,l),i}$, the agent samples a desired pair $(k_i,l_i)$ for each block $i$, which corresponds to the index of feasible set $\Jc_W \times \Jc_\tau$ and is used to set $W_i$ and $\tau_i$. 
	
	As a last step, $ R_i(W_i,\tau_i)$ is rescaled to update $q_{(k,l),i+1}$. 
	\begin{align}
		\textstyle q_{(k,l),i+1}=q_{(k,l),i}+\frac{\beta+1_{(k,l)=(k_i,l_i)}R_i(W_i,\tau_i)/M}{u_{(k,l),i}}. \label{Eq: sec4-3-3}
	\end{align}
 We next establish a bound on $\mathrm{Gap_{Ave}}$ for Ada-PORTAL. 
 \begin{theorem}\label{Thm2: Ada-PORTAL}
     Under the same conditions of \Cref{Thm1: average dynamic suboptimality gap with known variation}, with probability at least $1-\delta$, the average dynamic suboptimality gap of Ada-PORTAL in \Cref{Alg: ADPO} is  upper-bounded as $\mathrm{Gap_{Ave}}(K)\leq \tilde{O}({H^{\frac{11}{6}}d^{\frac{5}{6}}A^{\frac{1}{2}}}(A+d^2)^{\frac{1}{2}}K^{-\frac{1}{6}}(\Delta^{\sqrt{P}}+\Delta^{{\phi}}+1)^{\frac{1}{6}}+2HK^{-\frac{1}{3}}(\Delta^P+\Delta^\pi+1)^{\frac{1}{3}})$.
 \end{theorem}
\textbf{Comparison with \citet{DBLP:conf/colt/WeiL21}:} It is interesting to compare Ada-PORTAL with an alternative black-box type of approach to see its advantage. A black-box technique called MASTER was proposed in \citet{DBLP:conf/colt/WeiL21}, which can also work with any base algorithm such as PORTAL to handle unknown variation budgets. Such a combined approach of MASTER+PORTAL turns out to have a worse $\mathrm{Gap_{Ave}}$ than our Ada-PORTAL in \Cref{Alg: ADPO}. To see this, denote $\Delta={\Delta^{{\phi}}}+\Delta^{\sqrt{P}}+\Delta^{\pi}$. The $\mathrm{Gap_{Ave}}$ of MASTER+PORTAL is $\tilde{O}(K^{-\frac{1}{6}}\Delta^{\frac{1}{3}})$. Then if variation budgets are not too small, i.e. $\Delta \geq \tilde{O}(1)$, this $\mathrm{Gap_{Ave}}$ is worse than Ada-PORTAL. See detailed discussion in \Cref{Appd: B.2}.


\section{Conclusion}
         
In the paper, we investigate nonstationary RL under low-rank MDPs. We first propose a notion of average dynamic suboptimality gap $\mathrm{Gap_{Ave}}$ to evaluate the performance of a series of policies in a nonstationary environment. Then we propose a sample-efficient policy optimization algorithm PORTAL and its parameter-free version Ada-PORTAL. We further provide upper bounds on $\mathrm{Gap_{Ave}}$ for both algorithms. As future work, it is interesting to investigate the impact of various constraints such as safety requirements in nonstationary RL under function approximations.
	\bibliographystyle{icml2023}
	\bibliography{DRPO}

\begin{thebibliography}{38}
\providecommand{\natexlab}[1]{#1}
\providecommand{\url}[1]{\texttt{#1}}
\expandafter\ifx\csname urlstyle\endcsname\relax
  \providecommand{\doi}[1]{doi: #1}\else
  \providecommand{\doi}{doi: \begingroup \urlstyle{rm}\Url}\fi

\bibitem[Agarwal et~al.(2020{\natexlab{a}})Agarwal, Henaff, Kakade, and
  Sun]{DBLP:conf/nips/AgarwalHKS20}
Agarwal, A., Henaff, M., Kakade, S.~M., and Sun, W.
\newblock {PC-PG:} policy cover directed exploration for provable policy
  gradient learning.
\newblock In \emph{Advances in Neural Information Processing Systems 33},
  2020{\natexlab{a}}.

\bibitem[Agarwal et~al.(2020{\natexlab{b}})Agarwal, Kakade, Krishnamurthy, and
  Sun]{DBLP:conf/nips/AgarwalKKS20}
Agarwal, A., Kakade, S.~M., Krishnamurthy, A., and Sun, W.
\newblock {FLAMBE:} structural complexity and representation learning of low
  rank mdps.
\newblock In \emph{Advances in Neural Information Processing Systems 33},
  2020{\natexlab{b}}.

\bibitem[Agarwal et~al.(2022)Agarwal, Song, Sun, Wang, Wang, and
  Zhang]{agarwal2022provable}
Agarwal, A., Song, Y., Sun, W., Wang, K., Wang, M., and Zhang, X.
\newblock Provable benefits of representational transfer in reinforcement
  learning.
\newblock \emph{arXiv preprint arXiv:2205.14571}, 2022.

\bibitem[Besbes et~al.(2015)Besbes, Gur, and
  Zeevi]{DBLP:journals/ior/BesbesGZ15}
Besbes, O., Gur, Y., and Zeevi, A.
\newblock Non-stationary stochastic optimization.
\newblock \emph{Oper. Res.}, 63\penalty0 (5):\penalty0 1227--1244, 2015.

\bibitem[Bojarski et~al.(2016)Bojarski, Del~Testa, Dworakowski, Firner, Flepp,
  Goyal, Jackel, Monfort, Muller, Zhang, et~al.]{bojarski2016end}
Bojarski, M., Del~Testa, D., Dworakowski, D., Firner, B., Flepp, B., Goyal, P.,
  Jackel, L.~D., Monfort, M., Muller, U., Zhang, J., et~al.
\newblock End to end learning for self-driving cars.
\newblock \emph{arXiv preprint arXiv:1604.07316}, 2016.

\bibitem[Bubeck \& Cesa{-}Bianchi(2012)Bubeck and
  Cesa{-}Bianchi]{DBLP:journals/ftml/BubeckC12}
Bubeck, S. and Cesa{-}Bianchi, N.
\newblock Regret analysis of stochastic and nonstochastic multi-armed bandit
  problems.
\newblock \emph{Found. Trends Mach. Learn.}, 2012.

\bibitem[Cai et~al.(2020)Cai, Yang, Jin, and Wang]{DBLP:conf/icml/CaiYJW20}
Cai, Q., Yang, Z., Jin, C., and Wang, Z.
\newblock Provably efficient exploration in policy optimization.
\newblock In \emph{Proceedings of the 37th International Conference on Machine
  Learning}, 2020.

\bibitem[Cheng et~al.(2022)Cheng, Feng, Yang, Zhang, and Liang]{chengprovable}
Cheng, Y., Feng, S., Yang, J., Zhang, H., and Liang, Y.
\newblock Provable benefit of multitask representation learning in
  reinforcement learning.
\newblock In \emph{Advances in Neural Information Processing Systems}, 2022.

\bibitem[Cheng et~al.(2023)Cheng, Huang, Yang, and Liang]{cheng2023improved}
Cheng, Y., Huang, R., Yang, J., and Liang, Y.
\newblock Improved sample complexity for reward-free reinforcement learning
  under low-rank mdps.
\newblock \emph{arXiv preprint arXiv:2303.10859}, 2023.

\bibitem[Cheung et~al.(2019)Cheung, Simchi{-}Levi, and
  Zhu]{DBLP:conf/aistats/CheungSZ19}
Cheung, W.~C., Simchi{-}Levi, D., and Zhu, R.
\newblock Learning to optimize under non-stationarity.
\newblock In \emph{The 22nd International Conference on Artificial Intelligence
  and Statistics}, volume~89 of \emph{Proceedings of Machine Learning
  Research}, pp.\  1079--1087. {PMLR}, 2019.

\bibitem[Cheung et~al.(2020)Cheung, Simchi{-}Levi, and
  Zhu]{DBLP:conf/icml/CheungSZ20}
Cheung, W.~C., Simchi{-}Levi, D., and Zhu, R.
\newblock Reinforcement learning for non-stationary markov decision processes:
  The blessing of (more) optimism.
\newblock In \emph{Proceedings of the 37th International Conference on Machine
  Learning}, 2020.

\bibitem[Dann et~al.(2017)Dann, Lattimore, and
  Brunskill]{DBLP:conf/nips/DannLB17}
Dann, C., Lattimore, T., and Brunskill, E.
\newblock Unifying {PAC} and regret: Uniform {PAC} bounds for episodic
  reinforcement learning.
\newblock In \emph{Advances in Neural Information Processing Systems 30}, 2017.

\bibitem[Fei et~al.(2020)Fei, Yang, Wang, and Xie]{DBLP:conf/nips/FeiYWX20}
Fei, Y., Yang, Z., Wang, Z., and Xie, Q.
\newblock Dynamic regret of policy optimization in non-stationary environments.
\newblock In \emph{Advances in Neural Information Processing Systems 33}, 2020.

\bibitem[Gajane et~al.(2018)Gajane, Ortner, and
  Auer]{DBLP:journals/corr/abs-1805-10066}
Gajane, P., Ortner, R., and Auer, P.
\newblock A sliding-window algorithm for markov decision processes with
  arbitrarily changing rewards and transitions.
\newblock \emph{CoRR}, abs/1805.10066, 2018.

\bibitem[Garivier \& Moulines(2011)Garivier and Moulines]{garivier2011upper}
Garivier, A. and Moulines, E.
\newblock On upper-confidence bound policies for switching bandit problems.
\newblock In \emph{International Conference on Algorithmic Learning Theory},
  2011.

\bibitem[Gu et~al.(2017)Gu, Holly, Lillicrap, and Levine]{gu2017deep}
Gu, S., Holly, E., Lillicrap, T., and Levine, S.
\newblock Deep reinforcement learning for robotic manipulation with
  asynchronous off-policy updates.
\newblock In \emph{IEEE international conference on robotics and automation},
  2017.

\bibitem[Hall \& Willett(2015)Hall and Willett]{DBLP:journals/jstsp/HallW15}
Hall, E.~C. and Willett, R.~M.
\newblock Online convex optimization in dynamic environments.
\newblock \emph{{IEEE} J. Sel. Top. Signal Process.}, 9\penalty0 (4):\penalty0
  647--662, 2015.

\bibitem[Hazan(2016)]{DBLP:journals/ftopt/Hazan16}
Hazan, E.
\newblock Introduction to online convex optimization.
\newblock \emph{Found. Trends Optim.}, 2\penalty0 (3-4):\penalty0 157--325,
  2016.

\bibitem[Jiang et~al.(2017)Jiang, Krishnamurthy, Agarwal, Langford, and
  Schapire]{jiang2017contextual}
Jiang, N., Krishnamurthy, A., Agarwal, A., Langford, J., and Schapire, R.~E.
\newblock Contextual decision processes with low bellman rank are
  pac-learnable.
\newblock In \emph{International Conference on Machine Learning}, 2017.

\bibitem[Levine et~al.(2016)Levine, Finn, Darrell, and Abbeel]{levine2016end}
Levine, S., Finn, C., Darrell, T., and Abbeel, P.
\newblock End-to-end training of deep visuomotor policies.
\newblock \emph{The Journal of Machine Learning Research}, 17\penalty0
  (1):\penalty0 1334--1373, 2016.

\bibitem[Ma et~al.(2021)Ma, Li, Kochenderfer, Isele, and
  Fujimura]{DBLP:conf/icra/Ma0KIF21}
Ma, X., Li, J., Kochenderfer, M.~J., Isele, D., and Fujimura, K.
\newblock Reinforcement learning for autonomous driving with latent state
  inference and spatial-temporal relationships.
\newblock In \emph{{IEEE} International Conference on Robotics and Automation},
  2021.

\bibitem[Mao et~al.(2021)Mao, Zhang, Zhu, Simchi{-}Levi, and
  Basar]{DBLP:conf/icml/MaoZZSB21}
Mao, W., Zhang, K., Zhu, R., Simchi{-}Levi, D., and Basar, T.
\newblock Near-optimal model-free reinforcement learning in non-stationary
  episodic mdps.
\newblock In \emph{Proceedings of the 38th International Conference on Machine
  Learning}, 2021.

\bibitem[Modi et~al.(2021)Modi, Chen, Krishnamurthy, Jiang, and
  Agarwal]{modi2021model}
Modi, A., Chen, J., Krishnamurthy, A., Jiang, N., and Agarwal, A.
\newblock Model-free representation learning and exploration in low-rank mdps.
\newblock \emph{arXiv preprint arXiv:2102.07035}, 2021.

\bibitem[Schulman et~al.(2017)Schulman, Wolski, Dhariwal, Radford, and
  Klimov]{DBLP:journals/corr/SchulmanWDRK17}
Schulman, J., Wolski, F., Dhariwal, P., Radford, A., and Klimov, O.
\newblock Proximal policy optimization algorithms.
\newblock \emph{CoRR}, abs/1707.06347, 2017.

\bibitem[Shani et~al.(2020)Shani, Efroni, Rosenberg, and
  Mannor]{DBLP:conf/icml/ShaniE0M20}
Shani, L., Efroni, Y., Rosenberg, A., and Mannor, S.
\newblock Optimistic policy optimization with bandit feedback.
\newblock In \emph{Proceedings of the 37th International Conference on Machine
  Learning}, 2020.

\bibitem[Silver et~al.(2016)Silver, Huang, Maddison, Guez, Sifre, Van
  Den~Driessche, Schrittwieser, Antonoglou, Panneershelvam, Lanctot,
  et~al.]{silver2016mastering}
Silver, D., Huang, A., Maddison, C.~J., Guez, A., Sifre, L., Van Den~Driessche,
  G., Schrittwieser, J., Antonoglou, I., Panneershelvam, V., Lanctot, M.,
  et~al.
\newblock Mastering the game of go with deep neural networks and tree search.
\newblock \emph{nature}, 529\penalty0 (7587):\penalty0 484--489, 2016.

\bibitem[Silver et~al.(2017)Silver, Hubert, Schrittwieser, Antonoglou, Lai,
  Guez, Lanctot, Sifre, Kumaran, Graepel, et~al.]{silver2017mastering}
Silver, D., Hubert, T., Schrittwieser, J., Antonoglou, I., Lai, M., Guez, A.,
  Lanctot, M., Sifre, L., Kumaran, D., Graepel, T., et~al.
\newblock Mastering chess and shogi by self-play with a general reinforcement
  learning algorithm.
\newblock \emph{arXiv preprint arXiv:1712.01815}, 2017.

\bibitem[Silver et~al.(2018)Silver, Hubert, Schrittwieser, Antonoglou, Lai,
  Guez, Lanctot, Sifre, Kumaran, Graepel, et~al.]{silver2018general}
Silver, D., Hubert, T., Schrittwieser, J., Antonoglou, I., Lai, M., Guez, A.,
  Lanctot, M., Sifre, L., Kumaran, D., Graepel, T., et~al.
\newblock A general reinforcement learning algorithm that masters chess, shogi,
  and go through self-play.
\newblock \emph{Science}, 362\penalty0 (6419):\penalty0 1140--1144, 2018.

\bibitem[Sun et~al.(2019)Sun, Jiang, Krishnamurthy, Agarwal, and
  Langford]{DBLP:conf/colt/SunJKA019}
Sun, W., Jiang, N., Krishnamurthy, A., Agarwal, A., and Langford, J.
\newblock Model-based {RL} in contextual decision processes: {PAC} bounds and
  exponential improvements over model-free approaches.
\newblock In \emph{Conference on Learning Theory}, 2019.

\bibitem[Touati \& Vincent(2020)Touati and
  Vincent]{DBLP:journals/corr/abs-2010-12870}
Touati, A. and Vincent, P.
\newblock Efficient learning in non-stationary linear markov decision
  processes.
\newblock \emph{CoRR}, abs/2010.12870, 2020.

\bibitem[Tsybakov(2009)]{DBLP:books/daglib/0035708}
Tsybakov, A.~B.
\newblock \emph{Introduction to Nonparametric Estimation}.
\newblock Springer series in statistics. Springer, 2009.

\bibitem[Uehara et~al.(2022)Uehara, Zhang, and Sun]{DBLP:conf/iclr/UeharaZS22}
Uehara, M., Zhang, X., and Sun, W.
\newblock Representation learning for online and offline {RL} in low-rank mdps.
\newblock In \emph{The Tenth International Conference on Learning
  Representations}, 2022.

\bibitem[Wei \& Luo(2021)Wei and Luo]{DBLP:conf/colt/WeiL21}
Wei, C. and Luo, H.
\newblock Non-stationary reinforcement learning without prior knowledge: an
  optimal black-box approach.
\newblock In \emph{Conference on Learning Theory}, 2021.

\bibitem[Xu et~al.(2021)Xu, Liang, and Lan]{DBLP:conf/icml/XuLL21}
Xu, T., Liang, Y., and Lan, G.
\newblock {CRPO:} {A} new approach for safe reinforcement learning with
  convergence guarantee.
\newblock In \emph{Proceedings of the 38th International Conference on Machine
  Learning}, 2021.

\bibitem[Zanette et~al.(2021)Zanette, Cheng, and
  Agarwal]{zanette2021cautiously}
Zanette, A., Cheng, C.-A., and Agarwal, A.
\newblock Cautiously optimistic policy optimization and exploration with linear
  function approximation.
\newblock In \emph{Conference on Learning Theory}, 2021.

\bibitem[Zhao et~al.(2021)Zhao, Gu, Zhang, Yang, Liu, Tang, and
  Liu]{DBLP:conf/aaai/ZhaoGZYLTL21}
Zhao, X., Gu, C., Zhang, H., Yang, X., Liu, X., Tang, J., and Liu, H.
\newblock {DEAR:} deep reinforcement learning for online advertising impression
  in recommender systems.
\newblock In \emph{Thirty-Fifth {AAAI} Conference on Artificial Intelligence},
  2021.

\bibitem[Zhong et~al.(2021)Zhong, Yang, Wang, and
  Szepesv{\'{a}}ri]{DBLP:journals/corr/abs-2110-08984}
Zhong, H., Yang, Z., Wang, Z., and Szepesv{\'{a}}ri, C.
\newblock Optimistic policy optimization is provably efficient in
  non-stationary mdps.
\newblock \emph{CoRR}, abs/2110.08984, 2021.

\bibitem[Zhou et~al.(2020)Zhou, Chen, Varshney, and
  Jagmohan]{DBLP:journals/corr/abs-2010-04244}
Zhou, H., Chen, J., Varshney, L.~R., and Jagmohan, A.
\newblock Nonstationary reinforcement learning with linear function
  approximation.
\newblock \emph{CoRR}, 2020.

\end{thebibliography}
        \newpage
	\appendix
         \onecolumn
         \allowdisplaybreaks

{\Large {\bf Supplementary Materials}}

	\section{Proof of \Cref{Thm1: average dynamic suboptimality gap with known variation}} \label{Appd A}
	We summarize frequently used notations in the following list.
	\begin{align*}
		\begin{array}{ll}
			\zeta_{k,W} & \frac{2\log\left(2|\Phi||\Psi|kH/\delta\right)}{W}\\
			\lambda_{k,W} & O(d\log(|\Phi|\min\{k,W\}TH/\delta)) \\
			\alpha_{k,W} & \sqrt{2WA\zeta_{k,W}+\lambda_{k,W} d} = O(\sqrt{4A\log\left(2|\Phi||\Psi|kH/\delta\right)+\lambda_{k,W} d})\\
                \tap_{k,W} & 5\sqrt{2WA\zeta_{k,W}+\lambda_{k,W} d}\\
			\beta_{k,W} & \sqrt{9dA(2WA\zeta_{k,W}+\lambda_{k,W} d)  + \lambda_{k,W} d }\\
			\eta & \sqrt{\frac{L \log A}{K}}\\
			\Delta^{P}_{\mathcal{H},\mathcal{I}} & \sum_{h \in \mathcal{H}} \sum_{i \in \mathcal{I}}\max_{(s,a)\in \Sc \times \Ac}\norm{P_h^{\star,i+1}(\cdot|s,a)-P_h^{\star,i}(\cdot|s,a)}_{TV}\\
              \Delta^{\sqrt{P}}_{\mathcal{H},\mathcal{I}} & \sum_{h \in \mathcal{H}} \sum_{i \in \mathcal{I}}\max_{(s,a)\in \Sc \times \Ac}\sqrt{\norm{P_h^{\star,i+1}(\cdot|s,a)-P_h^{\star,i}(\cdot|s,a)}_{TV}}\\
              \Delta^{\phi}_{\mathcal{H},\mathcal{I}} & \sum_{h \in \mathcal{H}} \sum_{i \in \mathcal{I}}\max_{(s,a)\in \Sc \times \Ac}\norm{\phi_h^{\star,i+1}(s,a)-\phi_h^{\star,i}(s,a)}_2\\
               \Delta^{r}_{\mathcal{H},\mathcal{I}} & \sum_{h \in \mathcal{H}} \sum_{i \in \mathcal{I}}\max_{(s,a)\in \Sc \times \Ac}\norm{r_h^{\star,i+1}(s,a)-r_h^{\star,i}(s,a)}_2\\
               \Delta^{\pi}_{\mathcal{H},\mathcal{I}} & \sum_{h \in \mathcal{H}} \sum_{i \in \mathcal{I}} \max_{s \in \Sc} \norm{ \pi_h^{\star,i}(\cdot|s)-\pi_h^{\star,i-1}(\cdot|s)}_{TV}\\
			f_h^k(s,a) &\|\hP_h^{k}(\cdot|s,a) - P^{\star,k}_h(\cdot|s,a)\|_{TV} \\
			U_{h,\phi}^{k,W} &\sum_{i=1 \lor k-W}^{k-1} \Eb_{s_h\sim(P^{  {\star,i}  },\tilde{\pi}^i),a_h\sim \Uc(\Ac)}\left[\phi(s_h,a_h)(\phi(s_h,a_h))^\top\right] + \lambda_{k,W} I_d\\
			\hU_h^{k,W} & \sum_{\tDc_h^{(k,h+1)}}\hphi_h(s_h,a_h)\hphi_h(s_h,a_h)^\top+\lambda_{k,W} I_d\\
			W_{h,\phi}^{k,W} &\sum_{i=1 \lor k-W}^{k-1}\Eb_{(s_h,a_h)\sim(P^{{\star,i}  },\tilde{\pi}^i)}\left[\phi(s_h,a_h)(\phi(s_h,a_h))^\top\right] + \lambda_{k,W} I_d \\
			b^k_h & \min\left\{\alpha_{k,W}\left\|\hphi_{h}^k(s_{h},a_{h})\right\|_{(U^{k,W}_{h,\hphi^k})^{-1}},1\right\}\\
			\hat{b}_h^k & \min\left\{\tap_{k,W}\left\|\hphi_{h}^k(s_{h},a_{h})\right\|_{(\hat{U}^{k,W}_{h})^{-1}},1\right\}
		\end{array}
	\end{align*}

 \begin{proof}[Proof Sketch of \Cref{Thm1: average dynamic suboptimality gap with known variation}] The proof contains the following three main steps.
 
 \textbf{Step 1 (\Cref{Appd: A.1}):} We first decompose the average dynamic suboptimal gap into three terms as in \Cref{lemma: regret decomposition}, which can be divided into two parts: one part corresponds to the model estimation error and the other part corresponds to the performance difference between the target policy chosen by the agent and optimal policy. We then bound the two parts separately.

 \textbf{Step 2 (\Cref{Appd: A.2}):} For the first part corresponding to model estimation error, first by \Cref{lemma: Step_back_Bounded TV_hP,lemma: bounded difference of value function}, we show that the model estimation error can be bounded by the average of the truncated value functions of the bonus terms i.e. $\frac{1}{K}\sum_{k=1}^K\hV_{\hP^k,\hb^k}^{\pi}$ plus a term w.r.t. variation budgets. We then upper bound the average of $\hV_{\hP^k,\hb^k}^{\pi}$ as in \Cref{Lemma: value function summation}. To this end, we divide the total $K$ rounds into blocks with equal length of $W$ and adopt an auxiliary anchor representation for each block to deal with the challenge arising from time-varying representations when using standard elliptical potential based methods.
 
\textbf{Step 3 (\Cref{Appd: A.3}): } For the second part corresponding to the performance difference bound, similarly to dealing with changing representations, since the optimal policy changes over time, we adopt an auxiliary anchor policy and decompose the performance difference bound into two terms as in \Cref{ineq: Performance Difference Bound Decompose} and bound the two terms separately.  
\end{proof}

We further note that the above analysis techniques can also be applied to 
  RL problems where model mis-specification exists, i.e. $\phi^{\star,k} \notin \Phi, \mu^{\star,k} \notin \Psi$.

{\bf Organization of the Proof for \Cref{Thm1: average dynamic suboptimality gap with known variation}.} Our proof of \Cref{Thm1: average dynamic suboptimality gap with known variation} is organized as follows. In \Cref{Appd: A.1}, we provide the decomposition of the average dynamic suboptimality gap $\mathrm{Gap_{Ave}}$ in \Cref{Eq: Regret Decomposition}; in \Cref{Appd: A.2}, we bound the first and third terms of $\mathrm{Gap_{Ave}}$; in \Cref{Appd: A.3}, we bound the second term of $\mathrm{Gap_{Ave}}$, and in \Cref{Appd A.4.: Proof of Thm1} we combine our results to complete the proof of \Cref{Thm1: average dynamic suboptimality gap with known variation}. We provide all the supporting lemmas in \Cref{Appd: A.5}.

	\subsection{Average Suboptimaility Gap Decomposition} \label{Appd: A.1}
	\begin{lemma}\label{lemma: regret decomposition} We denote $\pi_h^{\star,k}(\cdot|s)= \arg\max_{\pi}V^{\pi}_{P^{\star,k},r^k}$. Then the average dynamic suboptimality gap has the following decomposition:
		\begin{align}
			\mathrm{Gap_{Ave}}(K)&=\frac{1}{K}\sum_{k=1}^{K}V^{\star}_{P^{\star,k},r^k}-V^{\pi^k}_{P^{\star,k},r^k}\nonumber\\
			&=\frac{1}{K}\sum_{k=1}^{K}\sum_{h\in[H]}\Eb_{(s_h,a_h) \sim (P^{\star,k},\pi^{\star,k})} \left[\left\{P^{\star,k}_h-\hP_h^k\right\}V_{h+1,\hP^k,r^k}^{\pi^{k}}\right]\nonumber\\
			&+\frac{1}{K}\sum_{k=1}^{K}\sum_{h\in H}\Eb_{s_h \sim (P^{\star,k},\pi^{\star,k})} \left[\langle Q_{h,
				\hP^k,r^k}^{\pi^k}(s_h,\cdot),\pi_h^{\star,k}(\cdot|s_h)-\pi_h^{k}(\cdot|s_h)\rangle\right]\nonumber\\
			&+
			\frac{1}{K}\sum_{k=1}^{K}{V}_{\hP^k,r^k}^{\pi^{k}}-V^{\pi^k}_{P^{\star,k},r^k} \label{Eq: Regret Decomposition}
		\end{align}
	\end{lemma}
	\begin{proof}
		For any function $f : \Sc \times \Ac \rightarrow \Rb$ and any $(k,h,s) \in [K] \times [H] \times \Sc$, define the following operators:
		\begin{align*}
			(\Jb_{k,h}^\star f)(s)= \langle f(s,\cdot),\pi_h^{\star,k}(\cdot|s)\rangle, \quad (\Jb_{k,h} f)(s)= \langle f(s,\cdot),\pi_h^{k}(\cdot|s)\rangle.
		\end{align*}
		We next consider the following decomposition:
		\begin{align}\label{eq:vgap}
			V^{\star}_{P^{\star,k},r^k}-V^{\pi^k}_{P^{\star,k},r^k}= 
			\underbrace{V^{\star}_{P^{\star,k},r^k}-{V}_{\hP^k,r^k}^{\pi^{k}}}_{G_1}+{{V}_{\hP^k,r^k}^{\pi^{k}}-V^{\pi^k}_{P^{\star,k},r^k}},
		\end{align}
            The term $G_1$ can be bounded as follows:
		\begin{align}
			G_1 
			&= V^{\star}_{P^{\star,k},r^k}-{V}_{\hP^k,r^k}^{\pi^{k}}\nonumber\\
			& = \left(\Jb_{k,1}^\star Q_{1,P^{\star,k},r^k}^{\pi^{\star,k}}\right)-\left(\Jb_{k,1} Q_{1,\hP^k,r^k}^{\pi^{k}}\right)\nonumber\\
			& = (\Jb_{k,1}^\star (Q_{1,P^{\star,k},r^k}^{\pi^{\star,k}}-Q_{1,\hP^k,r^k}^{\pi^{k}}))+((\Jb_{k,1}^\star-\Jb_{k,1})Q_{1,\hP^k,r^k}^{\pi^{k}})\nonumber\\
			& = (\Jb_{k,1}^\star (r^k_1(s,\cdot)+P^{\star,k}_1V_{2,P^{\star,k},r^k}^{\pi^{\star,k}}-(r^k_1(s,\cdot)+\hP_1^kV_{2,\hP^k,r^k}^{\pi^{k}}) ))+((\Jb_{k,1}^\star-\Jb_{k,1})Q_{1,\hP^k,r^k}^{\pi^{k}})\nonumber\\
			& =(\Jb_{k,1}^\star (P^{\star,k}_1V_{2,P^{\star,k},r^k}^{\pi^{\star,k}}-\hP_1^kV_{2,\hP^k,r^k}^{\pi^{k}} ))+((\Jb_{k,1}^\star-\Jb_{k,1})Q_{1,\hP^k,r^k}^{\pi^{k}})\nonumber\\
			& =\left(\Jb_{k,1}^\star \left(P^{\star,k}_1\left\{V_{2,P^{\star,k},r^k}^{\pi^{\star,k}}-V_{2,\hP^k,r^k}^{\pi^{k}}\right\}+\left\{P^{\star,k}_1-\hP_1^k\right\}V_{2,\hP^k,r^k}^{\pi^{k}} \right)\right)+((\Jb_{k,1}^\star-\Jb_{k,1})Q_{1,\hP^k,r^k}^{\pi^{k}})\nonumber\\
			& = \left(\Jb_{k,1}^\star \left\{P^{\star,k}_1-\hP_1^k\right\}V_{2,\hP^k,r^k}^{\pi^{k}}\right)
			+\Eb_{s_2 \sim (P^{\star,k},\pi^{\star,k})}\left[V_{2,P^{\star,k},r^k}^{\pi^{\star,k}}(s_2)-V_{2,\hP^k,r^k}^{\pi^{k}}(s_2)\right]
			+((\Jb_{k,1}^\star-\Jb_{k,1})Q_{1,\hP^k,r^k}^{\pi^{k}})\nonumber\\
			& = \sum_{h\in[H]}\Eb_{(s_h,a_h) \sim (P^{\star,k},\pi^{\star,k})} \left[\left\{P^{\star,k}_h-\hP_h^k\right\}V_{h+1,\hP^k,r^k}^{\pi^{k}}\right]\nonumber\\
			& +\sum_{h\in [H]}\Eb_{s_h\sim (P^{\star,k},\pi^{\star,k})} \left[\langle Q_{h,
				\hP^k,r^k}^{\pi^k}(s_h,\cdot),\pi_h^{\star,k}(\cdot|s_h)-\pi_h^{k}(\cdot|s_h)\rangle\right] \label{Eq: G1_decomposition}
		\end{align}	
  Substituting the above result to \Cref{eq:vgap} completes the proof.
	\end{proof}

	\subsection{First and Third Terms of $\mathrm{Gap_{Ave}}$ in \Cref{Eq: Regret Decomposition}: Model Estimation Error Bound} \label{Appd: A.2}

	\subsubsection{First Term in \Cref{Eq: Regret Decomposition}}
	\begin{lemma}\label{Lemma: First Term in Decomposition}
		With probability at least $1-\delta$, we have 
		\begin{align*}
			&\frac{1}{K}\sum_{k=1}^{K}\sum_{h\in[H]}\Eb_{(s_h,a_h) \sim (P^{\star,k},\pi^{\star,k})} \left[\left\{P^{\star,k}_h-\hP_h^k\right\}V_{h+1,\hP,r}^{\pi^{k},k}\right]\\
			&\quad {\leq} O\left(\frac{H}{K}\left[\sqrt{KdA(A\log(|\Phi||\Psi|KH/\delta)+d^2)}\left[H\sqrt{\frac{Kd}{W}\log(W)}+\sqrt{HW^2\Delta_{[H],[K]}^{{\phi}}}\right]+\sqrt{W^3AC_B}\Delta^{\sqrt{P}}_{[H],[K]}\right]\right)\; .
		\end{align*}
	\end{lemma}
	\begin{proof}
 We proceed the proof by deriving the bound:
		\begin{align*}
			& \frac{1}{K}\sum_{k=1}^{K}\sum_{h\in[H]}\Eb_{(s_h,a_h) \sim (P^{\star,k},\pi^{\star,k})} \left[\left\{P^{\star,k}_h-\hP_h^k\right\}V_{h+1,\hP,r}^{\pi^{k},k}\right]\\ 
			& \quad \leq \frac{1}{K}\sum_{k=1}^{K}\sum_{h\in[H]}\Eb_{(s_h,a_h) \sim (P^{\star,k},\pi^{\star,k})} \left[f_h^k(s_h,a_h)\right]\\
			& \quad = \frac{1}{K}\sum_{k=1}^{K}\Eb_{a_1 \sim \pi^{\star,k}} \left[f_1^k(s_1,a_1)\right]+\frac{1}{K}\sum_{k=1}^{K}\sum_{h=2}^H\Eb_{(s_h,a_h) \sim (P^{\star,k},\pi^{\star,k})} \left[f_h^k(s_h,a_h)\right]\\
			& \quad = \frac{1}{K}\sum_{k=1}^{K}\Eb_{a_1 \sim \pi^{\star,k}} \left[f_1^k(s_1,a_1)\right]+\frac{1}{K}\sum_{k=1}^{K}\sum_{h=2}^H\Eb_{(s_h,a_h) \sim (\hP^k,\pi^{\star,k})} \left[f_h^k(s_h,a_h)\right]\\
			& \quad +\frac{1}{K}\sum_{k=1}^{K}\sum_{h=2}^H\left\{\Eb_{(s_h,a_h) \sim (P^{\star,k},\pi^{\star,k})} \left[f_h^k(s_h,a_h)\right]-\Eb_{(s_h,a_h) \sim (\hP^k,\pi^{\star,k})} \left[f_h^k(s_h,a_h)\right]\right\}\\
			& \quad \overset{\RM{1}}{\leq}       \frac{2}{K}\sum_{h=2}^{H}\sum_{k=1}^{K}\hV_{\hP^k,\hb^k}^{\pi^{\star,k}}+2\sum_{h=2}^{H}\left[\frac{1}{K}\sum_{k=1}^{K}\sqrt{WA\left(\zeta_{k,W}+\frac{1}{2}C_B\Delta^P_{1,[k-W,k-1]}\right)}\right.\\
            & \qquad \left.+ \frac{1}{K}\sum_{k=1}^{K}\sum_{\ph=2}^{h}\sqrt{\frac{1}{2d}WAC_B\Delta^P_{[\ph-1,\ph],[k-W,k-1]}}\right] \\
			& \quad \leq       \frac{2H}{K}\sum_{k=1}^{K}\hV_{\hP^k,\hb^k}^{\tpi^{k}}+\frac{2H}{K}\left[\sum_{k=1}^{K}\sqrt{WA\left(\zeta_{k,W}+\frac{1}{2}C_B\Delta^P_{1,[k-W,k-1]}\right)}\right.\\
            & \qquad \left.+\sum_{k=1}^{K} \sum_{h=2}^{H}\sqrt{\frac{1}{2d}WAC_B\Delta^P_{[h-1,h],[k-W,k-1]}}\right]\\
			& \quad \overset{\RM{2}}{\leq}  \frac{2H}{K}\sum_{k=1}^{K}\hV_{\hP^k,\hb^k}^{\tpi^{k}}+ \frac{2H}{K}\sum_{k=1}^{K}\sqrt{WA\zeta_{k,W}} +\frac{2H}{K}\sqrt{W^3AC_B}\Delta_{[H],[K]}^{\sqrt{P}}\\
			& \quad \overset{\RM{3}}{\leq} O\left(\frac{H}{K}\left[\sqrt{KdA(A\log(|\Phi||\Psi|KH/\delta)+d^2)}\left[H\sqrt{\frac{Kd}{W}\log(W)}+\sqrt{HW^2\Delta_{[H],[K]}^{{\phi}}}\right]+\sqrt{W^3AC_B}\Delta^{\sqrt{P}}_{[H],[K]}\right]\right), 
		\end{align*}
		where $\RM{1}$ follows from \Cref{lemma: Step_back_Bounded TV_hP,lemma: bounded difference of value function}, $\RM{2}$ follows because $\sqrt{a+b}\leq \sqrt{a}+\sqrt{b}, \forall a,b \geq 0$ and $\sum_{k=1}^K\Delta^{\sqrt{P}}_{\{h\},[k-W,k-1]} \leq W\Delta^{\sqrt{P}}_{\{h\},[K]}$, and $\RM{3}$ follows from \Cref{Lemma: value function summation}.
	\end{proof}
	\subsubsection{Third Term in \Cref{Eq: Regret Decomposition}}
 \begin{lemma}\label{Lemma: Third Term in Decomposition}
 	With probability at least $1-\delta$, we have 
		\begin{align*}
		    	&\frac{1}{K}\sum_{k=1}^{K} \Big[{V}_{\hP^k,r^k}^{\pi^{k}}-V^{\pi^k}_{P^{\star,k},r^k}\Big]\\
		    	&\quad  \leq O\left(\frac{1}{K}\sqrt{KdA(A\log(|\Phi||\Psi|KH/\delta)+d^2)}\left[H\sqrt{\frac{Kd}{W}\log(W)}+\sqrt{HW^2\Delta_{[H],[K]}^{{\phi}}}\right]+\sqrt{W^3AC_B}\Delta^{\sqrt{P}}_{[H],[K]}\right)\; .
		\end{align*}
	\end{lemma}
	\begin{proof}
		We define the model error as $f^k_h(s_h,a_h)=\norm{\sP_h(\cdot|s_h,a_h)-\hP_h(\cdot|s_h,a_h)}_{TV}$. We next derive the following bound:
  \begin{align*}
      &\frac{1}{K}\sum_{k=1}^{K}\left[{V}_{\hP^k,r^k}^{\pi^{k}}-V^{\pi^k}_{P^{\star,k},r^k} \right]\\
      & \quad
			\overset{\RM{1}}{\leq} \frac{1}{K}\sum_{k=1}^{K}\hV_{\hP^k,\hb^k}^{\pi^k}+ \frac{1}{K}\sum_{k=1}^{K}\sum_{h=2}^{H}\sqrt{\frac{3}{\lambda_{W}}WAC_B\Delta^P_{\{h-1,h\},[k-W,k-1]}}\\
     & \qquad +\frac{1}{K}\sum_{k=1}^{K}\sqrt{A\left(\zeta_{k,W}+2C_B\Delta^P_{1,[k-W,k-1]}\right)}\\
           & \quad \overset{\RM{2}}{\leq}       \frac{1}{K}\sum_{k=1}^{K}\hV_{\hP^k,\hb^k}^{\tpi^{k}}+ \frac{1}{K}\sum_{k=1}^{K}\sum_{h=2}^{H}\sqrt{\frac{3}{\lambda_{W}}WAC_B\Delta^P_{\{h-1,h\},[k-W,k-1]}}\\
     & \qquad + \frac{1}{K}\sum_{k=1}^{K}\sqrt{A\left(\zeta_{k,W}+2C_B\Delta^P_{1,[k-W,k-1]}\right)}\\
			& \quad \overset{\RM{3}}{\leq}  O\left(\frac{1}{K}\sqrt{KdA(A\log(|\Phi||\Psi|KH/\delta)+d^2)}\left[H\sqrt{\frac{Kd}{W}\log(W)}+\sqrt{HW^2\Delta_{[H],[K]}^{{\phi}}}\right]\right.\\
     & \qquad \left. +\sqrt{W^3AC_B}\Delta^{\sqrt{P}}_{[H],[K]}\right), 
  \end{align*}
  where $\RM{1}$ follows from \Cref{lemma: bounded difference of value function}, $\RM{2}$ follows from the definition of $\tpi^k$, and $\RM{3}$ follows from \Cref{Lemma: value function summation}.
	\end{proof}
	\subsection{Second Term of $\mathrm{Gap_{Ave}}$ in \Cref{Eq: Regret Decomposition}: Performance Difference Bound} \label{Appd: A.3}
	The second term in \Cref{lemma: regret decomposition} $\frac{1}{K}\sum_{k=1}^{K}\sum_{h\in H}\Eb_{s_h \sim (P^{\star,k},\pi^{\star,k})} [\langle Q_{h,
		\hP^k,r^k}^{\pi^k}(s_h,\cdot),\pi_h^{\star,k}(\cdot|s_h)-\pi_h^{k}(\cdot|s_h)\rangle]$ can be further decomposed as
	\begin{align}
    	&\frac{1}{K}\sum_{k=1}^{K}\sum_{h\in H}\Eb_{s_h \sim (P^{\star,k},\pi^{\star,k})} \left[\langle Q_{h,
		\hP^k,r^k}^{\pi^k}(s_h,\cdot),\pi_h^{\star,k}(\cdot|s_h)-\pi_h^{k}(\cdot|s_h)\rangle\right]\nonumber\\
		& = \frac{1}{K}\sum_{l \in [L]} \sum_{k=(l-1)\tau+1}^{l\tau}\sum_{h \in [H]} \Eb_{P^{\star,k},\pi^{\star,k}}\left[\langle \hat{Q}^{k}_h(s_h,\cdot), \pi_h^{\star,k}(\cdot|s_h)-\pi_h^{k}(\cdot|s_h)\rangle\right]\nonumber\\
		&  = \underbrace{\frac{1}{K}\sum_{l \in [L]} \sum_{k=(l-1)\tau+1}^{l\tau}\sum_{h \in [H]} \Eb_{P^{\star,(l-1)\tau+1},\pi^{\star,(l-1)\tau+1}}\left[\langle \hat{Q}^{k}_h(s_h,\cdot), \pi_h^{\star,k}(\cdot|s_h)-\pi_h^{k}(\cdot|s_h)\rangle\right]}_{(a)}\nonumber\\
		& +\underbrace{\frac{1}{K}\sum_{l \in [L]} \sum_{k=(l-1)\tau+1}^{l\tau}\sum_{h \in [H]} \left(\Eb_{P^{\star,k},\pi^{\star,k}}-\Eb_{P^{\star,(l-1)\tau+1},\pi^{\star,(l-1)\tau+1}}\right)\left[\langle \hat{Q}^{k}_h(s_h,\cdot), \pi_h^{\star,k}(\cdot|s_h)-\pi_h^{k}(\cdot|s_h)\rangle\right]}_{(b)}.\label{ineq: Performance Difference Bound Decompose}
	\end{align}
	
	\subsubsection{Bound $(a)$ in  \Cref{ineq: Performance Difference Bound Decompose}}
	We first present the following lemma of the descent result introduced in \cite{DBLP:conf/icml/CaiYJW20}. 
	\begin{lemma}\label{lemma: step descent}
		For any distribution $p^\star$ and $p$ supported on $\Ac$ and state $s\in\Sc$, and function $Q: \Sc\times\Ac \rightarrow [0,H]$, it holds for a distribution $p^\prime$ supported on $\Ac$ with $p^\prime(\cdot) \propto p(\cdot) \cdot \exp^{\eta\cdot Q(s,\cdot)}$ that
		\begin{align*}
			\langle Q(s,\cdot), p^\star(\cdot)-p(\cdot)\rangle \leq \frac{1}{2}\eta + \frac{1}{\eta}\left[D_{KL}(p^\star(\cdot)\|p(\cdot))-D_{KL}(p^\star(\cdot)\|\pp(\cdot))\right]. 
		\end{align*}
	\end{lemma}

	\begin{lemma}\label{Lemma: Performance Difference Bound Decompose_a}
Term $(a)$ in  \Cref{ineq: Performance Difference Bound Decompose}	satisfies the following bound:    
		\begin{align*}
			& \sum_{l \in [L]}\sum_{k=(l-1)\tau+1}^{l\tau}\sum_{h \in [H]} \Eb_{P^{\star,(l-1)\tau+1},\pi^{\star,(l-1)\tau+1}}\left[\langle \hat{Q}^{k}_h(s_h,\cdot), \pi_h^{\star,k}(\cdot|s_h)-\pi_h^{k}(\cdot|s_h)\rangle\right]\\
			& \quad \leq \frac{1}{2}\eta KH+\frac{1}{\eta}LH\log A+\tau \Delta_{[H],[K]}^{\pi}.
		\end{align*}    
	\end{lemma}
	\begin{proof}
		We first decompose $(a)$ into two parts:
		\begin{align*}
			&\sum_{l \in [L]}\sum_{k=(l-1)\tau+1}^{l\tau}\sum_{h \in [H]} \Eb_{P^{\star,(l-1)\tau+1},\pi^{\star,(l-1)\tau+1}}\left[\langle \hat{Q}^{k}_h(s_h,\cdot), \pi_h^{\star,k}(\cdot|s_h)-\pi_h^{k}(\cdot|s_h)\rangle\right]\\
			& \quad = \underbrace{\sum_{l \in [L]}\sum_{k=(l-1)\tau+1}^{l\tau}\sum_{h \in [H]} \Eb_{P^{\star,(l-1)\tau+1},\pi^{\star,(l-1)\tau+1}}\left[\langle \hat{Q}^{k}_h(s_h,\cdot), \pi_h^{\star,(l-1)\tau+1}(\cdot|s_h)-\pi_h^{k}(\cdot|s_h)\rangle\right]}_{(I)}\\
			& \quad + \underbrace{\sum_{l \in [L]}\sum_{k=(l-1)\tau+1}^{l\tau}\sum_{h \in [H]} \Eb_{P^{\star,(l-1)\tau+1},\pi^{\star,(l-1)\tau+1}}\left[\langle \hat{Q}^{k}_h(s_h,\cdot), \pi_h^{\star,k}(\cdot|s_h)-\pi_h^{\star,(l-1)\tau+1}(\cdot|s_h)\rangle\right]}_{(II)}
		\end{align*}
		\textbf{I) Bound the term (I)}\newline
		By \Cref{lemma: step descent}, we have 
		\begin{align*}
			(I) 
			&\leq \frac{1}{2}\eta KH+ \sum_{h\in [H]}\frac{1}{\eta} \sum_{l \in [L]} \Eb_{P^{\star,(l-1)\tau+1},\pi^{\star,(l-1)\tau+1}}\\
			&\times \left[\sum_{k=(l-1)\tau+1}^{l\tau}\left[D_{KL}\left(\pi^{\star,(l-1)\tau+1}(\cdot|s_h)\|\pi^k(\cdot|s_h)\right)-D_{KL}\left(\pi^{\star,(l-1)\tau+1}(\cdot|s_h)\|\pi^{k+1}(\cdot|s_h)\right)\right]\right]\\
			&\leq \frac{1}{2}\eta KH+ \sum_{h\in [H]}\frac{1}{\eta}\sum_{l \in [L]} \Eb_{P^{\star,(l-1)\tau+1},\pi^{\star,(l-1)\tau+1}} \\
			&\times \left[D_{KL}\left(\pi^{\star,(l-1)\tau+1}(\cdot|s_h)\|\pi^{(l-1)\tau+1}(\cdot|s_h)\right)-D_{KL}\left(\pi^{\star,(l-1)\tau+1}(\cdot|s_h)\|\pi^{l\tau+1}(\cdot|s_h)\right)\right]\\
			&\leq \frac{1}{2}\eta KH+ \sum_{h\in [H]}\frac{1}{\eta} \sum_{l \in [L]} \Eb_{P^{\star,(l-1)\tau+1},\pi^{\star,(l-1)\tau+1}}\left[D_{KL}\left(\pi^{\star,(l-1)\tau+1}(\cdot|s_h)\|\pi^{(l-1)\tau+1}(\cdot|s_h)\right)\right]\\
			&\leq \frac{1}{2}\eta KH+ \frac{1}{\eta}LH\log A,
		\end{align*}
                  where the last equation follows because 
			\begin{align*}
				D_{KL}\left(\pi^{\star,(l-1)\tau+1}(\cdot|s_h)\|\pi^{(l-1)\tau+1}(\cdot|s_h)\right)
            &=\sum_{a\in \Ac}\pi^{\star,(l-1)\tau+1}(a|s_h)\log(A \cdot \pi^{(l-1)\tau+1}(\cdot|s_h))\\
           &= \log A + \sum_{a} \pi^{\star,(l-1)\tau+1}(a_h|s_h)\cdot \log \pi^{\star,(l-1)\tau+1}(a_h|s_h)\\
           & \leq \log A.
			\end{align*}
		\textbf{II) Bound the term (II)}\newline
		\begin{align*}
			(II) 
			& \leq \sum_{l \in [L]}\sum_{k=(l-1)\tau+1}^{l\tau}\sum_{h \in [H]} \Eb_{P^{\star,(l-1)\tau+1},\pi^{\star,(l-1)\tau+1}}\left[\norm{ \pi_h^{\star,k}(\cdot|s_h)-\pi_h^{\star,(l-1)\tau+1}(\cdot|s_h)}_{TV}\right]\\
			& \leq \sum_{l \in [L]}\sum_{k=(l-1)\tau+1}^{l\tau}\sum_{h \in [H]}\sum_{t=(l-1)\tau+2}^{k} \Eb_{P^{\star,(l-1)\tau+1},\pi^{\star,(l-1)\tau+1}}\left[\norm{ \pi_h^{\star,t}(\cdot|s_h)-\pi_h^{\star,t-1}(\cdot|s_h)}_{TV}\right]\\
			& \leq \sum_{l \in [L]}\sum_{k=(l-1)\tau+1}^{l\tau}\sum_{t=(l-1)\tau+2}^{k} \sum_{h \in [H]} \max_{s \in \Sc}\left[\norm{ \pi_h^{\star,t}(\cdot|s)-\pi_h^{\star,t-1}(\cdot|s)}_{TV}\right]\\
			& \leq \sum_{l \in [L]}\sum_{k=(l-1)\tau+1}^{l\tau}\sum_{t=(l-1)\tau+1}^{l\tau} \sum_{h \in [H]} \max_{s \in \Sc} \left[\norm{ \pi_h^{\star,t}(\cdot|s)-\pi_h^{\star,t-1}(\cdot|s)}_{TV}\right]\\
			& \leq \tau \sum_{k\in[K]} \sum_{h \in [H]} \max_{s \in \Sc} \left[\norm{ \pi_h^{\star,k}(\cdot|s)-\pi_h^{\star,k-1}(\cdot|s)}_{TV}\right]\\
			& \leq \tau \Delta_{[H],[K]}^{\pi}.
		\end{align*}
	\end{proof}
	\subsubsection{Bound $(b)$ in  \Cref{ineq: Performance Difference Bound Decompose}}
	\begin{lemma}\label{Lemma: Performance Difference Bound Decompose_b}
The term $(b)$ in  \Cref{ineq: Performance Difference Bound Decompose} can be bounded as follows:
  \begin{align*}
			& \sum_{l \in [L]} \sum_{k=(l-1)\tau+1}^{l\tau}\sum_{h \in [H]} \left(\Eb_{P^{\star,k},\pi^{\star,k}}-\Eb_{P^{\star,(l-1)\tau+1},\pi^{\star,(l-1)\tau+1}}\right)\left[\langle \hat{Q}^{k}_h(s_h,\cdot), \pi_h^{\star,k}(\cdot|s_h)-\pi_h^{k}(\cdot|s_h)\rangle\right]\\
            & \qquad \leq 2H\tau (\Delta_{[H],[K]}^P+\Delta_{[H],[K]}^{\pi}).
		\end{align*}
	\end{lemma}
	\begin{proof}
		Denote the indicator function of state $s_h$ as $\Ib(s_h)$, and then we have
		\begin{align}
			&\sum_{l \in [L]} \sum_{k=(l-1)\tau+1}^{l\tau}\sum_{h \in [H]} \left(\Eb_{P^{\star,k},\pi^{\star,k}}-\Eb_{P^{\star,(l-1)\tau+1},\pi^{\star,(l-1)\tau+1}}\right)\left[\langle \hat{Q}^{k}_h(s_h,\cdot), \pi_h^{\star,k}(\cdot|s_h)-\pi_h^{k}(\cdot|s_h)\rangle\right]\nonumber\\
			& \quad \leq \sum_{l \in [L]} \sum_{k=(l-1)\tau+1}^{l\tau}\sum_{h \in [H]}\int_{s_h} \left|\Pb_{P^{\star,k}}^{\pi^{\star,k}}(s_h)-\Pb_{P^{\star,(l-1)\tau+1}}^{\pi^{\star,(l-1)\tau+1}}(s_h)\right|ds_h\nonumber\\
			& \quad \leq \sum_{l \in [L]} \sum_{k=(l-1)\tau+1}^{l\tau}\sum_{h \in [H]}\sum_{t=(l-1)\tau+2}^{k} \int_{s_h} \left|\Pb_{P^{\star,t}}^{\pi^{\star,t}}(s_h)-\Pb_{P^{\star,t-1}}^{\pi^{\star,t-1}}(s_h)\right|ds_h, \label{Eq: Lemma13-1}    
		\end{align}
where $\Pb_P^\pi(s)$ denotes the visitation probability at state $s$ under model $P$ and policy $\pi$.
  
Consider $\int_{s_h} \left|\Pb_{P^{\star,t}}^{\pi^{\star,t}}(s_h)-\Pb_{P^{\star,t-1}}^{\pi^{\star,t-1}}(s_h)\right|ds_h$ can be further decomposed as
\begin{align*}
    \int_{s_h} \left|\Pb_{P^{\star,t}}^{\pi^{\star,t}}(s_h)-\Pb_{P^{\star,t-1}}^{\pi^{\star,t-1}}(s_h)\right|ds_h \leq \int_{s_h}\left|\Pb_{P^{\star,t}}^{\pi^{\star,t-1}}(s_h)-\Pb_{P^{\star,t-1}}^{\pi^{\star,t-1}}(s_h)\right|ds_h+ \int_{s_h} \left|\Pb_{P^{\star,t}}^{\pi^{\star,t}}(s_h)-\Pb_{P^{\star,t}}^{\pi^{\star,t-1}}(s_h)\right|ds_h.
\end{align*}

For the first term $\int_{s_h}\left|\Pb_{P^{\star,t}}^{\pi^{\star,t-1}}(s_h)-\Pb_{P^{\star,t-1}}^{\pi^{\star,t-1}}(s_h)\right|ds_h$,
\begin{align}
    &\int_{s_h}\left|\Pb_{P^{\star,t}}^{\pi^{\star,t-1}}(s_h)-\Pb_{P^{\star,t-1}}^{\pi^{\star,t-1}}(s_h)\right|ds_h\nonumber\\
    & \leq \int_{s_h}\sum_{i=1}^h\left|{(P_1^{\star,t})}^{\pi_1^{\star,t-1}}\ldots{(P_i^{\star,t})}^{\pi_i^{\star,t-1}}\ldots{(P_{h-1}^{\star,t-1})}^{\pi_{h-1}^{\star,t-1}}(s_h)-{(P_1^{\star,t})}^{\pi_1^{\star,t-1}}\ldots{(P_i^{\star,t-1})}^{\pi_i^{\star,t-1}}\ldots{(P_{h-1}^{\star,t-1})}^{\pi_{h-1}^{\star,t-1}}(s_h)\right|ds_h\nonumber\\
    & \scriptstyle\leq \int_{s_h}\sum_{i=1}^h\left|\int_{s_2,\ldots, s_{h-1}}\left|{(P_i^{\star,t})}^{\pi_i^{\star,t-1}}(s_{i+1}|s_i)-{(P_i^{\star,t-1})}^{\pi_i^{\star,t-1}}(s_{i+1}|s_i)\right|\prod_{j=1}^{i-1}{(P_j^{\star,t})}^{\pi_j^{\star,t-1}}(s_{j+1}|s_j)\prod_{j=i+1}^{h-1}{(P_j^{\star,t-1})}^{\pi_j^{\star,t-1}}(s_{j+1}|s_j)ds_2\ldots ds_{h-1}\right|ds_h\nonumber\\
     & \overset{\RM{1}}{\leq}  \int_{s_h}\sum_{i=1}^h\left|\int_{s_2,\ldots,s_i,s_{i+2},\ldots s_{h-1}}\int_{s_{i+1}}\left|{(P_i^{\star,t})}^{\pi_i^{\star,t-1}}(s_{i+1}|s_i)-{(P_i^{\star,t-1})}^{\pi_i^{\star,t-1}}(s_{i+1}|s_i)\right|\right.\nonumber\\
     & \qquad \qquad \qquad \left. \max_{s_{i+1}\in\Sc}\prod_{j=1}^{i-1}{(P_j^{\star,t})}^{\pi_j^{\star,t-1}}(s_{j+1}|s_j)\prod_{j=i+1}^{h-1}{(P_j^{\star,t-1})}^{\pi_j^{\star,t-1}}(s_{j+1}|s_j)ds_2\ldots ds_{h-1}\right|ds_h\nonumber\\
      & \overset{\RM{2}}{\leq} \int_{s_h}\sum_{i=1}^h\left|\int_{s_2,\ldots,s_{i-1},s_{i+2},\ldots s_{h-1}}\max_{s_{i}\in\Sc}\int_{s_{i+1}}\left|{(P_i^{\star,t})}^{\pi_i^{\star,t-1}}(s_{i+1}|s_i)-{(P_i^{\star,t-1})}^{\pi_i^{\star,t-1}}(s_{i+1}|s_i)\right|\right.\nonumber\\
     & \qquad \qquad \qquad \left. \int_{s_{i}}\max_{s_{i+1}\in\Sc}\prod_{j=1}^{i-1}{(P_j^{\star,t})}^{\pi_j^{\star,t-1}}(s_{j+1}|s_j)\prod_{j=i+1}^{h-1}{(P_j^{\star,t-1})}^{\pi_j^{\star,t-1}}(s_{j+1}|s_j)ds_2\ldots ds_{h-1}\right|ds_h\nonumber\\
     & \overset{\RM{3}}{\leq} \int_{s_h}\sum_{i=1}^h \left|\max_{(s,a) \in \Sc\times\Ac}\norm{P_i^{\star,t}(\cdot|s,a)-P_i^{\star,t-1}(\cdot|s,a)}_{TV}\right.\nonumber\\
     &  \quad \left. \underbrace{\int_{s_1,\ldots,s_{i}}\prod_{j=1}^{i-1}{(P_j^{\star,t})}^{\pi_j^{\star,t-1}}(s_{j+1}|s_j)ds_1\ldots d s_{i}}_{=1}\underbrace{\int_{s_{i+2},\ldots,s_{h-1}}\max_{s_{i+1}\in\Sc}\prod_{j=i+1}^{h-1}{(P_j^{\star,t-1})}^{\pi_j^{\star,t-1}}(s_{j+1}|s_j)ds_{i+2}\ldots d s_{h-1}}_{\leq 1}\right|ds_h\nonumber\\
     &\leq \sum_{h \in [H]}\max_{(s,a) \in \Sc\times\Ac}\norm{P_i^{\star,t}(\cdot|s,a)-P_i^{\star,t-1}(\cdot|s,a)}_{TV}, \label{Eq: lemma13-P} 
\end{align}
where $\RM{1}$ and $\RM{2}$ follow from Holder’s inequality, and $\RM{3}$ follows from the definition of total variation distance.

Similarly for the second term and from \Cref{Lemma: lemma 5 in Fei} 
\begin{align}
    \int_{s_h} \left|\Pb_{P^{\star,t}}^{\pi^{\star,t}}(s_h)-\Pb_{P^{\star,t}}^{\pi^{\star,t-1}}(s_h)\right|ds_h \leq \sum_{h \in[H]}\max_{s \in \Sc}\norm{\pi_h^{\star,t}(\cdot|s)-\pi_h^{\star,t-1}(\cdot|s)}_{TV}. \label{Eq: lemma13-pi} 
\end{align} 
Plug \Cref{Eq: lemma13-P} and \Cref{Eq: lemma13-pi} into \Cref{Eq: Lemma13-1}, we have 
		\begin{align*}
		&\sum_{l \in [L]} \sum_{k=(l-1)\tau+1}^{l\tau}\sum_{h \in [H]} \left(\Eb_{P^{\star,k},\pi^{\star,k}}-\Eb_{P^{\star,(l-1)\tau+1},\pi^{\star,(l-1)\tau+1}}\right)\left[\langle \hat{Q}^{k}_h(s_h,\cdot), \pi_h^{\star,k}(\cdot|s_h)-\pi_h^{k}(\cdot|s_h)\rangle\right]\\ 
		& \quad \leq 2\sum_{l \in [L]} \sum_{k=(l-1)\tau+1}^{l\tau}\sum_{h \in [H]}\sum_{t=(l-1)\tau+2}^{k} \left(\sum_{i \in[H]}\max_{s \in \Sc}\norm{\pi_i^{\star,t}(\cdot|s)-\pi_i^{\star,t-1}(\cdot|s)}_{TV}\right.\\
        & \qquad \left. +\sum_{i \in[H]}\max_{(s,a)\in \Sc \times \Ac}\norm{P_i^{\star,t}(\cdot|s,a)-P_i^{\star,t-1}(\cdot|s,a)}_{TV}\right)\\ 
		& \quad \leq 2\sum_{h \in [H]}\left(\sum_{l \in [L]} \sum_{k=(l-1)\tau+1}^{l\tau}\sum_{t=(l-1)\tau+1}^{l\tau} \sum_{i\in[H]}\left(\max_{s \in \Sc}\norm{\pi_i^{\star,t}(\cdot|s)-\pi_i^{\star,t-1}(\cdot|s)}_{TV}\right.\right.\\
        & \qquad \left.\left.+\max_{(s,a)\in \Sc \times \Ac}\norm{P_i^{\star,t}(\cdot|s,a)-P_i^{\star,t-1}(\cdot|s,a)}_{TV}\right)\right)\\ 
		& \quad = 2H\tau (\Delta_{[H],[K]}^P+\Delta_{[H],[K]}^{\pi}). 
		\end{align*}
	\end{proof}
	\subsubsection{Combining $(a)$ and $(b)$ Together}
	\begin{lemma} \label{Lemma: Second Term in Decomposition}
		Let $\eta = \sqrt{\frac{L \log A}{K}}$, and we have 
		\begin{align*}
			& \sum_{l \in [L]} \sum_{k=(l-1)\tau+1}^{l\tau}\sum_{h \in [H]} \Eb_{\pi^{\star,k}}\left[\langle \hat{Q}^{\pi^k,k}_h(s_h,\cdot), \pi_h^{\star,k}(\cdot|s_h)-\pi_h^{k}(\cdot|s_h)\rangle\right]\\
           & \qquad \leq 2HK\sqrt{\log A/ \tau} + 3H\tau (\Delta_{[H],[K]}^P+\Delta_{[H],[K]}^{\pi}).
		\end{align*}        
	\end{lemma}
	\begin{proof}
 We derive the following bound:
   \begin{align*}
			&\sum_{k=1}^{K}\sum_{h\in H}\Eb_{s_h \sim (P^{\star,k},\pi^{\star,k})} \left[\langle Q_{h,
				\hP^k,r^k}^{\pi^k}(s_h,\cdot),\pi_h^{\star,k}(\cdot|s_h)-\pi_h^{k}(\cdot|s_h)\rangle\right]\\
			& \quad \overset{\RM{1}}{\leq} \frac{1}{2}\eta KH+\frac{1}{\eta}LH\log A+\tau \Delta_{[H],[K]}^{\pi}+ 2H\tau (\Delta_{[H],[K]}^P+\Delta_{[H],[K]}^{\pi})\\
			& \quad \overset{\RM{2}}{\leq} 2H\sqrt{K L\log A} + 3H\tau (\Delta_{[H],[K]}^P+\Delta_{[H],[K]}^{\pi})\\
			& \quad = 2HK\sqrt{\log A/ \tau} + 3H\tau (\Delta_{[H],[K]}^P+\Delta_{[H],[K]}^{\pi}),
		\end{align*}
		where $\RM{1}$ follows from \Cref{Lemma: Performance Difference Bound Decompose_a,Lemma: Performance Difference Bound Decompose_b}, and $\RM{2}$ follows from the choice of $\eta=\sqrt{\frac{L \log A}{K}}$.
	\end{proof}
	\subsection{Proof \Cref{Thm1: average dynamic suboptimality gap with known variation}} \label{Appd A.4.: Proof of Thm1}
	\begin{proof}[Proof of \Cref{Thm1: average dynamic suboptimality gap with known variation}]
		Combine \Cref{Lemma: First Term in Decomposition,Lemma: Second Term in Decomposition,Lemma: Third Term in Decomposition,lemma: regret decomposition}, we have
		\begin{align}      
		&\mathrm{Gap_{Ave}}(K)=\frac{1}{K}\sum_{k=1}^{K}V^{\star}_{P^{\star,k},r^k}-V^{\pi^k}_{P^{\star,k},r^k}\nonumber\\
		& \quad  \leq O\left(\frac{H}{K}\left[\sqrt{KdA(A\log(|\Phi||\Psi|KH/\delta)+d^2)}\left[H\sqrt{\frac{Kd}{W}\log(W)}+\sqrt{HW^2\Delta_{[H],[K]}^{{\phi}}}\right]+\sqrt{W^3AC_B}\Delta^{\sqrt{P}}_{[H],[K]}\right]\right)\nonumber\\
		& \quad + O\left( \frac{2H}{\sqrt{\tau}} + \frac{3H\tau}{K} (\Delta_{[H],[K]}^P+\Delta_{[H],[K]}^{\pi})\right)\nonumber\\
		& \quad = \underbrace{\tilde{O}\left(\sqrt{\frac{H^4d^2A}{W}\left(A+d^2\right)}+\sqrt{\frac{H^3dA}{K}\left(A+d^2\right)W^2\Delta_{[H],[K]}^{{\phi}}}+\sqrt{\frac{H^2W^3A}{K^2}}\Delta^{\sqrt{P}}_{[H],[K]}\right)}_{(I)}\nonumber\\
		& \qquad +\underbrace{\tilde{O}\left( \frac{2H}{\sqrt{\tau}}+ \frac{3H\tau}{K} (\Delta_{[H],[K]}^P+\Delta_{[H],[K]}^{\pi})\right)}_{(II)}.\label{Eq: Thm1-Final Bound}
\end{align}
\end{proof}
 	\subsection{Supporting Lemmas} \label{Appd: A.5}
 	We first provide the following concentration results, which is an extension of Lemma 39 in \citet{zanette2021cautiously}.
 	\begin{lemma}\label{lemma: concentration on b}
 		There exists a constant $\lambda_{k,W}=O(d\log(|\Phi|\min\{k,W\}TH/\delta))$, the following inequality holds for any $k\in[K],h\in[H], s_h\in\Sc,a_h\in\Ac$ and $\phi \in \Phi$ with probability at least $1-\delta$:
 		\begin{align*}
 		\frac{1}{5} \left\|\phi_h(s,a)\right\|_{(U_{h,\phi}^{k,W})^{-1}} \leq \left\|\phi_{h}(s,a)\right\|_{(\hat{U}_{h}^{k,W})^{-1}} \leq 3 \left\|\phi_h(s,a)\right\|_{(U_{h,\phi}^{k,W})^{-1}}
 		\end{align*}
 	\end{lemma}
 	The following result follows directly from \Cref{lemma: concentration on b}. 	\begin{corollary}\label{coro:concentration on b}
 		
 		The following inequality holds for any $k\in[K],h\in[H], s_h\in\Sc,a_h\in\Ac$ with probability at least $1-\delta$:
 		\begin{align*}
 		\min\left\{\frac{\tap_{k,W}} {5}\left\|\hphi_{h}^k(s_{h},a_{h})\right\|_{(U_{h,\hphi^k}^{k,W})^{-1}},1\right\} \leq \hb_h^k(s_h,a_h) \leq  3\tap_{k,W} \left\|\hphi_{h}^k(s_{h},a_{h})\right\|_{(U_{h,\hphi^k}^{k,W})^{-1}},
 		\end{align*}
 		where $\tap_{k,W} =5\sqrt{2WA\zeta_{k,W}+\lambda_{k,W} d}$.
 	\end{corollary}
  We next provide the MLE guarantee for nonstationary RL, which shows that under any exploration policy, the estimation error can be bounded with high probability. Differently from Theorem 21 in \citet{DBLP:conf/nips/AgarwalKKS20}, we capture the nonstationarity in the analysis.
	\begin{lemma}[Nonstationary MLE Guarantee]\label{lemma: nonstationary MLE guarantee}
		Given $\delta\in(0,1)$, under \Cref{assumption: bounded ratio,assumption: reachability,assumption: realizability}, let $C_B=\sqrt{\frac{C_B}{p_{\min}}}$, and  consider the transition kernels learned from line 3 in \Cref{Alg: MEPE}. We have the following inequality holds for any $n,h\geq 2$ with probability at least $1-\delta/2$:
		\begin{align}
			\frac{1}{W}\sum_{i= 1\lor(k-W)}^{k-1} \mathop{\Eb}_{s_{h-1}\sim\left(P^{\star,i},\pi^{i}\right)\atop{a_{h-1},a_h\sim \Uc(\Ac)\atop s_h\sim P^{\star,i}\left(\cdot|s_{h-1},a_{h-1}\right)}}\left[f_h^{k}(s_h,a_h)^2\right]
			\leq \zeta_{k,W}+2C_B\Delta^P_{h,[k-W,k-1]}, 
		\end{align} 
  where, $\zeta_{k,W} : = \frac{2\log\left(2|\Phi||\Psi|kH/\delta\right)}{W}$.
		In addition, for $h=1$,
		\begin{align*}
			\mathop{\Eb}_{a_1 \sim \Uc(\Ac)}\left[f_1^{k}(s_1,a_1)^2\right]\leq \zeta_{k,W}+2C_B\Delta^P_{1,[k-W,k-1]}.
		\end{align*}
	\end{lemma}
	
	\begin{proof}[Proof of \Cref{lemma: nonstationary MLE guarantee}] 
		For simplification, we denote $x=(s,a) \in \Xc, \Xc=\Sc \times \Ac$, $y=\ps\in \Yc, \Yc=\Sc$. The model estimation process in Algorithm 1 can be viewed as a sequential conditional probability estimation setting with an instance space $\mathcal{X}$ and a target space $\mathcal{Y}$, where the conditional density is given by $p^i(y | x) = P^{\star,i}( y|x)$ for any $i$. We are given a dataset $D:= \{(x_i,y_i)\}_{i= 1\lor(k-W)}^k$ , where $x_i \sim \Dc_i = \Dc_i(x_{1:i-1},y_{1:i-1})$ and $y_i \sim p^i(\cdot | x_i)$. Let $D^\prime$ denote a tangent sequence $\{(x_i^\prime,y_i^\prime)\}_{i= 1\lor(k-W)}^k$ where $x_i^\prime \sim \Dc_i(x_{1:i-1},y_{1:i-1})$ and $y_i^\prime \sim p^i(\cdot|x_i^\prime)$. Further, we consider a function class $\mathcal{F}=\Phi \times \Psi: (\mathcal{X} \times \mathcal{Y}) \rightarrow \Rb$ and assume that the reachability condition $P^{\star,i} \in \mathcal{F}$ holds for any $i$.	
		
		We first introduce one useful lemma in \cite{DBLP:conf/nips/AgarwalKKS20} to decouple data. 
		\begin{lemma}[Lemma 24 of \cite{DBLP:conf/nips/AgarwalKKS20}]\label{lemma: decouple}
			Let $D$ ba a dataset with at most $W$ samples and $D^\prime$ be the corresponding tangent sequence. Let $L(P,D)=\sum_{i= 1\lor(k-W)}^{k}l(P,(x_i,y_i))$ be any function that decomposes additively across examples where $l$ is any function. Let $\widehat{P}(D)$ be any estimator taking as input random variable $D$ and with range $\mathcal{F}$. Then
			\begin{align*}
				\Eb_{D}\left[\exp\left(L(\widehat{P}(D),D)-\log\Eb_{D^\prime}\left[\exp(L(\widehat{P}(D),D^\prime))\right]-\log|\mathcal{F}|\right)\right] \leq 1.
			\end{align*}
		\end{lemma}
		Suppose $\widehat{f}(D)$ is learned from the following maximum likelihood problem:
		\begin{align}
			\widehat{P}(D):= {\arg\max}_{P \in \mathcal{F}}\sum_{(x_i,y_i)\in D}\log f(x_i,y_i).
		\end{align}
		Combining the Chernoff bound and \Cref{lemma: decouple}, we obtain an exponential tail bound, i.e., with probability at least $1-\delta$, 
		\begin{align}
			-\log\Eb_{D^\prime}\left[\exp(L(\widehat{P}(D),D^\prime))\right] \leq -L(\widehat{P}(D),D)+\log|\mathcal{F}|+\log(1/\delta). \label{ineq: chernoff in lemma A.8}
		\end{align}
		To proceed, we let $L(P,D)=\sum_{i= 1\lor(k-W)}^{k-1} - \frac{1}{2} \log(P^{\star,k}(x_i,y_i)/P(x_i,y_i))$ where $D$ is a dataset $\{(x_i,y_i)\}_{i= 1\lor(k-W)}^k$(and $D^\prime=\{(x_i^\prime,y_i^\prime)\}_{i= 1\lor(k-W)}^k$ is tangent sequence). Then $L(P^{\star,k},D)=0 \leq L(\hP,D)$.
		
		Then, the RHS of \Cref{ineq: chernoff in lemma A.8} can be bounded as
		\begin{align}
			\text{RHS of \Cref{ineq: chernoff in lemma A.8}} & =	\sum_{i= 1\lor(k-W)}^{k} \frac{1}{2} \log(P^{\star,k}(x_i,y_i)/\widehat{P}(x_i,y_i))+\log|\mathcal{F}|+\log(1/\delta) \nonumber \\
			&\leq \log|\mathcal{F}|+\log(1/\delta)={\log\left(|\Phi||\Psi| /\delta\right)} \label{ineq: RHS of Lemma A.8},
		\end{align}
		where the inequality follows because $\widehat{P}$ is MLE and from the assumption of reachability, and the last equality follows because $|\mathcal{F}|=|\Phi||\Psi|$.
  
		Consider for any distribution $p$ and $q$ over a domain $\mathcal{Z}$. Then the Hellinger distance $H^2(p\|q)=\int\left(\sqrt{p(z)}-\sqrt{q(z)}\right)^2dz$ satisfies that 
  \begin{align}
      & H^2(p\|q) \nonumber \\
      & \quad =\int\left(\sqrt{p(z)}-\sqrt{q(z)}\right)^2dz
 =\int p(z)+q(z)-2\sqrt{p(z)}\sqrt{q(z)}dz\nonumber\\
      & \quad =2\left(\int 1-\sqrt{p(z)}\sqrt{q(z)}dz\right)\ =2\left(\int 1-\sqrt{p(z)}\sqrt{q(z)}dz\right)\ = 2\Eb_{z \sim q}\left[1-\sqrt{p(z)/q(z)}\right].\label{Eq: H-distance}
  \end{align}
  
		Next, the LHS of \Cref{ineq: chernoff in lemma A.8} can be bounded as
		\begin{align}
			& \hspace{-0.1in}\text{LHS of \Cref{ineq: chernoff in lemma A.8}}	\nonumber\\
			& \overset{\RM{1}}{=}- \log \Eb_{D^\prime}\left[\exp\left(\sum_{i= 1\lor(k-W)}^{k-1}-\frac{1}{2}\log\left(\frac{P^{\star,k}(x_i^\prime,y_i^\prime)}{\widehat{P}(x_i^\prime,y_i^\prime)}\right)\right)\bigg|D\right]\nonumber\\
			& {=}\sum_{i= 1\lor(k-W)}^{k-1}- \log \Eb_{D}\left[\exp\left(-\frac{1}{2}\log\left(\frac{P^{\star,k}(x_i,y_i)}{\widehat{P}(x_i,y_i)}\right)\right)\right]\nonumber\\
                & {=}\sum_{i= 1\lor(k-W)}^{k-1}- \log \Eb_{D}\left[\sqrt{\frac{\widehat{P}(x_i,y_i)}{P^{\star,k}(x_i,y_i)}}\right]\nonumber\\
			& \overset{(\romannumeral2)}{\geq}  \sum_{i= 1\lor(k-W)}^{k-1}\left\{1-\Eb_{D}\left[\sqrt{\frac{\widehat{P}(x_i,y_i)}{P^{\star,k}(x_i,y_i)}}\right]\right\}\nonumber\\
			& {=} \sum_{i= 1\lor(k-W)}^{k-1}\left\{\Eb_{x_i \sim D}\left[1-\Eb_{y_i  \sim P^{\star,i}(\cdot|x_i)}\left[\sqrt{\frac{\widehat{P}(x_i,y_i)}{P^{\star,k}(x_i,y_i)}}\right]\right]\right\}\nonumber\\
			& {=} \sum_{i= 1\lor(k-W)}^{k-1}\left\{\Eb_{x_i \sim D}\left[1-\Eb_{y_i  \sim P^{\star,k}(\cdot|x_i)}\left[\sqrt{\frac{\widehat{P}(x_i,y_i)}{P^{\star,k}(x_i,y_i)}}\right]\right]\right\}\nonumber\\
			&+\sum_{i= 1\lor(k-W)}^{k-1}\Eb_{x_i \sim D}\left[\Eb_{y_i  \sim P^{\star,k}(\cdot|x_i)}\left[\sqrt{\frac{\widehat{P}(x_i,y_i)}{P^{\star,k}(x_i,y_i)}}\right]-\Eb_{y_i  \sim P^{\star,i}(\cdot|x_i)}\left[\sqrt{\frac{\widehat{P}(x_i,y_i)}{P^{\star,k}(x_i,y_i)}}\right]\right]\nonumber\\
			& \overset{(\romannumeral3)}{\geq} 
			-WC_B\Delta^P_{h,[k-W,k-1]}+ \sum_{i= 1\lor(k-W)}^{k-1}\Eb_{x_i^\prime  \sim \Dc_i}\left[\mathrm{H}^2\left(P^{\star,k}(\cdot|x_i^\prime)\|\hP(\cdot|x_i^\prime)\right)\right]\nonumber\\
			& \overset{(\romannumeral4)}{\geq} 
			-WC_B\Delta^P_{h,[k-W,k-1]}+ \frac{1}{2}\sum_{i= 1\lor(k-W)}^{k-1}\Eb_{x_i  \sim \Dc_i}\left[\norm{P^{\star,k}(x_i,\cdot)-\widehat{P}(x_i,\cdot)}^2_{TV}\right]\label{ineq: LHS of Lemma A.8},
		\end{align}
		where $\RM{1}$ follows from the above definition of $L(f,D)$, $\RM{2}$ follows from the fact that $1-x \leq - \log(x)$, $(\romannumeral3)$ follows from the definition of variation budgets and \Cref{Eq: H-distance}, and $(\romannumeral4)$ follows from the fact that $\norm{p(\cdot)-q(\cdot)}_{TV}^2 \leq \mathrm{H}^2\left(p\|q\right)$ as indicated in Lemma 2.3 in \citet{DBLP:books/daglib/0035708}.

		Combining \Cref{ineq: chernoff in lemma A.8,ineq: RHS of Lemma A.8,ineq: LHS of Lemma A.8}, we have 
		\begin{align}
			\frac{1}{W}\sum_{i= 1\lor(k-W)}^{k-1} \mathop{\Eb}_{s_{h-1}\sim\left(P^{\star,i},\pi^{i}\right)\atop{a_{h-1},a_h\sim \Uc(\Ac)\atop s_h\sim P^{\star,i}\left(\cdot|s_{h-1},a_{h-1}\right)}}\left[f_h^{k}(s_h,a_h)^2\right]
			\leq 2C_B\Delta^P_{h,[k-W,k-1]} +\frac{2\log\left(|\Phi||\Psi| /\delta\right)}{W}
			. \label{ineq: fixed version lemma 3}
		\end{align}
		{We substitute $\delta$ with ${\delta}/{2nH}$ to ensure \Cref{ineq: fixed version lemma 3} holds for any $h \in [H]$ and $n$ with probability at least $1-\delta/2$, which finishes the proof.} 
	\end{proof}
	The following lemma extends Lemmas 12 and 13 under infinite discount MDPs under stationary case in \cite{DBLP:conf/iclr/UeharaZS22} to episodic MDPs under nonstationary environment, which captures nonstationarity in analysis.
	\begin{lemma}[Nonstationary Step Back]\label{lemma:Nonstationary Step_Back} 
		Let $\mathcal{I}=\left[1\lor(k-W),k-1\right]$ be an index set and $\{P_{h-1}^i\}_{i \in \mathcal{I}} = \{\langle\phi^i_{h-1},\mu^i_{h-1}\rangle\}$ be a set of generic MDP model, and $\Pi$ be an arbitrary and possibly mixture policy. Define an expected Gram matrix as follows. For any $\phi \in \Phi$,	
		\[M_{h-1,\phi} = 
		\lambda_{W} I + \sum_{i\in \Ic}\mathop{\Eb}_{s_{h-1}\sim (P^{i,\star},\Pi) \atop a_{h-1}\sim \Pi}\left[\phi_{h-1}(s_{h-1},a_{h-1})\left(\phi_{h-1}(s_{h-1},a_{h-1})\right)^\top\right].
		\]  
		Further, let $f_{h-1}^i(s_{h-1},a_{h-1})$ be the total variation between $P_{h-1}^{i,\star}$ and $P_{h-1}^i$ at time step $h-1$. Suppose $g \in \mathcal{S} \times \mathcal{A} \rightarrow \mathbb{R}$ is bounded by $B\in(0,\infty)$, i.e., $\|g\|_\infty \leq B$. Then, $\forall h \geq 2, \forall\, {\rm policy }\,\pi_h$,
		\begin{align*}
			&\mathop{\Eb}_{s_h \sim P^k_{h-1} \atop a_h \sim \pi_h}[g(s_h,a_h)|s_{h-1},a_{h-1}]  \\
			& \quad\leq \left\|\phi^k_{h-1}(s_{h-1},a_{h-1})\right\|_{(M_{h-1,\phi^k})^{-1}} \times\nonumber\\
			&\qquad 
			\sqrt{\sum_{i\in \Ic} A \mathop{\Eb}_{s_{h}\sim(P^{i,\star},\Pi)\atop a_h\sim \Uc }\left[g(s_h,a_h)^2\right] + WB^2\Delta_{h-1,\Ic}^P + \lambda_{k,W} d B^2 + \sum_{i\in \Ic}AB^2\mathop{\Eb}_{s_{h-1}\sim(P^{i,\star},\Pi) \atop a_{h-1}\sim \Pi }\left[f
				^k_{h-1}(s_{h-1},a_{h-1})^2\right]}.
		\end{align*}
	\end{lemma}

	\begin{proof}
		We first derive the following bound:
		\begin{align*}
			& \mathop{\Eb}_{s_h \sim P^k_{h-1} \atop a_h \sim \pi_h}\left[ 
			g(s_h,a_h)|s_{h-1},a_{h-1}\right]\nonumber\\
			&\qquad=\int_{s_h}\sum_{a_h}g(s_h,a_h)\pi(a_h|s_h)\langle\phi^k_{h-1}(s_{h-1},a_{h-1}),\mu^k_{h-1}(s_h)\rangle d{s_h}\\
			&\qquad \leq \left\|\phi^k_{h-1}(s_{h-1},a_{h-1})\right\|_{(M_{h-1,\phi^k})^{-1}}\left\|\int\sum_{a_h}g(s_h,a_h)\pi(a_h|s_h)\mu^k_{h-1}(s_h)d{s_h}\right\|_{M_{h-1,\phi^k}},
		\end{align*}
		where the inequality follows from Cauchy-Schwarz inequality. We further expand the second term in the RHS of the above inequality as follows.
		\begin{align*}
			&\hspace{-5mm} \left\|\int\sum_{a_h}g(s_h,a_h)\pi(a_h|s_h)\mu^k_{h-1}(s_h)d{s_h}\right\|_{M_{h-1,\phi^k}}^2\\
			& \overset{(\romannumeral1)}{\leq}  \sum_{i\in \Ic} \mathop{\Eb}_{s_{h-1}\sim (P^{i,\star},\Pi) \atop a_{h-1}\sim \Pi}
			\left[\left(\int_{s_h}\sum_{a_h}g(s_h,a_h)\pi_h(a_h|s_h)\mu^k(s_h)^\top\phi^k(s_{h-1},a_{h-1})d{s_h}\right)^2\right] + \lambda_{k,W} d B^2\\
			&= \sum_{i\in \Ic} \mathop{\Eb}_{s_{h-1}\sim(P^{i,\star},\Pi) \atop a_{h-1} \sim \Pi } \left[\left(\mathop{\Eb}_{s_h \sim P^k_{h-1} \atop a_h \sim \pi_h} \left[g(s_h,a_h)\bigg|s_{h-1},a_{h-1}\right]\right)^2\right] + \lambda_{k,W} d B^2\\
			& \overset{(\romannumeral2)}\leq  \sum_{i\in \Ic} \mathop{\Eb}_{s_{h-1}\sim(P^{i,\star},\Pi) \atop a_{h-1} \sim \Pi } \left[\mathop{\Eb}_{s_h \sim P^k_{h-1} \atop a_h \sim \pi_h} \left[g(s_h,a_h)^2\bigg|s_{h-1},a_{h-1}\right]\right] + \lambda_{k,W} d B^2\\
			&\overset{(\romannumeral3)}{\leq}  \sum_{i\in \Ic} \mathop{\Eb}_{s_{h-1}\sim(P^{i,\star},\Pi) \atop a_{h-1}\sim\Pi }\left[\mathop{\Eb}_{s_h\sim P_{h-1}^{i,\star} \atop a_h\sim \pi_h}\left[g(s_h,a_h)^2\bigg|s_{h-1},a_{h-1}\right]\right] + \lambda_{k,W} d B^2 \\
			&\quad + \sum_{i\in \Ic} B^2 \mathop{\Eb}_{s_{h-1}\sim(P^{i,\star},\Pi) \atop a_{h-1}\sim\Pi }\left[\norm{P_{h-1}^{i,\star}(s_{h-1},a_{h-1})-P_{h-1}^{k,\star}(s_{h-1},a_{h-1})}_{TV}\right]\\
			&\quad+ \sum_{i\in \Ic} B^2\mathop{\Eb}_{s_{h-1}\sim(P^{i,\star},\Pi) \atop a_{h-1}\sim\Pi }\left[f^k_{h-1}(s_{h-1},a_{h-1})^2\right]\\
			&\overset{(\romannumeral4)}{\leq}  \sum_{i\in \Ic} A \mathop{\Eb}_{s_{h}\sim(P^{i,\star},\Pi)\atop a_h\sim \Uc }\left[g(s_h,a_h)^2\right]  + \lambda_{k,W} d B^2+ WB^2\Delta_{h-1,\Ic}^P +  \sum_{i\in \Ic}  B^2\mathop{\Eb}_{s_{h-1}\sim(P^{i,\star},\Pi) \atop a_{h-1}\sim \Pi }\left[f^k
			_{h-1}(s_{h-1},a_{h-1})^2\right],
		\end{align*}
		where $(\romannumeral1)$ follows from the assumption that $\|g\|_{\infty}\leq B$, $(\romannumeral2)$ follows from Jensen's inequality, $(\romannumeral3)$ follows because $f^k_{h-1}(s_{h-1},a_{h-1})$ is the total variation between $P_{h-1}^{k,\star}$ and $P^k_{h-1}$ at time step $h-1$, and $(\romannumeral4)$ follows from importance sampling and the definition of $\Delta_{h-1,\Ic}^P$.
		This finishes the proof.
	\end{proof}
 	 Recall that $f_h^k(s,a) =\|\hP_h^{k}(\cdot|s,a) - P^{\star,k}_h(\cdot|s,a)\|_{TV}$. Using \Cref{lemma:Nonstationary Step_Back}, we have the following lemma to bound the expectation of $f_h^k(s,a)$ under estimated transition kernels.
	\begin{lemma}\label{lemma: Step_back_Bounded TV_hP}
		Denote $\alpha_{k,W}=\sqrt{2WA\zeta_{k,W}+\lambda_{k,W} d}$. For any $k \in [K]$, policy $\pi$ and reward $r$, for all $h\geq 2$, we have
		\begin{align}
			\textstyle \mathop{\Eb}_{(s_h,a_h)\sim(\hP^k,\pi)} & \left[f_h^k(s_h,a_h)\bigg|s_{h-1}, a_{h-1}\right] \nonumber \\
			&\textstyle {\leq} \min\left\{\alpha_{k,W}\left\|\hphi_{h-1}^k(s_{h-1},a_{h-1})\right\|_{(U^k_{h-1,\hphi^k})^{-1}},1\right\}+\sqrt{\frac{1}{2d}WAC_B\Delta^P_{[h-1,h],[k-W,k-1]}},
		\end{align}
		and for $h=1$, we have
		\begin{align}	    
			\mathop{\Eb}_{a_1 \sim \pi} \left[ f_1^k(s_1,a_1)\right] {\leq} \sqrt{A\left(\zeta_{k,W}+\frac{1}{2}C_B\Delta^P_{1,[k-W,k-1]}\right)}.
		\end{align}
	\end{lemma}
	
	\begin{proof}
		For $h=1$, we have
		\begin{align*}
			\mathop{\Eb}_{ a_1 \sim \pi} \left[ f_1^k(s_1,a_1)\right] 
			\overset{(\romannumeral1)}{\leq} \sqrt{\mathop{\Eb}_{ a_1 \sim \pi} \left[ f_1^k(s_1,a_1)^2\right]} \overset{(\romannumeral2)}{\leq} \sqrt{A\left(\zeta_{k,W}+{2}C_B\Delta^P_{1,[k-W,k-1]}\right)},
		\end{align*}
		where $(\romannumeral1)$ follows from Jensen's inequality, and $\RM{2}$ follows from the importance sampling.
		
		Then for $h \geq 2$, we derive the following bound:
		\begin{align*}
			&\mathop{\Eb}_{(s_h,a_h)\sim(\hP^k,\pi)}  \left[ 
			f_h^k(s_h,a_h)\bigg|s_{h-1}, a_{h-1}\right]\\
			& \quad \overset{(\romannumeral1)}{\leq}\mathop{\Eb}_{a_{h-1}\sim\pi} \left[\left\|\hphi_{h-1}^k(s_{h-1},a_{h-1})\right\|_{(U^k_{h-1,\hphi^k})^{-1}}\times \left(A \sum_{i= 1\lor(k-W)}^{k-1} \mathop{\Eb}_{s_{h-1}\sim\left(P^{\star,i},\pi^{i}\right)\atop{a_{h-1},a_h\sim \Uc(\Ac)\atop s_h\sim P^{\star,i}\left(\cdot|s_{h-1},a_{h-1}\right)}}\left[f_h^k(s_h,a_h)^2\right] \right. \right.
			\\
			& \quad 
			\left. \left. + W\Delta^P_{h-1,[k-W,k-1]} + \lambda_{k,W} d  + A \sum_{i= 1\lor(k-W)}^{k-1} \mathop{\Eb}_{s_{h-2}\sim\left(P^{\star,i},\pi^{i}\right)\atop{a_{h-2},a_{h-1}\sim \Uc(\Ac)\atop s_{h-1}\sim P^{\star,i}\left(\cdot|s_{h-2},a_{h-2}\right)}}\left[f_{h-1}^k(s_{h-1},a_{h-1})^2\right]\right)^{-\frac{1}{2}}\right]
			\\
			&\quad \overset{(\romannumeral2)}{\leq} \mathop{\Eb}_{ a_{h-1}\sim\pi} \left[\sqrt{WA(\zeta_{k,W}+{2}C_B\Delta^P_{h,[k-W,k-1]})+WA(\zeta_{k,W}+{2}C_B\Delta^P_{h-1,[k-W,k-1]})+W\Delta^P_{h-1,[k-W,k-1]}+\lambda_{k,W} d}\right.\\
            & \qquad \qquad \left.\times \left\|\hphi_{h-1}^k(s_{h-1},a_{h-1})\right\|_{(U^k_{h-1,\hphi^k})^{-1}}  
			\right]\\
			&\quad \overset{\RM{3}}{=} \mathop{\Eb}_{ a_{h-1}\sim\pi} \left[\alpha_{k,W}\left\|\hphi_{h-1}^k(s_{h-1},a_{h-1})\right\|_{(U^k_{h-1,\hphi^k})^{-1}}  
			\right]+\sqrt{\frac{3}{\lambda_{W}}WAC_B\Delta^P_{\{h-1,h\},[k-W,k-1]}},
		\end{align*}
		where $(\romannumeral1)$ follows from \Cref{lemma:Nonstationary Step_Back} and because $|f_h^k(s_h,a_h)|\leq 1$; specially the first term inside the square root follows from the definition of $U_{h-1,\widehat{\phi}^k}^k $, the third term inside the square root follows from the importance sampling; $(\romannumeral2)$ follows from \Cref{lemma: nonstationary MLE guarantee}, and $\RM{3}$ follows because $\sqrt{a+b} \leq \sqrt{a}+\sqrt{b}, \forall a,b \geq 0$ and $\norm{\hphi_{h-1}^k(s_{h-1},a_{h-1})}_{(U^k_{h-1,\hphi^k})^{-1}} \leq \sqrt{\frac{1}{\lambda_{W}}}$.
	\end{proof}
 The next lemma follows a similar argument to that of \Cref{lemma: Step_back_Bounded TV_hP} with the only difference being the expectation over which $f_h^k(s_h,a_h)$ takes.
	\begin{lemma}\label{lemma: Step_back_Bounded TV_sP}
		Denote $\alpha_{k,W}=\sqrt{2WA\zeta_{k,W}+\lambda_{k,W} d}$. For any $k \in [K]$, policy $\pi$ and reward $r$, for all $h\geq 2$, we have
		\begin{align}
			\textstyle &\mathop{\Eb}_{(s_h,a_h)\sim(P^{\star,k},\pi)}  \left[f_h^k(s_h,a_h)\bigg|s_{h-1}, a_{h-1}\right] \nonumber \\
			& \qquad \qquad \textstyle {\leq} \min\left\{\alpha_{k,W}\left\|\phi^{\star,k}_{h-1}(s_{h-1},a_{h-1})\right\|_{(U^k_{h-1,\phi^{\star,k}})^{-1}},1\right\}+\sqrt{\frac{1}{2d}WAC_B\Delta^P_{\{h-1,h\},[k-W,k-1]}},
		\end{align}
		and for $h=1$, we have
		\begin{align}	    
			\mathop{\Eb}_{a_1 \sim \pi} \left[ f_1^k(s_1,a_1)\right] {\leq} \sqrt{A\left(\zeta_{k,W}+\frac{1}{2}C_B\Delta^P_{1,[k-W,k-1]}\right)}.
		\end{align}
	\end{lemma}
	
	\begin{proof}
		For $h=1$,
		\begin{align*}
			\mathop{\Eb}_{ a_1 \sim \pi} \left[ f_1^k(s_1,a_1)\right] 
			\overset{(\romannumeral1)}{\leq} \sqrt{\mathop{\Eb}_{ a_1 \sim \pi} \left[ f_1^k(s_1,a_1)^2\right]} \overset{(\romannumeral2)}{\leq} \sqrt{A\left(\zeta_{k,W}+\frac{1}{2}C_B\Delta^P_{1,[k-W,k-1]}\right)},
		\end{align*}
		where $(\romannumeral1)$ follows from Jensen's inequality, and $\RM{2}$ follows from the importance sampling.
		
		Then for $h \geq 2$, we derive the following bound:
		\begin{align*}
			&\mathop{\Eb}_{(s_h,a_h)\sim(P^{\star,k},\pi)}  \left[ 
			f_h^k(s_h,a_h)\bigg|s_{h-1}, a_{h-1}\right]\\
			&  \overset{(\romannumeral1)}{\leq}\mathop{\Eb}_{a_{h-1}\sim\pi} \left[\left\|\phi^{\star,k}_{h-1}(s_{h-1},a_{h-1})\right\|_{(U^k_{h-1,\phi^{\star,k}})^{-1}}\times 
			\sqrt{A \sum_{i= 1\lor(k-W)}^{k-1} \mathop{\Eb}_{s_{h-1}\sim\left(P^{\star,i},\pi^{i}\right)\atop{a_{h-1},a_h\sim \Uc(\Ac)\atop s_h\sim P^{\star,k}\left(\cdot|s_{h-1},a_{h-1}\right)}}\left[f_h^k(s_h,a_h)^2\right] + \lambda_{k,W} d   }\right]
			\\
			&\quad  \overset{(\romannumeral2)}{\leq} \mathop{\Eb}_{ a_{h-1}\sim\pi} \left[\sqrt{wA(\zeta_{k,W}+\frac{1}{2}C_B\Delta^P_{h,[k-W,k-1]})+wA(\zeta_{k,W}+\frac{1}{2}C_B\Delta^P_{h-1,[k-W,k-1]})+A C_B\Delta^P_{h,[k-W,k-1]}+\lambda_{k,W} d}\right.\\
               &\qquad\left. \times \left\|\phi^{\star,k}_{h-1}(s_{h-1},a_{h-1})\right\|_{(U^k_{h-1,\phi^{\star,k}})^{-1}}  
			\right]\\
			&\quad \overset{\RM{3}}{=} \mathop{\Eb}_{ a_{h-1}\sim\pi} \left[\alpha_{k,W}\left\|\phi^{\star,k}_{h-1}(s_{h-1},a_{h-1})\right\|_{(U^k_{h-1,\phi^{\star,k}})^{-1}}  
			\right]+\sqrt{\frac{1}{2d}WAC_B\Delta^P_{[h-1,h],[k-W,k-1]}},
		\end{align*}
		where $(\romannumeral1)$ follows from \Cref{lemma:Nonstationary Step_Back} and because $|f_h^k(s_h,a_h)|\leq 1$, the first term inside the square root follows from the definition of $U_{h-1,\widehat{\phi}^k}^k $, the third term inside the square root follows from the importance sampling, and $(\romannumeral2)$ follows from \Cref{lemma: nonstationary MLE guarantee}.

		The proof is completed by noting that $|f_h^k(s_h,a_h)|\leq 1$.
	\end{proof}

	The following lemma is a direct application of \Cref{lemma: Step_back_Bounded TV_hP}. By this lemma, we can show that the difference of value functions can be bounded by truncated value function plus a variation term.
	\begin{lemma}[{Bounded} difference of value functions]\label{lemma: bounded difference of value function}
		For $k\in[K]$, $\delta \geq 0$ any policy $\pi$ and reward $r$, with probability at least $1-\delta$, we have
		\begin{align}
			& \left| V_{P^{\star,k},r}^{\pi}-V_{\hP^k,r}^{\pi}\right|\nonumber\\
   & \quad
			\leq \hV_{\hP^k,\hb^k}^{\pi}+ \sum_{h=2}^{H}\sqrt{\frac{3}{\lambda_{W}}WAC_B\Delta^P_{\{h-1,h\},[k-W,k-1]}}+\sqrt{A\left(\zeta_{k,W}+2C_B\Delta^P_{1,[k-W,k-1]}\right)} \label{ineq: difference of value function bound}.
		\end{align}
	\end{lemma}
	\begin{proof}
		\textbf{(1):} We first show that $\left| V_{P^{\star,k},r}^{\pi}-V_{\hP^k,r}^{\pi}\right|
		\leq \hV_{\hP^k,f^k}^{\pi}$.
		
		Recall the definition of the estimated value functions $\hat{V}_{h,\hP^k , r}^{\pi} (s_h)$ and $\hat{Q}_{h,\hP^k, r}^{\pi}(s_h,a_h)$ for policy $\pi$:
		\begin{align*}
			&\hat{Q}_{h,\hP^k,r}^{\pi}(s_h,a_h) = \min\left\{1, r_h(s_h,a_h) + \hP_h^k\hV_{h+1,\hP^k,r}^{\pi}(s_h,a_h)\right\},\\
			&\hat{V}_{h,\hP^k,r}^{\pi}(s_h) =  \mathop{\Eb}_{\pi}\left[\hat{Q}^{\pi}_{h,\hP^k,r}(s_h,a_h)\right].
		\end{align*}
		We develop the proof by backward induction.
		
		When $h=H+1$, we have $\left|V_{H+1,P^{\star,k},r}^{\pi}(s_{H+1})-V_{H+1,\hP^k,r}^{\pi}(s_{H+1})\right| = 0  = \hat{V}_{H+1,\hP^k,f^k}^{\pi}(s_{H+1})$.
		
		Suppose that for $h+1$, $\left| V^{\pi}_{h+1,P^{\star,k},r}(s_{h+1})-V_{h+1,\hP^k,r}^{\pi}(s_{h+1})\right| \leq \hV_{h+1,\hP^k,f^k}^{\pi}(s_{h+1})$ holds for any $s_{h+1}$.
		
		Then, for $h$, by Bellman equation, we have,
		\begin{align}
			& \bigg| Q^{\pi}_{P^{\star,k},r}(s_h,a_h)-Q^{\pi}_{h,\hP^k,r}(s_h,a_h) \bigg| \nonumber\\
			& \quad = \bigg|P_h^{\star,k}V^{\pi}_{h+1,P^{\star,k},r}(s_h,a_h)-\hP_h^k V^{\pi}_{h,\hP^k,r}(s_h,a_h)\bigg| \nonumber\\
			& \quad = \bigg| \hP_h^k\left( V^{\pi}_{h+1,P^{\star,k},r}-V_{h+1,\hP^k,r}^{\pi}\right)(s_h,a_h) + \left( P_h^{\star,k}-\hP_h^k \right)V^{\pi}_{h,P^{\star,k},r}(s_h,a_h)\bigg| \nonumber\\
			& \quad \overset{(\romannumeral1)}\leq \min\bigg\{ 1, f_h^k(s_h,a_h) + \hP_h^k \bigg| V^{\pi}_{h+1,P^{\star,k},r}-V_{h+1,\hP^k,r}^{\pi}\bigg|(s_h,a_h) \bigg\}\nonumber\\
			& \quad \overset{(\romannumeral2)}\leq \min\bigg\{ 1, f_h^k(s_h,a_h) + \hP_h^k \hV_{h+1,\hP^k,f^k}^{\pi}(s_h,a_h) \bigg\}\nonumber\\
			& \quad = \hat{Q}_{h,\hP^k,f^k}^{\pi}(s_h,a_h), \label{eqn:lowrank:hatQ-Q<hatQ}
		\end{align}
		where $(\romannumeral1)$ follows because $\norm{\hP_h^k(\cdot|s_h,a_h) - P^{\star,k}_h(\cdot|s_h,a_h)}_{TV} =  f_h^k(s_h,a_h)$ and the value function is at most 1, and $(\romannumeral2)$ follows from the induction hypothesis.
		
		Then, by the definition of $\hV^{\pi}_{h,\hP^k,r}(s_h)$, we have
		\begin{align*}
			\bigg| V_{h, \hP^k,r}^{\pi}&(s_h) - V_{h, P^{\star,k},r}^{\pi}(s_h)\bigg|\\
			& = \bigg| \mathop{\Eb}_{\pi}\left[Q_{h,\hP^k, r}^{\pi} (s_h,a_h)\right] - \mathop{\Eb}_{\pi}\left[Q_{h,P^{\star,k},r}^{\pi}(s_h,a_h)\right]\bigg| \\
			& \leq   \mathop{\Eb}_{\pi}\left[\bigg|Q_{h,\hP^k, r}^{\pi} (s_h,a_h) - Q_{h,P^{\star,k},r}^{\pi}(s_h,a_h)\bigg|\right] \\
			&\overset{(\romannumeral1)}\leq  \mathop{\Eb}_{\pi}\left[\hat{Q}_{h,\hP^k, f^k}^{\pi} (s_h,a_h) \right] \\
			&= \hV_{h,\hP^k,f^k}^{\pi}(s_h),
		\end{align*}
		where $(\romannumeral1)$ follows from \Cref{eqn:lowrank:hatQ-Q<hatQ}.
		
		Therefore, by induction, we have 
		\begin{align*}
			\left| V_{P^{\star,k},r}^{\pi}-V_{\hP^k,r}^{\pi}\right|\leq
			\hV_{\hP^k,f^k}^{\pi}.
		\end{align*}
		\textbf{(2):} Then we prove that $$\hV_{\hP^k,f^k}^{\pi}\leq \hV_{\hP^k,\hb^k}^{\pi}+\sum_{h=2}^{H}\sqrt{\frac{3}{\lambda_{W}}WAC_B\Delta^P_{\{h-1,h\},[k-W,k-1]}}+\sqrt{A\left(\zeta_{k,W}+2C_B\Delta^P_{1,[k-W,k-1]}\right)}.$$
		
		By \Cref{lemma: Step_back_Bounded TV_hP} and the fact that the total variation distance is upper bounded by 1, $\forall h\geq2$, with probability at least $1-\delta/2$, we have 
		\begin{align}
			\mathop{\Eb}_{\hP^k,\pi}\left[f_h^k(s_h,a_h)\bigg| s_{h-1}\right] &\leq \mathop{\Eb}_{\pi }\left[\min\left(\alpha_{k,W}\left\|\hphi_{h-1}^k(s_{h-1},a_{h-1})\right\|_{(U^k_{h-1,\hphi^k})^{-1}}, 1\right) \right] \nonumber \\
           &+\sqrt{\frac{3}{\lambda_{W}}WAC_B\Delta^P_{\{h-1,h\},[k-W,k-1]}}. \label{ineq: f_1 bound}
		\end{align}
		Similarly, when $h=1$, 
		\begin{align}
			\mathop{\Eb}_{a_1\sim\pi}\left[f_1^k(s_1,a_1)\right]\leq \sqrt{A\left(\zeta_{k,W}+{2}C_B\Delta^P_{1,[k-W,k-1]}\right)}.\label{ineq:f1<K_zeta}
		\end{align}
		Based on \Cref{coro:concentration on b}, \Cref{ineq: f_1 bound} and $\tap_{k,W} = 5 \alpha
		_{k,W}$, we have
		\begin{align}
		&	\mathop{\Eb}_{\pi} \left[\hb^k_h(s_h,a_h)\bigg|s_h\right]+\sqrt{\frac{3}{\lambda_{W}}WAC_B\Delta^P_{\{h,h+1\},[k-W,k-1]}} \nonumber\\
            &\quad \geq \mathop{\Eb}_{\pi}\left[\min\left(\alpha_{k,W}\left\|\hphi_{h}^k(s_h,a_h)\right\|_{(U_{h,\hphi^k}^k)^{-1}}, 1\right) \right]+\sqrt{\frac{3}{\lambda_{W}}WAC_B\Delta^P_{\{h-1,h\},[k-W,k-1]}}\nonumber\\
            & \quad \geq \mathop{\Eb}_{\hP^k,\pi }\left[f_{h+1}^k(s_{h+1},a_{h+1})\bigg|s_h\right] .\label{ineq:f<b}
		\end{align}

		For the base case $h= H$, we have 
		
		\begin{align*}
			&\hspace{-0.1in}\mathop{\Eb}_{\hP^k,\pi}\left[\hV_{H,\hP^k,f^k}^{\pi}(s_H)\bigg|s_{h-1}, a_{h-1}\right]\nonumber\\ 
			&= \mathop{\Eb}_{\hP^k, \pi}\left[f_H^k(s_H,a_H)\bigg| s_{H-1}\right]\\
			&\leq \mathop{\Eb}_{\pi}\left[b_{H-1}^k(s_{H-1},a_{H-1})|s_{H-1}\right]+\sqrt{\frac{3}{\lambda_{W}}WAC_B\Delta^P_{\{H-1,H\},[k-W,k-1]}}\\
			&\leq \min\left\{1, \mathop{\Eb}_{\pi}\left[\hat{Q}_{H-1,\hP^k,\hb^k}^{\pi}(s_{H-1},a_{H-1})\bigg|s_{h-1}, a_{h-1}\right]\right\}+\sqrt{\frac{3}{\lambda_{W}}WAC_B\Delta^P_{\{H-1,H\},[k-W,k-1]}}\\
			& = \hV^{\pi}_{H-1,\hP^k,\hb^k}(s_{H-1})+\sqrt{\frac{3}{\lambda_{W}}WAC_B\Delta^P_{\{H-1,H\},[k-W,k-1]}}.
		\end{align*}
		
		For any step $h+1, h\geq 2$, assume that $\mathop{\Eb}_{\hP^k, \pi}\left[\hV_{h+1,\hP^k,f^k}^{\pi}(s_{h+1})\bigg|s_{h}\right]\leq \hV_{h,\hP^k, \hb^k}^{\pi}(s_h)+\sum_{\ph=h+1}^{H}\sqrt{\frac{3}{\lambda_{W}}WAC_B\Delta^P_{\{\ph-1,\ph\},[k-W,k-1]}}$ holds . Then, by Jensen's inequality, we obtain
		\begin{align*}
			\mathop{\Eb}_{\hP^k, \pi}&\bigg[\hV_{h,\hP^k,f^k}^{\pi}(s_h)\bigg|s_{h-1}, a_{h-1}\bigg] \\
			& \leq \min\left\{1, \mathop{\Eb}_{\hP^k, \pi}\left[ f_h^k(s_h,a_h) + \hP_h^k\hV_{h+1,\hP^k,f^k}^{\pi}(s_h,a_h)\bigg|s_{h-1}, a_{h-1}\right]\right\}\\
			& \overset{(\romannumeral1)}\leq \min\left\{1, \mathop{\Eb}_{\pi}\left[ \hb_{h-1}^k(s_{h-1},a_{h-1})\right] + \sqrt{\frac{3}{\lambda_{W}}WAC_B\Delta^P_{\{h-1,h\},[k-W,k-1]}}\right.\\
   & \qquad \left.+\mathop{\Eb}_{\hP^k,\pi}\left[\mathop{\Eb}_{\hP^k,\pi} \left[ \hV_{h+1,\hP^k,f^k}^{\pi}(s_{h+1})\bigg|s_{h}\right]\bigg|s_{h-1}, a_{h-1}\right]\right\}\\
			& \overset{(\romannumeral2)}\leq \min\left\{1, \mathop{\Eb}_{\pi}\left[ b_{h-1}^k(s_{h-1},a_{h-1})\right]+ \sqrt{\frac{3}{\lambda_{W}}WAC_B\Delta^P_{\{h-1,h\},[k-W,k-1]}}\right.\\
			& \left. \qquad+ \mathop{\Eb}_{\hP^k,\pi} \left[ \hV_{h,\hP^k,\hb^k}^{\pi}(s_{h})\bigg|s_{h-1}, a_{h-1}\right]+\sum_{\ph=h+1}^{H}\sqrt{\frac{3}{\lambda_{W}}WAC_B\Delta^P_{\{\ph-1,\ph\},[k-W,k-1]}}\right\}\\
			& = \min\left\{1, \mathop{\Eb}_{\pi}\left[\hat{Q}_{h-1,\hP^k,\hb^k}^{\pi}(s_{h-1},a_{h-1})\right] \right\}+\sum_{\ph=h}^{H}\sqrt{\frac{3}{\lambda_{W}}WAC_B\Delta^P_{\{\ph-1,\ph\},[k-W,k-1]}}\\
			& = \hV^{\pi}_{h-1,\hP^k,\hb^k}(s_{h-1})+\sum_{\ph=h}^{H}\sqrt{\frac{3}{\lambda_{W}}WAC_B\Delta^P_{\{\ph-1,\ph\},[k-W,k-1]}},
		\end{align*}
		where $(\romannumeral1)$ follows from \Cref{ineq:f<b}, and $(\romannumeral2)$ is due to the induction hypothesis.

		By induction, we conclude that
		\begin{align*}
			\hV_{\hP^k,f^k}^{\pi} & = \mathop{\Eb}_{\pi}\left[f_1^{(s)}(s_1,a_1)\right] + \mathop{\Eb}_{\hP^k,\pi}\left[\hV_{2,\hP^k,f^k}^{\pi}(s_2)\bigg|s_1\right]\\
			&\leq \sqrt{A\left(\zeta_{k,W}+{2}C_B\Delta^P_{1,[k-W,k-1]}\right)} + \hV_{\hP^k,\hb^k}^{\pi}+\sum_{\ph=2}^{H}\sqrt{\frac{3}{\lambda_{W}}WAC_B\Delta^P_{\{\ph-1,\ph\},[k-W,k-1]}}.
		\end{align*}
		
		Combining Step 1 and Step 2, we conclude that
  \begin{align*}
        &\left|V_{P^*,r}^{\pi} - V_{\hP^k,r}^{\pi}\right| \\
        &\quad\leq \hV_{\hP^k,\hb^k}^{\pi}+\sqrt{A\left(\zeta_{k,W}+{2}C_B\Delta^P_{1,[k-W,k-1]}\right)} + \sum_{\ph=2}^{H}\sqrt{\frac{3}{\lambda_{W}}WAC_B\Delta^P_{\{\ph-1,\ph\},[k-W,k-1]}}.
  \end{align*}		
	\end{proof}	
    Similarly to \Cref{lemma: bounded difference of value function}, we can prove that the total variance distance is bounded by $\hV_{\hP^{k},\hb^{k}}^{\tpi_k}$ plus a variation budget term as follows. \Cref{lemma: bounded difference of value function,lemma: V_b bound TV} together justify the choice of exploration policy for the off-policy exploration.
    \begin{lemma}\label{lemma: V_b bound TV}
	Fix $\delta \in (0,1)$, for any $h \in [H], k \in [K]$, any policy $\pi$, with probability at least $1-\delta/2$, 
	\begin{align*}
		&\mathop{\Eb}_{s_h \sim (P^{\star,k},\pi)\atop s_h \sim \pi}\left[f_h^k(s_h,a_h)\right]\\
  & \quad \leq 2\left(\hV_{\hP^k,\hb^k}^{\pi}+\sum_{h=2}^{H}\sqrt{\frac{3}{\lambda_{W}}WAC_B\Delta^P_{\{h-1,h\},[k-W,k-1]}}+\sqrt{A\left(\zeta_{k,W}+2C_B\Delta^P_{1,[k-W,k-1]}\right)}\right).
	\end{align*}
\end{lemma}
\begin{proof}
	Fix any policy $\pi$, for any $h \geq 2$, we have 
	\begin{align}
		&\hspace{-0.3in}\mathop{\Eb}_{s_{h} \sim (\hP^k,\pi)\atop a_{h} \sim \pi}\left[\hat{Q}_{h,\hP^k,\hb^k}^{\pi}(s_{h},a_{h})\right]\nonumber\\
		&=\mathop{\Eb}_{s_{h-1} \sim (\hP^k,\pi)\atop a_{h-1} \sim \pi}\left[\hP^k_{h}\hat{V}_{h,\hP^k,\hb^k}^{\pi}(s_{h-1},a_{h-1})\right]\nonumber\\
		&\leq \mathop{\Eb}_{s_{h-1} \sim (\hP^k,\pi)\atop a_{h-1} \sim \pi}\left[\min\left\{1,\hb^k_{h-1}(s_{h-1},a_{h-1})+\hP^k_{h-1}\hat{V}_{h,\hP^k,\hb^k}^{\pi}(s_{h-1},a_{h-1})\right\}\right]\nonumber\\
		&=\mathop{\Eb}_{s_{h-1} \sim (\hP^k,\pi)\atop a_{h-1} \sim \pi}\left[\hat{Q}_{h-1,\hP^k,\hb^k}^{\pi}(s_{h-1},a_{h-1})\right]\nonumber\\
		& \leq \ldots \nonumber\\
		& \leq \mathop{\Eb}_{ a_{1} \sim \pi}\left[\hat{Q}_{1,\hP^k,\hb^k}^{\pi}(s_{1},a_{1})\right]\nonumber\\
		& = \hat{V}_{\hP^k,\hb^k}^{\pi}.\label{ineq: lemmaB1-1}
	\end{align}
	Hence, for $h \geq 2$, we have
	\begin{align}
		\mathop{\Eb}_{s_h \sim (\hP^k,\pi)\atop a_h \sim \pi}\left[f_h^k(s_h,a_h)\right] &\overset{(\romannumeral1)}{\leq}\mathop{\Eb}_{s_{h-1} \sim (\hP^k,\pi)\atop a_{h-1} \sim \pi}\left[\hb_{h-1}^k(s_{h-1},a_{h-1})\right]+\sqrt{\frac{3}{\lambda_{W}}WAC_B\Delta^P_{\{h-1,h\},[k-W,k-1]}} \nonumber\\
		&\overset{(\romannumeral2)}{\leq}\mathop{\Eb}_{s_{h-1} \sim (\hP^k,\pi)\atop a_{h-1} \sim \pi}\left[\hat{Q}_{h-1,\hP^k,\hb^k}^{\pi}(s_{h-1},a_{h-1})\right]+\sqrt{\frac{3}{\lambda_{W}}WAC_B\Delta^P_{\{h-1,h\},[k-W,k-1]}}\nonumber\\
		&\overset{(\romannumeral3)}{\leq}\hat{V}_{\hP^k,\hb^k}^{\pi}+\sqrt{\frac{3}{\lambda_{W}}WAC_B\Delta^P_{\{h-1,h\},[k-W,k-1]}},\label{ineq: lemmaB1-2}
	\end{align}
	where $(\romannumeral1)$ follows from \Cref{ineq:f<b}, $(\romannumeral2)$ follows from the definition of $\hat{Q}_{h-1,\hP^k,\hb^k}^{\pi}(s_{h-1},a_{h-1})$, and $(\romannumeral3)$ follows from \Cref{ineq: lemmaB1-1}.
	\begin{align*}
		& \hspace{-0.1in}\mathop{\Eb}_{ s_h \sim (P^{\star,k},\pi)\atop s_h \sim \pi}\left[f_h^k(s_h,a_h)\right]\\
		&\leq  \mathop{\Eb}_{s_h \sim (\hP^k,\pi)\atop a_h \sim \pi}\left[f_h^k(s_h,a_h)\right]
        +\left| \mathop{\Eb}_{s_h \sim (P^{\star,k},\pi)\atop a_h \sim \pi}\left[f_h^k(s_h,a_h)\right]-\mathop{\Eb}_{s_h \sim (\hP^k,\pi)\atop a_h \sim \pi}\left[f_h^k(s_h,a_h)\right]\right|\\
		& {\leq} 2\left(\hV_{\hP^k,\hb^k}^{\pi}+\sum_{h=2}^{H}\sqrt{\frac{3}{\lambda_{W}}WAC_B\Delta^P_{\{h-1,h\},[k-W,k-1]}}+\sqrt{A\left(\zeta_{k,W}+2C_B\Delta^P_{1,[k-W,k-1]}\right)}\right) ,
	\end{align*}
	where the last equation follows from \Cref{ineq: lemmaB1-2} and  \Cref{lemma: bounded difference of value function}.
\end{proof}

  	\begin{lemma}\label{lemma: Step_back_Bounded Bonus_sP}
		Denote $\tap_{k,W}=5\alpha_{k,W}$,
		$\alpha_{k,W}=\sqrt{2WA\zeta_{k,W}+\lambda_{k,W} d}$, and $\beta_{k,W}=\sqrt{9dA\alpha_{k,W}^2  + \lambda_{k,W} d }$. For any $k \in [K]$, policy $\pi$ and reward $r$, for all $h\geq 2$, we have
		\begin{align}
			\mathop{\Eb}_{(s_h,a_h)\sim(P^{\star,k},\pi) } & \left[\hb_h^k(s_h,a_h)\bigg|s_{h-1}, a_{h-1}\right]  \nonumber\\
           & {\leq}\beta_{k,W}\left\|\phi^{\star,k}_{h-1}(s_{h-1},a_{h-1})\right\|_{(W^k_{h-1,\phi^{\star,k}})^{-1}}+\sqrt{\frac{A}{d}}\Delta^{\sqrt{P}}_{h-1,[k-W,k-1]},
		\end{align}
		and for $h=1$, we have
		\begin{align}	    
			\mathop{\Eb}_{a_1 \sim \pi} \left[ \hb_1^k(s_1,a_1)\right] {\leq} \sqrt{\frac{9Ad\alpha_{k,W}^2}{W}}.
		\end{align}
	\end{lemma}
	
	\begin{proof}
		For $h=1$,
		\begin{align*}
			\mathop{\Eb}_{ a_1 \sim \pi} \left[ \hb_1^k(s_1,a_1)\right] 
			\overset{(\romannumeral1)}{\leq} \sqrt{\mathop{\Eb}_{ a_1 \sim \pi} \left[ \hb_1^k(s_1,a_1)^2\right]} \overset{(\romannumeral2)}{\leq} \sqrt{\frac{9Ad\alpha_{k,W}^2}{W}},
		\end{align*}
		where $(\romannumeral1)$ follows from Jensen's inequality, and $\RM{2}$ follows from the importance sampling.
		
		 Then for $h \geq 2$, we first notice that
  		\begin{align}
			& \hspace{-0.1in}\sum_{i= 1\lor(k-W)}^{k-1} \mathop{\Eb}_{s_{h}\sim\left(P^{\star,i},\pi^{i}\right)\atop a_h\sim \Uc(\Ac)}\left[\hb_h^k(s_h,a_h)^2\right] \nonumber\\
			& = \sum_{i= 1\lor(k-W)}^{k-1} \mathop{\Eb}_{s_{h}\sim\left(P^{\star,i},\pi^{i}\right)\atop a_h\sim \Uc(\Ac)}\left[\alpha_{k,W}^2\left\|\hphi_{h}^k(s_h,a_h)\right\|^2_{(\hU_{h}^k)^{-1}}\right]\nonumber\\
			& \leq \sum_{i= 1\lor(k-W)}^{k-1} \mathop{\Eb}_{s_{h}\sim\left(P^{\star,i},\pi^{i}\right)\atop a_h\sim \Uc(\Ac)}\left[9\alpha_{k,W}^2\left\|\hphi_{h}^k(s_h,a_h)\right\|^2_{(U_{h,\hphi^k}^k)^{-1}}\right]\nonumber\\
			& \leq 9\alpha_{k,W}^2\textrm{tr}\left(\sum_{i= 1\lor(k-W)}^{k-1} \mathop{\Eb}_{s_{h}\sim\left(P^{\star,i},\pi^{i}\right)\atop a_h\sim \Uc(\Ac)}\left[\hphi_{h}^k(s_h,a_h)\hphi_{h}^k(s_h,a_h)^\top\right](U_{h,\hphi^k}^k)^{-1}\right)\nonumber\\
			& \leq 9\alpha_{k,W}^2\textrm{tr}\left(I_d\right)=9d\alpha_{k,W}^2, \label{Eq: A.6-1}
		\end{align}
            Because $\sqrt{a+b}\leq \sqrt{a}+\sqrt{b}$ and for any $k \in [K]$, $h \in [H]$, $\sqrt{\max_{(s,a)\in \Sc \times \Ac}\norm{P_h^{\star,k+1}(\cdot|s,a)-P_h^{\star,k}(\cdot|s,a)}_{TV}}=\max_{(s,a)\in \Sc \times \Ac}\sqrt{\norm{P_h^{\star,k+1}(\cdot|s,a)-P_h^{\star,k}(\cdot|s,a)}_{TV}}$ , then for any $\Hc,\Ic$, we can convert $\ell_\infty$ variation budgets to square-root $\ell_\infty$ variation budgets.  
            \begin{align}
                \sqrt{\Delta^P_{\Hc,\Ic}}\leq  \Delta^{\sqrt{P}}_{\Hc,\Ic}. \label{Eq: A.6-2}
            \end{align}
		Recall that $W^k_{h,\phi}=\sum_{i= 1\lor(k-W)}^{k-1} \mathop{\Eb}_{ (s_{h},a_h)\sim\left(P^{\star,i},\pi^{i}\right)}\left[\phi_h(s_h,a_h)\phi_h(s_h,a_h)^\top\right]+\lambda_{k,W} I_d$. We derive the following bound:
		\begin{align*}
			&\mathop{\Eb}_{(s_h,a_h)\sim(P^{\star,k},\pi)}  \left[ 
			\hb_h^k(s_h,a_h)\bigg|s_{h-1}, a_{h-1}\right]\\
			& \quad \overset{(\romannumeral1)}{\leq}\mathop{\Eb}_{a_{h-1}\sim\pi} \left[\left\|\phi^{\star,k}_{h-1}(s_{h-1},a_{h-1})\right\|_{(W^k_{h-1,\phi^{\star,k}})^{-1}}\right.\\
			&\quad \quad \left. \times 
			\sqrt{A \sum_{i= 1\lor(k-W)}^{k-1} \mathop{\Eb}_{s_{h-1},a_{h-1}\sim\left(P^{\star,i},\pi^{i}\right)\atop{s_h \sim P_{h-1}^{(\star,i)}(\cdot|s_{h-1},a_{h-1})\atop{a_h \sim \Uc(\Ac)}}}\left[\hb_h^k(s_h,a_h)^2\right]+ A\Delta^P_{h-1,[k-W,k-1]} + \lambda_{k,W} d   }\right]\\
			& \quad \overset{(\romannumeral2)}{\leq}\mathop{\Eb}_{a_{h-1}\sim\pi} \left[\left\|\phi^{\star,k}_{h-1}(s_{h-1},a_{h-1})\right\|_{(W^k_{h-1,\phi^{\star,k}})^{-1}}\times 
			\sqrt{9dA\alpha_{k,W}^2  + \lambda_{k,W} d   }\right]+\sqrt{\frac{A}{d}\Delta^P_{h-1,[k-W,k-1]}}\\
               & \quad \overset{(\romannumeral3)}{\leq}\mathop{\Eb}_{a_{h-1}\sim\pi} \left[\left\|\phi^{\star,k}_{h-1}(s_{h-1},a_{h-1})\right\|_{(W^k_{h-1,\phi^{\star,k}})^{-1}}\times 
			\sqrt{9dA\alpha_{k,W}^2  + \lambda_{k,W} d   }\right]+\sqrt{\frac{A}{d}}\Delta^{\sqrt{P}}_{h-1,[k-W,k-1]}
		\end{align*}
      where $\RM{1}$ follows from \Cref{lemma:Nonstationary Step_Back}, $\RM{2}$ follows from \Cref{Eq: A.6-1}, and $\RM{3}$ follows \Cref{Eq: A.6-2}.
	\end{proof}
        Before next lemma, we first introduce a notion related to matrix norm. For any matrix $A$, $\norm{A}_2$ denotes the matrix norms induced by vector $\ell_2$-norm. Note that $\norm{A}_2$ is also known as the spectral norm of matrix $A$ and is equal to the largest singular value of matrix $A$. 
	\begin{lemma}  \label{Lemma: value function summation}
		With probability at least $1-\delta$, the summation of the truncated value functions $\hV^{{\pi_k}}_{\hP^{k},\hb^{k}}$ under exploration policies $\{{\tpi_k}\}_{k\in[K]}$ is bounded by:
		\begin{align*}
		\sum_{k=1}^K\hV_{\hP^k,\hb^k}^{\tpi^k} &\leq O\left(\sqrt{KdA(A\log(|\Phi||\Psi|KH/\delta)+d^2)}\left[H\sqrt{\frac{Kd}{W}\log(W)}+\sqrt{HW^2\Delta_{[H],[K]}^{{\phi}}}\right]\right.\\
        & \left. \qquad \qquad+\sqrt{W^3AC_B}\Delta^{\sqrt{P}}_{[H],[K]}\right).
		\end{align*}
	\end{lemma}
	\begin{proof}
			For any $k \in [K]$, any policy $\pi$, we have
		\begin{align}
			\hV_{\hP^k,\hb^k}^\pi-V_{P^{\star,k},\hb^{k}}^\pi& \leq \mathop{\Eb}_{\pi}\left[\hP^k_{1}\hV^{\pi}_{2,\hP^k,\hb^k}(s_1,a_1) - P^{\star,k}_1 V^{\pi}_{2,P^{\star,k},\hb^k}(s_1,a_1)\right]\nonumber\\
			& = \mathop{\Eb}_{\pi}\left[\left(\hP^k_{1} - P^{\star,k}_1\right)\hV^{\pi}_{2,\hP^k,\hb^k}(s_1,a_1) + P^{\star,k}_1\left(\hV^{\pi}_{2,\hP^k,\hb^k} - V^{\pi}_{2,P^{\star,k},\hb^k}\right)(s_1,a_1)\right]\nonumber\\
			&\leq \mathop{\Eb}_{\pi}\left[f_1^k(s_1,a_1) + P^{\star,k}_1\left(\hV^{\pi}_{2,\hP^k,\hb^k} - V^{\pi}_{2,P^{\star,k},\hb^k}\right)\right]\nonumber\\
			&\leq \mathop{\Eb}_{\pi}\left[f_1^k(s_1,a_1)\right] + \mathop{\Eb}_{P^{\star,k},\pi}\left[\hV^{\pi}_{2,\hP^k,\hb^k} - V^{\pi}_{2,P^{\star,k},\hb^k}\right]\nonumber\\
			&\leq \mathop{\Eb}_{P^{\star,k}, \pi}\left[\sum_{h=1}^H f^k(s_h,a_h)\right] = V^{\pi}_{P^{\star,k},f^k}, \label{Eq: hb < sb + sf}
		\end{align}
		As a result, we have
	    \begin{align*}
		\sum_{k=1}^K\hV_{\hP^k,\hb^k}^\pi
		\leq \sum_{k=1}^KV^{\pi}_{P^{\star,k},f^k}+\sum_{k=1}^KV_{P^{\star,k},\hb^{k}}^\pi \; .
		\end{align*}
		\textbf{Step 1:} We first bound $\sum_{k=1}^KV^{\tpi^k}_{P^{\star,k},f^k}$ via an \textbf{auxiliary anchor representation}.
  
		Recall that $U^{k,W}_{h,\phi^{\star,k}}=\sum_{i= 1\lor(k-W)}^{k-1} \mathop{\Eb}_{s_{h}\sim\left(P^{\star,i},\tpi^{i}\right)\atop{a_h\sim \Uc(\Ac)}}\left[\phi^{\star,k}_h(s_h,a_h)\phi^{\star,k}_h(s_h,a_h)^\top\right]+\lambda_{k,W} I_d$ and we define $\tU^{k,W,t}_{h,\phi^{\star,k}}=\sum_{i= tW+1}^{k-1} \mathop{\Eb}_{s_{h}\sim\left(P^{\star,i},\tpi^{i}\right)\atop{a_h\sim \Uc(\Ac)}}\left[\phi^{\star,k}_h(s_h,a_h)\phi^{\star,k}_h(s_h,a_h)^\top\right]+\lambda_{k,W} I_d$. We first note that for any $h$, the following equation holds. 
		\begin{align}
			&\sum_{k\in[K]}\Eb_{(s_h,a_h) \sim (P^{\star,k},\tpi^k)}\left[\alpha_{k,W}\left\|\phi^{\star,k}_{h}(s_{h},a_{h})\right\|_{(U^k_{h,\phi^{\star,k}})^{-1}}\right]\nonumber\\
			& \quad = \sqrt{\sum_{k\in[K]}\alpha_{k,W}^2\sum_{k\in[K]}\Eb_{(s_h,a_h) \sim (P^{\star,k},\tpi^k)}\left[\left\|\phi^{\star,k}_{h}(s_{h},a_{h})\right\|_{(U^{k,W}_{h,\phi^{\star,k}})^{-1}}\right]^2}\nonumber\\
                & \quad = \sqrt{\sum_{k\in[K]}\alpha_{k,W}^2\sum_{t=0}^{\lfloor K/W\rfloor}\sum_{k=tW+1}^{(t+1)W} \Eb_{(s_h,a_h) \sim (P^{\star,k},\tpi^k)}\left[\left\|\phi^{\star,k}_{h}(s_{h},a_{h})\right\|_{(U^{k,W}_{h,\phi^{\star,k}})^{-1}}\right]^2} \label{Eq: Lemma7-1}
                \end{align}
                The $\phi^{\star,k}$ and $U$ in \Cref{Eq: Lemma7-1} both change with the round index $k$. To deal with such an issue, we divide the entire round into $\lfloor\frac{K}{W}\rfloor+1$ blocks with an equal length of $W$. For each block $t \in \{0,\ldots,\lfloor\frac{K}{W}\rfloor \}$, we select an \textbf{auxiliary anchor representation} $\phi^{\star,tW+1}$ and decompose \Cref{Eq: Lemma7-1} as follows. We first derive the following equation: 
	        \begin{align}
		     &  \sum_{k=tW+1}^{(t+1)W}\Eb_{(s_h,a_h) \sim (P^{\star,k},\tpi^k)}\left[\left\|\phi_{h}^{\star,k}(s_{h},a_{h})\right\|^2_{(U^{k,W}_{h,\phi^{\star,k}})^{-1}}\right]\nonumber\\
       & \qquad  \qquad -\sum_{k=tW+1}^{(t+1)W}\Eb_{(s_h,a_h) \sim (P^{\star,k},\tpi^k)}\left[\left\|\phi_{h}^{\star,tW+1}(s_{h},a_{h})\right\|^2_{(U^{k,W}_{h,\phi^{\star,tW+1}})^{-1}}\right]\nonumber\\
			& \quad = \sum_{k=tW+1}^{(t+1)W}\Eb_{(s_h,a_h) \sim (P^{\star,k},\tpi^k)}\left[\left\|\phi_{h}^{\star,k}(s_{h},a_{h})\right\|^2_{(U^{k,W}_{h,\phi^{\star,k}})^{-1}}-\left\|\phi_{h}^{\star}(s_{h},a_{h})\right\|^2_{(U^{k,W}_{h,\phi^{\star,tW+1}})^{-1}}\right]\nonumber\\
			& \quad = \sum_{k=tW+1}^{(t+1)W}\Eb_{(s_h,a_h) \sim (P^{\star,k},\tpi^k)}\left[\left\|\phi_{h}^{\star,k}(s_{h},a_{h})\right\|^2_{(U^{k,W}_{h,\phi^{\star,k}})^{-1}}-\left\|\phi_{h}^{\star,k}(s_{h},a_{h})\right\|^2_{(U^{k,W}_{h,\phi^{\star,tW+1}})^{-1}}\right.\nonumber\\
			& \left.\quad +\left\|\phi_{h}^{\star,k}(s_{h},a_{h})\right\|^2_{(U^{k,W}_{h,\phi^{\star,tW+1}})^{-1}}-\left\|\phi_{h}^{\star,tW+1}(s_{h},a_{h})\right\|^2_{(U^{k,W}_{h,\phi^{\star,tW+1}})^{-1}}\right].\nonumber\\
		& \quad = \underbrace{\sum_{k=tW+1}^{(t+1)W}\Eb_{(s_h,a_h) \sim (P^{\star,k},\tpi^k)}\left[\left\|\phi_{h}^{\star,k}(s_{h},a_{h})\right\|^2_{(U^{k,W}_{h,\phi^{\star,k}})^{-1}}-\left\|\phi_{h}^{\star,k}(s_{h},a_{h})\right\|^2_{(U^{k,W}_{h,\phi^{\star,tW+1}})^{-1}}\right]}_{(I)}\nonumber\\
		& \quad + \underbrace{\sum_{k=tW+1}^{(t+1)W}\Eb_{(s_h,a_h) \sim (P^{\star,k},\tpi^k)}\left[\left\|\phi_{h}^{\star,k}(s_{h},a_{h})\right\|^2_{(U^{k,W}_{h,\phi^{\star,tW+1}})^{-1}}-\left\|\phi_{h}^{\star,tW+1}(s_{h},a_{h})\right\|^2_{(U^{k,W}_{h,\phi^{\star,tW+1}})^{-1}}\right]}_{(II)}. \label{Eq: Lemma7-2}
	\end{align}
	\textbf{For term $(II)$}, we have
	\begin{align}
		&\sum_{k=tW+1}^{(t+1)W}\Eb_{(s_h,a_h) \sim (P^{\star,k},\tpi^k)}\left[\left\|\phi_{h}^{\star,k}(s_{h},a_{h})\right\|^2_{(U^{k,W}_{h,\phi^{\star,tW+1}})^{-1}}-\left\|\phi_{h}^{\star,tW+1}(s_{h},a_{h})\right\|^2_{(U^{k,W}_{h,\phi^{\star,tW+1}})^{-1}}\right] \nonumber\\
		& \scriptstyle\quad \leq \sum_{k=tW+1}^{(t+1)W}\Eb_{(s_h,a_h) \sim (P^{\star,k},\tpi^k)}\left[\phi_{h}^{\star,k}(s_{h},a_{h})^\top(U^{k,W}_{h,\phi^{\star,tW+1}})^{-1}\phi_{h}^{\star,k}(s_{h},a_{h})-\phi_{h}^{\star,k}(s_{h},a_{h})^\top(U^{k,W}_{h,\phi^{\star,tW+1}})^{-1}\phi_{h}^{\star,tW+1}(s_{h},a_{h})\right.\nonumber\\
		& \quad + \left.\phi_{h}^{\star,k}(s_{h},a_{h})^\top(U^{k,W}_{h,\phi^{\star,tW+1}})^{-1}\phi_{h}^{\star,tW+1}(s_{h},a_{h})-\phi_{h}^{\star,tW+1}(s_{h},a_{h})^\top(U^{k,W}_{h,\phi^{\star,tW+1}})^{-1}\phi_{h}^{\star,tW+1}(s_{h},a_{h})\right]\nonumber\\
		& \quad \overset{\RM{1}}{\leq} \sum_{k=tW+1}^{(t+1)W}\Eb_{(s_h,a_h) \sim (P^{\star,k},\tpi^k)}\left[\norm{\phi_{h}^{\star,k}(s_{h},a_{h})}_2\norm{(U^{k,W}_{h,\phi^{\star,tW+1}})^{-1}}_2\norm{\phi_{h}^{\star,k}(s_{h},a_{h})-\phi_{h}^{\star,tW+1}(s_{h},a_{h})}_2\right.\nonumber\\
		& \quad + \left.\norm{\phi_{h}^{\star,k}(s_{h},a_{h})-\phi_{h}^{\star,tW+1}(s_{h},a_{h})}_2\norm{(U^{k,W}_{h,\phi^{\star,tW+1}})^{-1}}_2\norm{\phi_{h}^{\star,tW+1}(s_{h},a_{h})}_2\right]\nonumber\\
		& \quad \leq \sum_{k=tW+1}^{(t+1)W}\Eb_{(s_h,a_h) \sim (P^{\star,k},\tpi^k)}\left[\frac{2}{\lambda_{W}}\norm{\phi_{h}^{\star,k}(s_{h},a_{h})-\phi_{h}^{\star,tW+1}(s_{h},a_{h})}_2\right]\nonumber\\
		& \quad \leq \frac{2W}{\lambda_{W}}\Delta_{\{h\},[tW+1,t(W+1)-1]}^{\phi}, \label{Eq: Lemma7-3}
	\end{align}
	where $\RM{1}$ follows from the property of the matrix norms induced by vector $\ell_2$-norm.  \newline
	\textbf{For term $(I)$}, we have
	\begin{align}
		&\sum_{k=tW+1}^{(t+1)W}\Eb_{(s_h,a_h) \sim (P^{\star,k},\tpi^k)}\left[\left\|\phi_{h}^{\star,k}(s_{h},a_{h})\right\|^2_{(U^{k,W}_{h,\phi^{\star,k}})^{-1}}-\left\|\phi_{h}^{\star,k}(s_{h},a_{h})\right\|^2_{(U^{k,W}_{h,\phi^{\star,tW+1}})^{-1}}\right]\nonumber\\
		& \quad =\sum_{k=tW+1}^{(t+1)W}\Eb_{(s_h,a_h) \sim (P^{\star,k},\tpi^k)}\left[\phi_{h}^{\star,k}(s_{h},a_{h})^\top\left((U^{k,W}_{h,\phi^{\star,k}})^{-1}-(U^{k,W}_{h,\phi^{\star,tW+1}})^{-1}\right)\phi_{h}^{\star,k}(s_{h},a_{h})\right]\nonumber\\
		& \quad =\sum_{k=tW+1}^{(t+1)W}\Eb_{(s_h,a_h) \sim (P^{\star,k},\tpi^k)}\left[\phi_{h}^{\star,k}(s_{h},a_{h})^\top(U^{k,W}_{h,\phi^{\star,k}})^{-1}\left(U^{k,W}_{h,\phi^{\star,tW+1}}-U^{k,W}_{h,\phi^{\star,k}}\right)(U^{k,W}_{h,\phi^{\star,tW+1}})^{-1}\phi_{h}^{\star,k}(s_{h},a_{h})\right]\nonumber\\
		& \scriptstyle \quad \overset{\RM{1}}{\leq} \sum_{k=tW+1}^{(t+1)W}\Eb_{(s_h,a_h) \sim (P^{\star,k},\tpi^k)}\left[\norm{\phi_{h}^{\star,k}(s_{h},a_{h})}_2\norm{(U^{k,W}_{h,\phi^{\star,k}})^{-1}}_2\norm{\left(U^{k,W}_{h,\phi^{\star,tW+1}}-U^{k,W}_{h,\phi^{\star,k}}\right)}_2\norm{(U^{k,W}_{h,\phi^{\star,tW+1}})^{-1}}_2\norm{\phi_{h}^{\star,k}(s_{h},a_{h})}_2\right]\nonumber\\
		& \scriptstyle \quad   \overset{\RM{2}}{\leq} \frac{1}{\lambda_{W}^2} \sum_{k=tW+1}^{(t+1)W}\Eb_{(s_h,a_h) \sim (P^{\star,k},\tpi^k)}\left[\norm{\sum_{i= 1 \lor k-W}^{k-1} \mathop{\Eb}_{(s_{h}, a_h)\sim\left(P^{\star,i},\tpi^{i}\right)}\left[\phi^{\star,tW+1}_h(s_h,a_h)\phi^{\star,tW+1}_h(s_h,a_h)^\top-\phi^{\star,k}_h(s_h,a_h)\phi^{\star,k}_h(s_h,a_h)^\top\right]}_2\right]\nonumber\\
		& \quad \leq \frac{1}{\lambda_{W}^2} \sum_{k=tW+1}^{(t+1)W}\sum_{i= 1 \lor k-W}^{k-1} \mathop{\Eb}_{(s_{h}, a_h)\sim\left(P^{\star,i},\tpi^{i}\right)}\left[\norm{\phi^{\star,tW+1}_h(s_h,a_h)\phi^{\star,tW+1}_h(s_h,a_h)^\top-\phi^{\star,k}_h(s_h,a_h)\phi^{\star,k}_h(s_h,a_h)^\top}_2\right]\nonumber\\
		& \quad \leq \frac{1}{\lambda_{W}^2} \sum_{k=tW+1}^{(t+1)W}\sum_{i= 1 \lor k-W}^{k-1} \mathop{\Eb}_{(s_{h}, a_h)\sim\left(P^{\star,i},\tpi^{i}\right)}\left[\norm{\phi^{\star,tW+1}_h(s_h,a_h)\phi^{\star,tW+1}_h(s_h,a_h)^\top-\phi^{\star,tW+1}_h(s_h,a_h)\phi^{\star,k}_h(s_h,a_h)^\top}_2\right.\nonumber\\
		&\quad \left.+\norm{\phi^{\star,tW+1}_h(s_h,a_h)\phi^{\star,k}_h(s_h,a_h)^\top-\phi^{\star,k}_h(s_h,a_h)\phi^{\star,k}_h(s_h,a_h)^\top}_2\right]\nonumber\\
		& \quad \leq \frac{2}{\lambda_{W}^2}\sum_{k=tW+1}^{(t+1)W}\sum_{i= 1 \lor k-W}^{k-1}\Eb_{(s_h,a_h) \sim (P^{\star,i},\tpi^i)}\left[\norm{\phi_{h}^{\star,k}(s_{h},a_{h})-\phi^{\star,tW+1}_h(s_{h},a_{h})}_2\right]\nonumber\\
		& \quad \leq \sum_{k=tW+1}^{(t+1)W}\sum_{i= 1 \lor k-W}^{k-1}\frac{2}{\lambda_{W}^2}\Delta_{h,[tW+1,k-1]}^{\phi}, \label{Eq: Lemma7-4}
	\end{align}
        where $\RM{1}$ follows from the property of the matrix norms induced by vector $\ell_2$-norm and $\RM{2}$ follows from that $\norm{\phi_{h}^{\star,k}(s_{h},a_{h})}_2 \leq 1$.
        
	Furthermore,
	\begin{align}
         & \sum_{k=tW+1}^{(t+1)W}\Eb_{(s_h,a_h) \sim (P^{\star,k},\tpi^k)}\left[\left\|\phi^{\star,tW+1}_h(s_{h},a_{h})\right\|^2_{(U^{k,W}_{h,\phi^{\star,tW+1}})^{-1}}\right]\nonumber\\
         & \quad =\sum_{k=tW+1}^{(t+1)W} \mathrm{tr}\left(\Eb_{(s_h,a_h) \sim (P^{\star,k},\tpi^k)}\left[\phi^{\star,tW+1}_{h}(s_{h},a_{h})\phi^{\star,tW+1}_{h}(s_{h},a_{h})^\top\right] {(U^{k,W}_{h,\phi^{\star,tW+1}})^{-1}}\right)\nonumber\\
         & \quad \leq \sum_{k=tW+1}^{(t+1)W} \mathrm{tr}\left(\Eb_{(s_h,a_h) \sim (P^{\star,k},\tpi^k)}\left[\phi^{\star,tW+1}_{h}(s_{h},a_{h})\phi^{\star,tW+1}_{h}(s_{h},a_{h})^\top\right] (\tU^{k,W,t}_{h,\phi^{\star,tW+1}})^{-1}\right)\nonumber\\
         & \quad \leq \sum_{k=tW+1}^{(t+1)W} A\mathop{\Eb}_{s_h\sim (P^{\star,k},\tpi^k)\atop{a_h \sim \Uc(\Ac)}}\mathrm{tr}\left[\phi^{\star,tW+1}_{h}(s_{h},a_{h})\phi^{\star,tW+1}_{h}(s_{h},a_{h})^\top\right] (\tU^{k,W,t}_{h,\phi^{\star,tW+1}})^{-1}\nonumber\\
          & \quad \leq 2Ad\log(1+\frac{W}{d\lambda_0}),\label{Eq: Lemma7-5}
	\end{align}
	where the last equation follows from \Cref{Lemma: Elliptical_potential}.\newline
	Then combining \Cref{Eq: Lemma7-2,Eq: Lemma7-3,Eq: Lemma7-4,Eq: Lemma7-5}, we have 
	\begin{align}
		&\sum_{k=tW+1}^{(t+1)W}\Eb_{(s_h,a_h) \sim (P^{\star,k},\tpi^k)}\left[\left\|\phi_{h}^{\star,k}(s_{h},a_{h})\right\|^2_{(U^{k,W}_{h,\phi^{\star,k}})^{-1}}\right] \nonumber\\
		& \quad \leq \ 2Ad\log(1+\frac{W}{d\lambda_0})+\frac{2W}{\lambda_{W}}\Delta_{\{h\},[tW+1,t(W+1)-1]}^{\phi} +\sum_{k=tW+1}^{(t+1)W}\sum_{i= 1 \lor k-W}^{k-1}\frac{2}{\lambda_{W}^2}\Delta_{h,[tW+1,k-1]}^{\phi}. \label{Eq: Lemma7-6}
	\end{align}
	Substituting \Cref{Eq: Lemma7-6} into \Cref{Eq: Lemma7-1}, we have 
	\begin{align}
	&\sum_{k\in[K]}\Eb_{(s_h,a_h) \sim (P^{\star,k},\tpi^k)}\left[\alpha_{k,W}\left\|\phi^{\star,k}_{h}(s_{h},a_{h})\right\|_{(U^k_{h,\phi^{\star,k}})^{-1}}\right]\nonumber\\
	& \quad \leq \sqrt{\sum_{k\in[K]}\alpha_{k,W}^2\sum_{t=0}^{\lfloor K/W\rfloor}\sum_{k=tW+1}^{(t+1)W} \Eb_{(s_h,a_h) \sim (P^{\star,k},\tpi^k)}\left[\left\|\phi^{\star,k}_{h}(s_{h},a_{h})\right\|_{(U^{k,W}_{h,\phi^{\star,k}})^{-1}}\right]^2}\nonumber \\
	& \quad \leq \scriptstyle\sqrt{\sum_{k\in[K]}\alpha_{k,W}^2\sum_{t=0}^{\lfloor K/W\rfloor}\left[2Ad\log(1+\frac{W}{d\lambda_0})+\frac{2W}{\lambda_{W}}\Delta_{\{h\},[tW+1,t(W+1)-1]}^{\phi} +\sum_{k=tW+1}^{(t+1)W}\sum_{i= 1 \lor k-W}^{k-1}\frac{2}{\lambda_{W}^2}\Delta_{h,[tW+1,k-1]}^{\phi}\right]}\nonumber\\
	& \quad \leq \sqrt{K\left(2WA\zeta_{k,W}+\lambda_{k,W} d\right)\left[\frac{2KAd}{W}\log(1+\frac{W}{d\lambda_0})+\frac{2W}{\lambda_{W}}\Delta_{\{h\},[K]}^{\phi} +\frac{2W^2}{\lambda_{W}^2}\Delta_{\{h\},[K]}^{\phi}\right]}\label{Eq: Lemma7-7}
	\end{align}
		where the second equation follows from \Cref{Eq: Lemma7-6}.\newline
		Then we derive the following bound:
		\begin{align} 
			& \sum_{k=1}^{K}V^{\tpi^{k}}_{P^{\star,k},f^k} \nonumber\\
			& \quad = \sum_{k\in[K]}\sum_{h\in[H]}\Eb_{(s_h,a_h) \sim (P^{\star,k},\tpi^k)}\left[f^k_h(s_h,a_h)\right]\nonumber\\
			&\scriptstyle \quad \overset{\RM{1}}{\leq}     \sum_{k\in[K]}\left\{\sum_{h=2}^{H}\left[\Eb_{(s_{h-1},a_{h-1}) \sim (P^{\star,k},\tpi^k)}\left[\alpha_{k,W}\left\|\phi^{\star,k}_{h-1}(s_{h-1},a_{h-1})\right\|_{(U^k_{h-1,\phi^{\star,k}})^{-1}}\right]+\sqrt{\frac{1}{2d}WAC_B\Delta^P_{[h-1,h],[k-W,k-1]}}\right]\right.\nonumber\\
			& \quad \left.+\sqrt{A\left(\zeta_{k,W}+\frac{1}{2}C_B\Delta^P_{1,[k-W,k-1]}\right)}\right\}\nonumber\\
			&  \quad \leq \sum_{h=1}^{H-1}\sum_{k\in[K]}\Eb_{(s_h,a_h) \sim (P^{\star,k},\tpi^k)}\left[\alpha_{k,W}\left\|\phi^{\star,k}_{h}(s_{h},a_{h})\right\|_{(U^k_{h,\phi^{\star,k}})^{-1}}\right] \nonumber\\
            & \qquad \qquad +\sum_{k\in[K]}\sum_{h=1}^{H}\sqrt{WAC_B\Delta^P_{[h-1,h],[k-W,k-1]}}+\sum_{k \in [K]}\sqrt{A\zeta_{k,W}}\nonumber\\
			& \quad \overset{\RM{2}}{\leq} \sum_{h=1}^{H-1}\sqrt{K\left(2WA\zeta_{k,W}+\lambda_{k,W} d\right)\left[\frac{2AKd}{W}\log(1+\frac{W}{d\lambda_0})+\frac{2W}{\lambda_{W}}\Delta_{\{h\},[K]}^{\phi} +\frac{2W^2}{\lambda_{W}^2}\Delta_{\{h\},[K]}^{\phi}\right]}\nonumber\\
			& \quad +\sum_{k\in[K]}\sum_{h=1}^{H}\sqrt{WAC_B\Delta^P_{[h-1,h],[k-W,k-1]}}+\sum_{k \in [K]}\sqrt{A\zeta_{k,W}}\nonumber\\
               & \quad \leq \sqrt{K\left(2WA\zeta_{k,W}+\lambda_{k,W} d\right)}\left[H\sqrt{\frac{2AKd}{W}\log(1+\frac{W}{d\lambda_0})}+\sqrt{\frac{2HW}{\lambda_{W}}\Delta_{[H],[K]}^{{\phi}}} +\sqrt{\frac{2HW^2}{\lambda_{W}^2}\Delta_{[H],[K]}^{{\phi}}}\right]\nonumber\\
               & \quad +\sqrt{W^3AC_B}\Delta^{\sqrt{P}}_{[H],[K]}+\sum_{k \in [K]}\sqrt{A\zeta_{k,W}}\nonumber\\
               & \quad \leq O\left(\sqrt{K(A\log(|\Phi||\Psi|KH/\delta)+d^2)}\left[H\sqrt{\frac{2AKd}{W}\log(W)}+\sqrt{HW^2\Delta_{[H],[K]}^{{\phi}}}\right]+\sqrt{W^3AC_B}\Delta^{\sqrt{P}}_{[H],[K]}\right), \label{Eq: Lemma6-Step1}
		\end{align}
		where $\RM{1}$ follows from \Cref{lemma: Step_back_Bounded TV_sP}, and $\RM{2}$ follows from \Cref{Eq: Lemma7-7}.\newline
		\textbf{Step 2:} We next bound $\sum_{k=1}^K V_{P^{\star,k},\hb^{k}}^{\tpi^{k}}$ via an \textbf{auxiliary anchor representation}. \newline
            Similarly to the proof \textbf{Step 1}, we further bound $\sum_{k\in[K]}\Eb_{(s_h,a_h) \sim (P^{\star,k},\tpi^k)}\left[\beta_{k,W}\left\|\phi^{\star,k}_{h}(s_{h},a_{h})\right\|_{(W^k_{h,\phi^{\star,k}})^{-1}}\right]$. We define $W^{k,W}_{h,\phi^{\star,k}}=\sum_{i= 1\lor(k-W)}^{k-1} \mathop{\Eb}_{(s_{h},a_h)\sim\left(P^{\star,i},\tpi^{i}\right)}\left[\phi^{\star,k}_h(s_h,a_h)\phi^{\star,k}_h(s_h,a_h)^\top\right]+\lambda_{k,W} I_d$ and $\tW^{k,W,t}_{h,\phi^{\star,k}}=\sum_{i= tW+1}^{k-1} \mathop{\Eb}_{(s_{h},a_h)\sim\left(P^{\star,i},\tpi^{i}\right)}\left[\phi^{\star,k}_h(s_h,a_h)\phi^{\star,k}_h(s_h,a_h)^\top\right]+\lambda_{k,W} I_d$. We first note that for any $h$, we have 
		\begin{align}
			&\sum_{k\in[K]}\Eb_{(s_h,a_h) \sim (P^{\star,k},\tpi^k)}\left[\beta_{k,W}\left\|\phi^{\star,k}_{h}(s_{h},a_{h})\right\|_{(W^k_{h,\phi^{\star,k}})^{-1}}\right]\nonumber\\
			& \quad = \sqrt{\sum_{k\in[K]}\beta_{k,W}^2\sum_{k\in[K]}\Eb_{(s_h,a_h) \sim (P^{\star,k},\tpi^k)}\left[\left\|\phi^{\star,k}_{h}(s_{h},a_{h})\right\|_{(W^{k,W}_{h,\phi^{\star,k}})^{-1}}\right]^2}\nonumber\\
                & \quad = \sqrt{\sum_{k\in[K]}\beta_{k,W}^2\sum_{t=0}^{\lfloor K/W\rfloor}\sum_{k=tW+1}^{(t+1)W} \Eb_{(s_h,a_h) \sim (P^{\star,k},\tpi^k)}\left[\left\|\phi^{\star,k}_{h}(s_{h},a_{h})\right\|_{(W^{k,W}_{h,\phi^{\star,k}})^{-1}}\right]^2} \label{Eq: Lemma8-1}
                \end{align}
                The $\phi^{\star,k}$ and $W$ in \Cref{Eq: Lemma8-1} both change with the round index $k$. To deal with this issue, we decompose it as follows. We first derive the following equation: 
	        	\begin{align}
		     	& \sum_{k=tW+1}^{(t+1)W}\Eb_{(s_h,a_h) \sim (P^{\star,k},\tpi^k)}\left[\left\|\phi_{h}^{\star,k}(s_{h},a_{h})\right\|^2_{(W^{k,W}_{h,\phi^{\star,k}})^{-1}}\right]\nonumber\\
                &\qquad\qquad\qquad-\sum_{k=tW+1}^{(t+1)W}\Eb_{(s_h,a_h) \sim (P^{\star,k},\tpi^k)}\left[\left\|\phi_{h}^{\star,tW+1}(s_{h},a_{h})\right\|^2_{(W^{k,W}_{h,\phi^{\star,tW+1}})^{-1}}\right]\nonumber\\
				& \quad = \sum_{k=tW+1}^{(t+1)W}\Eb_{(s_h,a_h) \sim (P^{\star,k},\tpi^k)}\left[\left\|\phi_{h}^{\star,k}(s_{h},a_{h})\right\|^2_{(W^{k,W}_{h,\phi^{\star,k}})^{-1}}-\left\|\phi_{h}^{\star}(s_{h},a_{h})\right\|^2_{(W^{k,W}_{h,\phi^{\star,tW+1}})^{-1}}\right]\nonumber\\
				& \quad = \sum_{k=tW+1}^{(t+1)W}\Eb_{(s_h,a_h) \sim (P^{\star,k},\tpi^k)}\left[\left\|\phi_{h}^{\star,k}(s_{h},a_{h})\right\|^2_{(W^{k,W}_{h,\phi^{\star,k}})^{-1}}-\left\|\phi_{h}^{\star,k}(s_{h},a_{h})\right\|^2_{(W^{k,W}_{h,\phi^{\star,tW+1}})^{-1}}\right.\nonumber\\
				& \left.\quad +\left\|\phi_{h}^{\star,k}(s_{h},a_{h})\right\|^2_{(W^{k,W}_{h,\phi^{\star,tW+1}})^{-1}}-\left\|\phi_{h}^{\star,tW+1}(s_{h},a_{h})\right\|^2_{(W^{k,W}_{h,\phi^{\star,tW+1}})^{-1}}\right].\nonumber\\
				& \quad = \underbrace{\sum_{k=tW+1}^{(t+1)W}\Eb_{(s_h,a_h) \sim (P^{\star,k},\tpi^k)}\left[\left\|\phi_{h}^{\star,k}(s_{h},a_{h})\right\|^2_{(W^{k,W}_{h,\phi^{\star,k}})^{-1}}-\left\|\phi_{h}^{\star,k}(s_{h},a_{h})\right\|^2_{(W^{k,W}_{h,\phi^{\star,tW+1}})^{-1}}\right]}_{(III)}\nonumber\\
				& \quad + \underbrace{\sum_{k=tW+1}^{(t+1)W}\Eb_{(s_h,a_h) \sim (P^{\star,k},\tpi^k)}\left[\left\|\phi_{h}^{\star,k}(s_{h},a_{h})\right\|^2_{(W^{k,W}_{h,\phi^{\star,tW+1}})^{-1}}-\left\|\phi_{h}^{\star,tW+1}(s_{h},a_{h})\right\|^2_{(W^{k,W}_{h,\phi^{\star,tW+1}})^{-1}}\right]}_{(IV)}. \label{Eq: Lemma8-2}
				\end{align}
	\textbf{For term $(IV)$}, we have
	\begin{align}
		&\sum_{k=tW+1}^{(t+1)W}\Eb_{(s_h,a_h) \sim (P^{\star,k},\tpi^k)}\left[\left\|\phi_{h}^{\star,k}(s_{h},a_{h})\right\|^2_{(W^{k,W}_{h,\phi^{\star,tW+1}})^{-1}}-\left\|\phi_{h}^{\star,tW+1}(s_{h},a_{h})\right\|^2_{(W^{k,W}_{h,\phi^{\star,tW+1}})^{-1}}\right] \nonumber\\
		& \quad \scriptstyle\leq \sum_{k=tW+1}^{(t+1)W}\Eb_{(s_h,a_h) \sim (P^{\star,k},\tpi^k)}\left[\phi_{h}^{\star,k}(s_{h},a_{h})^\top(W^{k,W}_{h,\phi^{\star,tW+1}})^{-1}\phi_{h}^{\star,k}(s_{h},a_{h})-\phi_{h}^{\star,k}(s_{h},a_{h})^\top(W^{k,W}_{h,\phi^{\star,tW+1}})^{-1}\phi_{h}^{\star,tW+1}(s_{h},a_{h})\right.\nonumber\\
		& \quad + \left.\phi_{h}^{\star,k}(s_{h},a_{h})^\top(W^{k,W}_{h,\phi^{\star,tW+1}})^{-1}\phi_{h}^{\star,tW+1}(s_{h},a_{h})-\phi_{h}^{\star,tW+1}(s_{h},a_{h})^\top(W^{k,W}_{h,\phi^{\star,tW+1}})^{-1}\phi_{h}^{\star,tW+1}(s_{h},a_{h})\right]\nonumber\\
		& \quad \overset{\RM{1}}{\leq} \sum_{k=tW+1}^{(t+1)W}\Eb_{(s_h,a_h) \sim (P^{\star,k},\tpi^k)}\left[\norm{\phi_{h}^{\star,k}(s_{h},a_{h})}_2\norm{(W^{k,W}_{h,\phi^{\star,tW+1}})^{-1}}_2\norm{\phi_{h}^{\star,k}(s_{h},a_{h})-\phi_{h}^{\star,tW+1}(s_{h},a_{h})}_2\right.\nonumber\\
		& \quad + \left.\norm{\phi_{h}^{\star,k}(s_{h},a_{h})-\phi_{h}^{\star,tW+1}(s_{h},a_{h})}_2\norm{(W^{k,W}_{h,\phi^{\star,tW+1}})^{-1}}_2\norm{\phi_{h}^{\star,tW+1}(s_{h},a_{h})}_2\right]\nonumber\\
		& \quad \leq \sum_{k=tW+1}^{(t+1)W}\Eb_{(s_h,a_h) \sim (P^{\star,k},\tpi^k)}\left[\frac{2}{\lambda_{W}}\norm{\phi_{h}^{\star,k}(s_{h},a_{h})-\phi_{h}^{\star,tW+1}(s_{h},a_{h})}_2\right]\nonumber\\
		& \quad \leq \frac{2W}{\lambda_{W}}\Delta_{\{h\},[tW+1,t(W+1)-1]}^{\phi}, \label{Eq: Lemma8-3}
	\end{align}
	where 
 $\RM{1}$ follows from the property of the matrix norms induced by vector $\ell_2$-norm.  
 
	\textbf{For term $(III)$}, we derive the following bound:
	\begin{align}
		&\sum_{k=tW+1}^{(t+1)W}\Eb_{(s_h,a_h) \sim (P^{\star,k},\tpi^k)}\left[\left\|\phi_{h}^{\star,k}(s_{h},a_{h})\right\|^2_{(W^{k,W}_{h,\phi^{\star,k}})^{-1}}-\left\|\phi_{h}^{\star,k}(s_{h},a_{h})\right\|^2_{(W^{k,W}_{h,\phi^{\star,tW+1}})^{-1}}\right]\nonumber\\
		& \quad =\sum_{k=tW+1}^{(t+1)W}\Eb_{(s_h,a_h) \sim (P^{\star,k},\tpi^k)}\left[\phi_{h}^{\star,k}(s_{h},a_{h})^\top\left((W^{k,W}_{h,\phi^{\star,k}})^{-1}-(W^{k,W}_{h,\phi^{\star,tW+1}})^{-1}\right)\phi_{h}^{\star,k}(s_{h},a_{h})\right]\nonumber\\
		& \quad =\sum_{k=tW+1}^{(t+1)W}\Eb_{(s_h,a_h) \sim (P^{\star,k},\tpi^k)}\left[\phi_{h}^{\star,k}(s_{h},a_{h})^\top(W^{k,W}_{h,\phi^{\star,k}})^{-1}\right.\nonumber\\
        & \hspace{2in}\left.\quad \times\left(W^{k,W}_{h,\phi^{\star,tW+1}}-W^{k,W}_{h,\phi^{\star,k}}\right)(W^{k,W}_{h,\phi^{\star,tW+1}})^{-1}\phi_{h}^{\star,k}(s_{h},a_{h})\right]\nonumber\\
		&\quad \overset{\RM{1}}{\leq} \sum_{k=tW+1}^{(t+1)W}\Eb_{(s_h,a_h) \sim (P^{\star,k},\tpi^k)}\left[\norm{\phi_{h}^{\star,k}(s_{h},a_{h})}_2\norm{(W^{k,W}_{h,\phi^{\star,k}})^{-1}}_2\norm{\left(W^{k,W}_{h,\phi^{\star,tW+1}}-W^{k,W}_{h,\phi^{\star,k}}\right)}_2\right.\nonumber\\
 &\hspace{3in}\left.\times\norm{(W^{k,W}_{h,\phi^{\star,tW+1}})^{-1}}_2\norm{\phi_{h}^{\star,k}(s_{h},a_{h})}_2\right]\nonumber\\
		& \quad \overset{\RM{2}}{\leq} \frac{1}{\lambda_{W}^2} \sum_{k=tW+1}^{(t+1)W}\Eb_{(s_h,a_h) \sim (P^{\star,k},\tpi^k)}\left[\left\|\sum_{i= 1 \lor k-W}^{k-1} \mathop{\Eb}_{(s_{h}, a_h)\sim\left(P^{\star,i},\tpi^{i}\right)}\left[\phi^{\star,tW+1}_h(s_h,a_h)\phi^{\star,tW+1}_h(s_h,a_h)^\top\right.\right.\right.\nonumber\\
        & \left.\left.\hspace{3in}-\phi^{\star,k}_h(s_h,a_h)\phi^{\star,k}_h(s_h,a_h)^\top\right]\big\|_2\right]\nonumber\\
		& \scriptstyle\quad \leq \frac{1}{\lambda_{W}^2} \sum_{k=tW+1}^{(t+1)W}\sum_{i= 1 \lor k-W}^{k-1} \mathop{\Eb}_{(s_{h}, a_h)\sim\left(P^{\star,i},\tpi^{i}\right)}\left[\norm{\phi^{\star,tW+1}_h(s_h,a_h)\phi^{\star,tW+1}_h(s_h,a_h)^\top-\phi^{\star,k}_h(s_h,a_h)\phi^{\star,k}_h(s_h,a_h)^\top}_2\right]\nonumber\\
		& \scriptstyle\quad \leq \frac{1}{\lambda_{W}^2} \sum_{k=tW+1}^{(t+1)W}\sum_{i= 1 \lor k-W}^{k-1} \mathop{\Eb}_{(s_{h}, a_h)\sim\left(P^{\star,i},\tpi^{i}\right)}\left[\norm{\phi^{\star,tW+1}_h(s_h,a_h)\phi^{\star,tW+1}_h(s_h,a_h)^\top-\phi^{\star,tW+1}_h(s_h,a_h)\phi^{\star,k}_h(s_h,a_h)^\top}_2\right.\nonumber\\
		&\quad \left.+\norm{\phi^{\star,tW+1}_h(s_h,a_h)\phi^{\star,k}_h(s_h,a_h)^\top-\phi^{\star,k}_h(s_h,a_h)\phi^{\star,k}_h(s_h,a_h)^\top}_2\right]\nonumber\\
		& \quad \leq \frac{2}{\lambda_{W}^2}\sum_{k=tW+1}^{(t+1)W}\sum_{i= 1 \lor k-W}^{k-1}\Eb_{(s_h,a_h) \sim (P^{\star,i},\tpi^i)}\left[\norm{\phi_{h}^{\star,k}(s_{h},a_{h})-\phi^{\star,tW+1}_h(s_{h},a_{h})}_2\right]\nonumber\\
		& \quad \leq \sum_{k=tW+1}^{(t+1)W}\sum_{i= 1 \lor k-W}^{k-1}\frac{2}{\lambda_{W}^2}\Delta_{h,[tW+1,k-1]}^{\phi}. \label{Eq: Lemma8-4}
	\end{align}
  where $\RM{1}$ follows from the property of the matrix norms induced by vector $\ell_2$-norm and $\RM{2}$ follows from that $\norm{\phi_{h}^{\star,k}(s_{h},a_{h})}_2 \leq 1$.
        
	Furthermore, we derive the following bound:
	\begin{align}
         & \sum_{k=tW+1}^{(t+1)W}\Eb_{(s_h,a_h) \sim (P^{\star,k},\tpi^k)}\left[\left\|\phi^{\star,tW+1}_h(s_{h},a_{h})\right\|^2_{(W^{k,W}_{h,\phi^{\star,tW+1}})^{-1}}\right]\nonumber\\
         & \quad =\sum_{k=tW+1}^{(t+1)W} \mathrm{tr}\left(\Eb_{(s_h,a_h) \sim (P^{\star,k},\tpi^k)}\left[\phi^{\star,tW+1}_{h}(s_{h},a_{h})\phi^{\star,tW+1}_{h}(s_{h},a_{h})^\top\right] {(W^{k,W}_{h,\phi^{\star,tW+1}})^{-1}}\right)\nonumber\\
         & \quad \leq \sum_{k=tW+1}^{(t+1)W} \mathrm{tr}\left(\Eb_{(s_h,a_h) \sim (P^{\star,k},\tpi^k)}\left[\phi^{\star,tW+1}_{h}(s_{h},a_{h})\phi^{\star,tW+1}_{h}(s_{h},a_{h})^\top\right] (\tW^{k,W,t}_{h,\phi^{\star,tW+1}})^{-1}\right)\nonumber\\
         & \quad \leq \sum_{k=tW+1}^{(t+1)W} \mathop{\Eb}_{(s_h,a_h)\sim (P^{\star,k},\tpi^k)}\mathrm{tr}\left[\phi^{\star,tW+1}_{h}(s_{h},a_{h})\phi^{\star,tW+1}_{h}(s_{h},a_{h})^\top\right] (\tW^{k,W,t}_{h,\phi^{\star,tW+1}})^{-1}\nonumber\\
          & \quad \leq 2d\log(1+\frac{W}{d\lambda_0}),\label{Eq: Lemma8-5}
	\end{align}
	where the last equation follows from \Cref{Lemma: Elliptical_potential}.\newline
	Then combining \Cref{Eq: Lemma7-2,Eq: Lemma7-3,Eq: Lemma7-4,Eq: Lemma8-5}, we have 
	\begin{align}
		&\sum_{k=tW+1}^{(t+1)W}\Eb_{(s_h,a_h) \sim (P^{\star,k},\tpi^k)}\left[\left\|\phi_{h}^{\star,k}(s_{h},a_{h})\right\|^2_{(W^{k,W}_{h,\phi^{\star,k}})^{-1}}\right] \nonumber\\
		& \quad \leq \ 2d\log(1+\frac{W}{d\lambda_0})+\frac{2W}{\lambda_{W}}\Delta_{\{h\},[tW+1,t(W+1)-1]}^{\phi} +\sum_{k=tW+1}^{(t+1)W}\sum_{i= 1 \lor k-W}^{k-1}\frac{2}{\lambda_{W}^2}\Delta_{h,[tW+1,k-1]}^{\phi}. \label{Eq: Lemma8-6}
	\end{align}
	Substituting \Cref{Eq: Lemma8-6} into \Cref{Eq: Lemma8-1}, we have 
	\begin{align}
	&\sum_{k\in[K]}\Eb_{(s_h,a_h) \sim (P^{\star,k},\tpi^k)}\left[\beta_{k,W}\left\|\phi^{\star,k}_{h}(s_{h},a_{h})\right\|_{(U^k_{h,\phi^{\star,k}})^{-1}}\right]\nonumber\\
	& \quad \leq \sqrt{\sum_{k\in[K]}\beta_{k,W}^2\sum_{t=0}^{\lfloor K/W\rfloor}\sum_{k=tW+1}^{(t+1)W} \Eb_{(s_h,a_h) \sim (P^{\star,k},\tpi^k)}\left[\left\|\phi^{\star,k}_{h}(s_{h},a_{h})\right\|_{(U^{k,W}_{h,\phi^{\star,k}})^{-1}}\right]^2}\nonumber \\
	& \quad \scriptstyle \leq \sqrt{\sum_{k\in[K]}\beta_{k,W}^2\sum_{t=0}^{\lfloor K/W\rfloor}\left[2d\log(1+\frac{W}{d\lambda_0})+\frac{2W}{\lambda_{W}}\Delta_{\{h\},[tW+1,t(W+1)-1]}^{\phi} +\sum_{k=tW+1}^{(t+1)W}\sum_{i= 1 \lor k-W}^{k-1}\frac{2}{\lambda_{W}^2}\Delta_{h,[tW+1,k-1]}^{\phi}\right]}\nonumber\\
	& \quad \scriptstyle \leq \sqrt{K\left(9dA(2WA\zeta_{k,W}+\lambda_{k,W} d)  + \lambda_{k,W} d \right)\left[\frac{2Kd}{W}\log(1+\frac{W}{d\lambda_0})+\frac{2W}{\lambda_{W}}\Delta_{\{h\},[K]}^{\phi} +\frac{2W^2}{\lambda_{W}^2}\Delta_{\{h\},[K]}^{\phi}\right]}.\label{Eq: Lemma8-7}
	\end{align}
		where the second equation follows from \Cref{Eq: Lemma8-6}.\newline
		Then, we derive the following bound:
		\begin{align}          
		&\sum_{k\in[K]}V_{P^{\star,k},\hb^{k}}^{\tpi^{k}}
		=\sum_{k\in[K]}\sum_{h\in[H]}\Eb_{(s_h,a_h) \sim (P^{\star,k},\tpi^k)}\left[\hb^{k}_h(s_h,a_h)\right]\nonumber\\
		&\quad \overset{\RM{1}}{\leq}     \sum_{k\in[K]}\left\{\sum_{h=2}^{H}\left\{\Eb_{(s_{h-1},a_{h-1}) \sim (P^{\star,k},\tpi^k)}\left[\beta_{k,W}\left\|\phi^{\star,k}_{h-1}(s_{h-1},a_{h-1})\right\|_{(W^k_{h-1,\phi^{\star,k}})^{-1}}\right]\right.\right.\nonumber\\
        &\hspace{3in}\left.\left.+\sqrt{\frac{A}{d}\Delta^P_{h-1,[k-W,k-1]}}\right\}+\sqrt{\frac{9Ad\alpha_{k,W}^2}{w}}\right\}\nonumber\\
		&\quad  \leq \sum_{h=1}^{H-1}\sum_{k\in[K]}\Eb_{(s_h,a_h) \sim (P^{\star,k},\tpi^k)}\left[\beta_{k,W}\left\|\phi^{\star,k}_{h}(s_{h},a_{h})\right\|_{(W^k_{h,\phi^{\star,k}})^{-1}}\right]\nonumber\\
        &\hspace{3in}+W\sqrt{\frac{A}{d}}\Delta^{\sqrt{P}}_{[H],[K]}+\sum_{k \in [K]}\sqrt{\frac{9Ad\alpha_{k,W}^2}{W}}\nonumber\\
		&\quad  \overset{\RM{2}}{\leq} \sqrt{K\left(9dA(2WA\zeta_{k,W}+\lambda_{k,W} d)  + \lambda_{k,W} d \right)}\left[H\sqrt{\frac{2Kd}{W}\log(1+\frac{W}{d\lambda_0})}+\sqrt{\frac{2HW}{\lambda_{W}}\Delta_{[H],[K]}^{{\phi}}}\right.\nonumber\\
        &\left.\hspace{3in} +\sqrt{\frac{2HW^2}{\lambda_{W}^2}\Delta_{[H],[K]}^{{\phi}}}\right]+W\sqrt{\frac{A}{d}}\Delta^{\sqrt{P}}_{[H],[K]}\nonumber\\
		& \quad \leq O\left(\sqrt{KdA(A\log(|\Phi||\Psi|KH/\delta)+d^2)}\left[H\sqrt{\frac{Kd}{W}\log(W)}+\sqrt{HW^2\Delta_{[H],[K]}^{{\phi}}}\right]\right.\nonumber\\
        &\left.\hspace{4in}+W\sqrt{{A}}\Delta^{\sqrt{P}}_{[H],[K]}\right), \label{Eq: Lemma6-Step2}
		\end{align}
		where $\RM{1}$ follows from \Cref{lemma: Step_back_Bounded Bonus_sP}, and $\RM{2}$ follows from \Cref{Eq: Lemma8-7}.\newline
	Finally, combining \Cref{Eq: hb < sb + sf,Eq: Lemma6-Step1,Eq: Lemma6-Step2}, we have 
		\begin{align*}
			\sum_{k=1}^K\hV_{\hP^k,\hb^k}^{\tpi^k} &\scriptstyle \leq O\left(\sqrt{K(A\log(|\Phi||\Psi|KH/\delta)+d^2)}\left[H\sqrt{\frac{2AKd}{W}\log(W)}+\sqrt{HW^2\Delta_{[H],[K]}^{{\phi}}}\right]+\sqrt{W^3AC_B}\Delta^{\sqrt{P}}_{[H],[K]}\right)\\
			&\scriptstyle +O\left(\sqrt{KdA(A\log(|\Phi||\Psi|KH/\delta)+d^2)}\left[H\sqrt{\frac{Kd}{W}\log(W)}+\sqrt{HW^2\Delta_{[H],[K]}^{{\phi}}}\right]+W\sqrt{{A}}\Delta^{\sqrt{P}}_{[H],[K]}\right)\\
			&\scriptstyle  \leq O\left(\sqrt{KdA(A\log(|\Phi||\Psi|KH/\delta)+d^2)}\left[H\sqrt{\frac{Kd}{W}\log(W)}+\sqrt{HW^2\Delta_{[H],[K]}^{{\phi}}}\right]+\sqrt{{W^3A}}\Delta^{\sqrt{P}}_{[H],[K]}\right). 
		\end{align*}
	\end{proof}

 The following visitation probability difference lemma is similar to lemma 5 in \cite{DBLP:conf/nips/FeiYWX20}, but we remove their Assumption 1.
 \begin{lemma} \label{Lemma: lemma 5 in Fei}
     For any transition kernels $\{P_h\}_{h=1}^H$,$h \in  [H], j \in [h-1], s_h\in \Sc$ and policies $\{\pi_i\}_{i=1}^H$ and $\pi_j^\prime$,  we have
     \begin{align*}
         \left|P_1^{\pi_1}\ldots P_j^{\pi_j}\ldots P_{h-1}^{\pi_{h-1}}(s_h)-P_1^{\pi_1}\ldots P_j^{\pi_j^\prime}\ldots P_{h-1}^{\pi_{h-1}}(s_h)\right| \leq \max_{s \in \Sc}\norm{\pi_j(\cdot|s)-\pi_j^\prime(\cdot|s)}_{TV}
     \end{align*}
 \end{lemma}
 \begin{proof}
     To prove this lemma, the only difference from lemma is that we need to show $\max_{s_j}\sum_{s_{j+1}}|P_j^{\pi_{j}}(s_{j+1}|s_j)-P_j^{\pi_j^\prime}(s_{j+1}|s_j)| \leq 2\max_{s \in \Sc}\norm{\pi_j(\cdot|s)-\pi_j^\prime(\cdot|s)}_{TV}$ holds without assumption. We show this as follows:
     \begin{align*}
    & \quad \max_{s_j}\sum_{s_{j+1}}|P^{\pi_j}(s_{j+1}|s_j)-P^{\pi_j^\prime}(s_{j+1}|s_j)|\\
&=\max_{s_j}\sum_{s_{j+1}}|\sum_{a}P(s_{j+1}|s_j,a)\pi_j(a|s_j)-\sum_{a}P(s_{j+1}|s_j,a)\pi_j^\prime(a|s_j)|\\
&\leq \max_{s_j}\sum_{s_{j+1}}\sum_{a}P(s_{j+1}|s_j,a)|\pi_j(a|s_j)-\pi_j^\prime(a|s_j)|\\
&=\max_{s_j}\sum_{a}\sum_{s_{j+1}}P(s_{j+1}|s_j,a)|\pi_j(a|s_j)-\pi_j^\prime(a|s_j)|\\
&=\max_{s_j}\sum_{a}|\pi_j(a|s_j)-\pi_j^\prime(a|s_j)|=2\max_{s \in \Sc}\|\pi_j(\cdot|s)-\pi_j^\prime(\cdot|s)\|_{TV}.
\end{align*}
 \end{proof}
\section{Further Discussion and Proof of \Cref{Coro: of Thm1}}\label{Appd B}
In this section, we first provide a detailed version and further discussion of \Cref{Coro: of Thm1} in \Cref{app:cordiscussion}, then present the proof in \Cref{Appd: B.1}, and finally present an interesting special case in \Cref{Appd A.5.: remark1}.

\subsection{Further Discussion of \Cref{Coro: of Thm1}}\label{app:cordiscussion}
We present a detailed version of \Cref{Coro: of Thm1} as follows. Let $\Pi_{[1,K]}(N)=\min\{K,\max\{1,N\}\}$ for any $K,N \in \mathbb{N}$. 
 \begin{corollary}[Detailed version of \Cref{Coro: of Thm1}]\label{cor:detailcor}
Under the same conditions of \cref{Thm1: average dynamic suboptimality gap with known variation}, if the variation budgets are known, then for different variation budget regimes, we can select the hyper-parameters correspondingly to attain the optimality for both $(I)$ w.r.t. $W$ and $(II)$ w.r.t. $\tau$ in \Cref{Eq: Basic Bound}. For $(I)$, with $W=\Pi_{[1,K]}(\lfloor H^{\frac{1}{3}}d^{\frac{1}{3}}K^{\frac{1}{3}}(\Delta^{\sqrt{P}}+\Delta^{{\phi}})^{-\frac{1}{3}}\rfloor)$, part $(I)$ is upper-bounded by 
\begin{equation}
     \left\{
    \begin{aligned}
    & \sqrt{\frac{H^4d^2A}{K}\left(A+d^2\right)},&\left(\Delta^{\sqrt{P}}+\Delta^{{\phi}}\right) \leq \frac{Hd}{K^2},\\
     & {H^{2}d^{\frac{5}{6}}A^{\frac{1}{2}}}\left(A+d^2\right)^{\frac{1}{2}}(HK)^{-\frac{1}{6}}\left(\Delta^{\sqrt{P}}+\Delta^{{\phi}}\right)^{\frac{1}{6}},& \left(\Delta^{\sqrt{P}}+\Delta^{{\phi}}\right)>\frac{Hd}{K^2}, 
    \end{aligned} 
    \right. \label{Eq: Thm1-I}
\end{equation}
For $(II)$ in \Cref{Eq: Basic Bound}, with $\tau=\Pi_{[1,K]}(\lfloor K^{\frac{2}{3}}(\Delta^P+\Delta^\pi)^{-\frac{2}{3}}\rfloor)$, part $(II)$ is upper bounded by
\begin{equation}
     \left\{
    \begin{aligned}
    & \frac{2H}{\sqrt{K}},&(\Delta^P+\Delta^{\pi}) \leq \frac{1}{\sqrt{K}},\\
     &2H^{\frac{4}{3}}(HK)^{-\frac{1}{3}}(\Delta^P+\Delta^\pi)^{\frac{1}{3}},& \frac{1}{\sqrt{K}} < (\Delta^P+\Delta^{\pi}) \leq K,\\
     & H+\frac{H(\Delta^P+\Delta^{\pi})}{K}, &  K < (\Delta^P+\Delta^{\pi})
    \end{aligned}
    \right. \label{Eq: Thm1-II}
\end{equation}
For any $\epsilon \geq 0$, if nonstationarity is not significantly large, i.e., there exists a constant $\gamma <1$ such that $(\Delta^P+\Delta^\pi) \leq (2HK)^\gamma$ and $(\Delta^{\sqrt{P}}+\Delta^{\phi})$ $\leq (2HK)^\gamma$, then PORTAL can achieve $\epsilon$-average suboptimal with polynomial trajectories.
\end{corollary}
        As a direct consequence of \Cref{Thm1: average dynamic suboptimality gap with known variation}, \Cref{Coro: of Thm1} 
        indicates that if variation budgets are known, then the agent can choose the best hyper-parameters directly based on the variation budgets. The $\mathrm{Gap_{Ave}}$ can be different depending on which regime the variation budgets fall into, as can be seen in \Cref{Eq: Thm1-I,Eq: Thm1-II}. 

At the high level, we further explain how the window size $W$ depends on the variations of environment as follows. If the nonstationarity is moderate and not significantly large, \Cref{Coro: of Thm1} indicates that for any $\epsilon$, \Cref{Alg: DPO} achieves $\epsilon$-average suboptimal with polynomial trajectories (see the specific form in \Cref{Eq: Sample complexity of Coro1} in \Cref{Appd: B.1}).
If the environment is near stationary and the variation is relatively small, i.e., $(\Delta^{\sqrt{P}}+\Delta^{{\phi}}) \leq {Hd}/{K^2}, (\Delta^P+\Delta^{\pi}) \leq {1}/{\sqrt{K}}$, then the best window size $W$ and the policy restart period $\tau$ are both $K$. This indicates that the agent does not need to take any forgetting rules to handle the variation. Then the $\mathrm{Gap_{Ave}}$ reduces to $\tilde{O}\left(\sqrt{{H^4d^2A}(A+d^2)/K}\right)$, which matches the result under a stationary environment.\footnote{We convert the sample complexity bound under infinite horizon MDPs in \citet{DBLP:conf/iclr/UeharaZS22} to the average dynamic suboptimality gap under episodic MDPs.}  

Furthermore, it is interesting to consider		
a special mildly changing environment, in which the representation $\phi^{\star}$ stays identical and only the state-embedding function $\mu^{\star,k}$ changes over time. The average dynamic suboptimality gap in \Cref{Eq: Basic Bound} reduces to
		 \begin{align*}
		  \tilde{O}\Big(  \underbrace{ \sqrt{\frac{H^4d^2A\left(A+d^2\right)}{W}}
		 +\sqrt{\frac{H^2W^3A}{K^2}}\Delta^{\sqrt{P}}}_{(I)}+ \underbrace{ \frac{H}{\sqrt{\tau}} + \frac{H\tau(\Delta^P+\Delta^{\pi})}{K}}_{(II)} \Big).
		\end{align*}  
		The part $(II)$ is the same as $(II)$ in \Cref{Eq: Basic Bound} and by choosing the best window size of $\overline{W}=H^{\frac{1}{2}}d^{\frac{1}{2}}(A+d^2)^{\frac{1}{4}}K^{\frac{1}{2}}(\Delta^{\sqrt{P}})^{-\frac{1}{2}}$, part $(I)$ becomes
		\begin{align}
		\tilde{O}\left({H^2d^{\frac{3}{4}}A^{\frac{1}{2}}}\left(A+d^2\right)^{\frac{3}{8}}(HK)^{-\frac{1}{4}}\left(\Delta^{\sqrt{P}}\right)^{\frac{1}{4}}\right). \label{Eq: ramark1} 
		\end{align} 
		Compared with the second regime in \Cref{Eq: Thm1-I}, \Cref{Eq: ramark1} is much smaller, benefited from identical representation function $\sphi$. In this way, samples in previous rounds can help to estimate the representation space so that $\overline{W}$ can be larger than $W$ in terms of the order of $K$, which yields efficiency gain compared with changing $\sphi$.   

On the other hand, if the nonstationarity is significantly large, for example, scales linearly with $K$, then for each round, the previous samples cannot help to estimate current best policy. Thus, the best $W$ and $\tau$ are both 1, and the average dynamic suboptimality gap reduces to $\tilde{O}\left(\sqrt{{H^4d^2A}\left(A+d^2\right)}\right)$. This indicates that for a fixed small accuracy $\epsilon \geq 0$, no matter how large the round $K$ is, \Cref{Alg: DPO} can never achieve $\epsilon$-average suboptimality.


\subsection{Proof of \Cref{cor:detailcor} (i.e., Detailed Version of \Cref{Coro: of Thm1})} \label{Appd: B.1}
If variation budgets are known, for different variation budgets regimes, we can tune the hyper-parameters correspondingly to reach the optimality for both the term $(I)$ that contains $W$ and the term $(II)$ that contains $\tau$.\newline
 For the first term $(I)$ in \Cref{Eq: Thm1-Final Bound}, there are two regimes:
 \begin{itemize}
     \item Small variation: $\left(\Delta^{\sqrt{P}}+\Delta^{{\phi}}\right) \leq \frac{Hd}{K^2}$,
    \begin{itemize}
         \item The best window size $W$ is $K$, which means that the variation is pretty mild and the environment is near stationary. In this case, by choosing window size $W=K$, the agent takes no forgetting rules to handle the variation. Then the first term $(I)$ reduces to $\sqrt{\frac{H^4d^2A}{K}\left(A+d^2\right)}$, which matches the result under a stationary environment.\footnote{We convert the regret bound under infinite horizon MDPs in \cite{DBLP:conf/iclr/UeharaZS22} to the average dynamic suboptimality gap under episodic MDP.} 
         \item Then for any $\epsilon \geq 0$, with $HK$ no more than $\tilde{O}\left(\frac{H^5d^2A(A+d^2)}{\epsilon^2}\right)$, term $(I) \leq \epsilon$.
     \end{itemize}
     \item Large variation: $\left(\Delta^{\sqrt{P}}+\Delta^{{\phi}}\right)>\frac{Hd}{K^2}$.
     \begin{itemize}
         \item By choosing the window size $W=H^{\frac{1}{3}}d^{\frac{1}{3}}K^{\frac{1}{3}}\left(\Delta^{\sqrt{P}}+\Delta^{{\phi}}\right)^{-\frac{1}{3}}$, the term $(I)$ reduces to $ {H^{2}d^{\frac{5}{6}}A^{\frac{1}{2}}}\left(A+d^2\right)^{\frac{1}{2}}(HK)^{-\frac{1}{6}}\left(\Delta^{\sqrt{P}}+\Delta^{{\phi}}\right)^{\frac{1}{6}}$.
         \item Since $\left(\Delta^{\sqrt{P}}+\Delta^{{\phi}}\right) \leq 2HK$, there exists $\gamma \leq 1$ s.t. $\left(\Delta^{\sqrt{P}}+\Delta^{{\phi}}\right) \leq (2HK)^\gamma$. Then $(I) \leq 2{H^{2}d^{\frac{5}{6}}A^{\frac{1}{2}}}\left(A+d^2\right)^{\frac{1}{2}}(HK)^{-\frac{1-\gamma}{6}}$. Then for any $\epsilon \geq 0$, if $\gamma \neq 1$, with $HK$ no more than $\tilde{O}\left(\frac{d^{\frac{5}{1-\gamma}}{H^{\frac{12}{1-\gamma}}}A^{\frac{3}{1-\gamma}}(A+d^2)^{\frac{3}{1-\gamma}}}{\epsilon^{\frac{6}{1-\gamma}}}\right)$, term $(I) \leq \epsilon$.
     \end{itemize}
 \end{itemize}

 For the second term $(II)$ in \Cref{Eq: Thm1-Final Bound}, there are three regimes elaborated as follows:
 \begin{itemize}
     \item Small variation: $(\Delta^P+\Delta^{\pi}) \leq \frac{1}{\sqrt{K}}$, 
 \begin{itemize}
         \item The best policy restart period $\tau$ is $K$, which means that the variation is pretty mild and the agent does not need to handle the variation. Then term $(II) \leq \frac{2H}{\sqrt{K}}$. 
         \item Then for any $\epsilon \geq 0$, with $HK$ no more than $\tilde{O}\left(\frac{H^2}{\epsilon^2}\right)$, term $(II) \leq \epsilon$.
     \end{itemize}
     \item Moderate variation: $\frac{1}{\sqrt{K}} < (\Delta^P+\Delta^{\pi}) \leq K$,
     \begin{itemize}
     	\item  The best policy restart period $\tau=K^{\frac{2}{3}}(\Delta^P+\Delta^\pi)^{-\frac{2}{3}}$, and the term $(II)$ reduces to $ 2HK^{-\frac{1}{3}}(\Delta^P+\Delta^\pi)^{\frac{1}{3}}$.
     	\item Since $\left(\Delta^P+\Delta^{\pi}\right) \leq 2HK$, there exists $\gamma \leq 1$ s.t. $\left(\Delta^P+\Delta^{\pi}\right)  \leq (2HK)^\gamma$. Then the term $(II) \leq  4H^{\frac{4}{3}}(HK)^{-\frac{1-\gamma}{3}}$. Then for any $\epsilon \geq 0$, if $\gamma \neq 1$, with $HK$ no more than $\tilde{O}\left(\frac{{H^{\frac{4}{1-\gamma}}}}{\epsilon^{\frac{3}{1-\gamma}}}\right)$, term $(II) \leq \epsilon$.
     \end{itemize}
    \item Large variation: $K< (\Delta^P+\Delta^{\pi})$,
    \begin{itemize}
    	\item The variation budgets scale linearly with $K$, which indicates that the nonstationarity of the environment is significantly large and lasts for the entire rounds. Hence in each round, the previous sample can not help to estimate the current best policy. So the best policy restart period $\tau=1$, and the second term $(II)$ reduces to $H+\frac{H(\Delta^P+\Delta^{\pi})}{K}=O(H)$, which implies that \Cref{Alg: DPO} can never achieve small average dynamic suboptimality gap for any large $K$.
    \end{itemize}    
 \end{itemize}
 In conclusion, the first term is upper bounded by
\begin{equation}
     (I)\leq\left\{
    \begin{aligned}
    &  \sqrt{\frac{H^4d^2A}{K}\left(A+d^2\right)},& \left(\Delta^{\sqrt{P}}+\Delta^{{\phi}}\right) \leq \frac{Hd}{K^2},\\
     & {H^{2}d^{\frac{5}{6}}A^{\frac{1}{2}}}\left(A+d^2\right)^{\frac{1}{2}}(HK)^{-\frac{1}{6}}\left(\Delta^{\sqrt{P}}+\Delta^{{\phi}}\right)^{\frac{1}{6}},& \left(\Delta^{\sqrt{P}}+\Delta^{{\phi}}\right)>\frac{Hd}{K^2}, 
    \end{aligned}
    \right.
\end{equation}
and the second term is upper bounded by
\begin{equation}
    (II)\leq\left\{
    \begin{aligned}
    &  \frac{2H}{\sqrt{K}},&(\Delta^P+\Delta^{\pi}) \leq \frac{1}{\sqrt{K}},\\
     & 2H^{\frac{4}{3}}(HK)^{-\frac{1}{3}}(\Delta^P+\Delta^\pi)^{\frac{1}{3}},& \frac{1}{\sqrt{K}} < (\Delta^P+\Delta^{\pi}) \leq K,\\
     &  H+\frac{H(\Delta^P+\Delta^{\pi})}{K}, & K < (\Delta^P+\Delta^{\pi})
    \end{aligned}
    \right.
\end{equation}
In addition, if the variation budgets are not significantly large, i.e. scale linearly with $K$, for any $\epsilon \geq 0$, \Cref{Alg: DPO} can achieve $\epsilon$-average dynamic suboptimality gap with at most polynomial samples. Specifically, if there exists a constant $\gamma < 1$ such that the variation budgets satisfying $\left(\Delta^P+\Delta^{\pi}\right)  \leq (2HK)^\gamma$ and $\left(\Delta^{\sqrt{P}}+\Delta^{{\phi}}\right) \leq (2HK)^\gamma$, then to achieve $\epsilon$-average dynamic suboptimality gap, i.e., $\mathrm{Gap_{Ave}}(K) \leq \epsilon$, \Cref{Alg: DPO} only needs to collect trajectories no more than 
\begin{equation}
    \left\{
    \begin{aligned}
      &\tilde{O}\left(\frac{H^5d^2A(A+d^2)}{\epsilon^2}\right),\\
   &\hspace{2.3in} \text{if}\quad\left(\Delta^{\sqrt{P}}+\Delta^{{\phi}}\right) \leq \frac{Hd}{K^2}, (\Delta^P+\Delta^{\pi}) \leq \frac{1}{\sqrt{K}};\\
     &\tilde{O}\left(\frac{d^{\frac{5}{1-\gamma}}{H^{\frac{12}{1-\gamma}}}A^{\frac{3}{1-\gamma}}(A+d^2)^{\frac{3}{1-\gamma}}}{\epsilon^{\frac{6}{1-\gamma}}}+\frac{H^2}{\epsilon^2}\right), \\
     & \hspace{1.85in} \text{if} \quad \frac{Hd}{K^2} < \left(\Delta^{\sqrt{P}}+\Delta^{{\phi}}\right) \leq (HK)^\gamma, (\Delta^P+\Delta^{\pi}) \leq \frac{1}{\sqrt{K}};\\
      &\tilde{O}\left(\frac{H^5d^2A(A+d^2)}{\epsilon^2}+\frac{{H^{\frac{4}{1-\gamma}}}}{\epsilon^{\frac{3}{1-\gamma}}}\right), \\
      &\hspace{2in} \text{if} \quad\left(\Delta^{\sqrt{P}}+\Delta^{{\phi}}\right) \leq \frac{Hd}{K^2}, \frac{1}{\sqrt{K}}<(\Delta^P+\Delta^{\pi}) \leq (HK)^\gamma;\\
     &\tilde{O}\left(\frac{d^{\frac{5}{1-\gamma}}{H^{\frac{12}{1-\gamma}}}A^{\frac{3}{1-\gamma}}(A+d^2)^{\frac{3}{1-\gamma}}}{\epsilon^{\frac{6}{1-\gamma}}}\right),\\
     &\hspace{1.5in} \text{if}\quad\frac{Hd}{K^2} < \left(\Delta^{\sqrt{P}}+\Delta^{{\phi}}\right) \leq (HK)^\gamma, \frac{1}{\sqrt{K}}<(\Delta^P+\Delta^{\pi}) \leq (HK)^\gamma.
    \end{aligned} \label{Eq: Sample complexity of Coro1}
    \right.
\end{equation}

\subsection{A Special Case}\label{Appd A.5.: remark1}
In this subsection, we provide a characterization of a special case, where the representation $\phi^{\star}$ stays identical and only the state-embedding function $\mu^{\star,k}$ changes over time.
In such a scenario, the variation budget $\Delta_{[H],[K]}^{{\phi}}=0$ and the average dynamic suboptimality gap bound in \Cref{Eq: Thm1-Final Bound} reduces to
\begin{align*}
    &\mathrm{Gap_{Ave}}(K)\\
	&\quad \leq \tilde{O}\left(\sqrt{\frac{H^4d^2A}{W}\left(A+d^2\right)}+\sqrt{\frac{H^2W^3A}{K^2}}\Delta^{\sqrt{P}}_{[H],[K]}+ \frac{2H}{\sqrt{\tau}}+ \frac{3H\tau}{K} (\Delta_{[H],[K]}^P+\Delta_{[H],[K]}^{\pi})\right)\\
	& \quad \leq \tilde{O}\left({H^{\frac{7}{4}}d^{\frac{3}{4}}A^{\frac{1}{2}}}\left(A+d^2\right)^{\frac{3}{8}}K^{-\frac{1}{4}}\left(\Delta_{[H],[K]}^{\sqrt{P}}\right)^{\frac{1}{4}}+ HK^{-\frac{1}{3}}(\Delta_{[H],[K]}^P+\Delta_{[H],[K]}^\pi)^{\frac{1}{3}}\right),
\end{align*}
  where the last equation follows from the choice of the window side $W=\tilde{O}\left(H^{\frac{1}{2}}d^{\frac{1}{2}}(A+d^2)^{\frac{1}{4}}K^{\frac{1}{2}}\left(\Delta_{[H],[K]}^{\sqrt{P}}\right)^{-\frac{1}{2}}\right)$ and the policy restart period $\tau=\tilde{O}\left(K^{\frac{2}{3}}(\Delta_{[H],[K]}^P+\Delta_{[H],[K]}^\pi)^{-\frac{2}{3}}\right)$ with known variation budgets.


\section{Proof of \Cref{Thm2: Ada-PORTAL} and Detailed Comparison with \citet{DBLP:conf/colt/WeiL21}}
        \subsection{Proof of \Cref{Thm2: Ada-PORTAL}}
	\begin{proof}[Proof of \Cref{Thm2: Ada-PORTAL}]
		Before our formal proof, we first explain several notations on different choices of $W$ and $\tau$ here.
		\begin{itemize}
			\item ($W^\star, \tau^\star$): We denote $W^\star=d^{\frac{1}{3}}H^{\frac{1}{3}}K^{\frac{1}{3}}(\Delta^\phi+\Delta^{\sqrt{P}}+1)^{-\frac{1}{3}}$, and
			$\tau^\star=\tilde{O}\left(K^{\frac{2}{3}}(\Delta^P+\Delta^\pi+1)^{-\frac{2}{3}}\right)$.
			\item ($\overline{W}, \overline{\tau}$): Because $W^\star=d^{\frac{1}{3}}H^{\frac{1}{3}}K^{\frac{1}{3}}(\Delta^\phi+\Delta^{\sqrt{P}}+1)^{-\frac{1}{3}} \leq d^{\frac{1}{3}}H^{\frac{1}{3}}K^{\frac{1}{3}} \leq J_W$ and $\tau^\star=\tilde{O}\left(K^{\frac{2}{3}}(\Delta^P+\Delta^\pi+1)^{-\frac{2}{3}}\right) \leq K^{\frac{2}{3}} \leq J_\tau$. As a result, there exists a $\overline{W} \in \Jc_W$ such that $\overline{W} \leq W^\star \leq 2\overline{W}$ and a $\overline{\tau} \in \Jc_\tau$ such that $\overline{\tau} \leq \tau^\star \leq 2\overline{\tau}$. 
			\item ($W^\dagger, \tau^\dagger$): ($W^\dagger, \tau^\dagger$) denotes the set of best choices of the window size $W$ and the policy restart period $\tau$ in feasible set that maximize $\sum_{i=1}^{\lceil K/M\rceil}R_i(W,\tau)$. 
		\end{itemize}
	Then we can decompose the average dynamic suboptimality gap as 
		\begin{align*}
			\mathrm{Gap_{Ave}} (K) &= \frac{1}{K}\sum_{k \in [K]} \left[V_{P^\star,k}^{\pi^{\star}}-V_{P^\star,k}^{\pi^{k}}\right]\\
			& = \frac{1}{K}\underbrace{\sum_{k =1}^{K} V_{P^\star,k}^{\pi^{\star}}-\sum_{i=1}^{\lceil K/M\rceil} \mathbb{E}\left[R_i(\overline{W},\overline{\tau})\right]}_{(I)}+\frac{1}{K}\underbrace{\sum_{i=1}^{\lceil K/M\rceil} \Eb[R_i(\overline{W},\overline{\tau})] -\sum_{i=1}^{\lceil K/M\rceil} \Eb[R_i(W_i,\tau_i)]}_{(II)},
		\end{align*}
            where the last inequality follows because if $\{\pi^k\}_{k=1}^K$ is the output of \Cref{Alg: ADPO} with the chosen window size $\{W_i\}_{i=1}^{\lceil T/M\rceil}$, $\Eb[R_i(W_i,\tau_i)]=\Eb\left[\sum_{k=(i-1)M+1}^{\min\{iM,K\}}V_1^k\right]=\sum_{k=(i-1)M+1}^{\min\{iM,K\}}V_{P^\star,k}^{\pi^k}$ holds.
            
        We next bound Terms $(I)$ and $(II)$ separately.
         
		\textbf{Term (I):} We derive the following bound:
		\begin{align*}
			&\frac{1}{K}\left\{\sum_{k =1}^{K} V_1^{\pi^{\star,k},k}-\sum_{i=1}^{\lceil K/M\rceil} R_i(\overline{W},\overline{\tau})\right\}\\
			& \qquad \overset{\RM{1}}{\leq}\tilde{O}\left(\sqrt{\frac{H^4d^2A}{\overline{W}}\left(A+d^2\right)}+\sqrt{\frac{H^3dA}{K}\left(A+d^2\right)\overline{W}^2\Delta_{[H],[K]}^{{\phi}}}+\sqrt{\frac{H^2\overline{W}^3A}{K^2}}\Delta^{\sqrt{P}}_{[H],[K]}\right)\\
			& \qquad +\tilde{O}\left( \frac{2H}{\sqrt{\overline{\tau}}}+ \frac{3H\overline{\tau}}{K} (\Delta_{[H],[K]}^P+\Delta_{[H],[K]}^{\pi})\right)\\
             & \qquad \overset{\RM{2}}{\leq} \tilde{O}\left(\sqrt{\frac{H^4d^2A}{W^\star}\left(A+d^2\right)}+\sqrt{\frac{H^3dA}{K}\left(A+d^2\right){W^\star}^2\Delta_{[H],[K]}^{{\phi}}}+\sqrt{\frac{H^2{W^\star}^3A}{K^2}}\Delta^{\sqrt{P}}_{[H],[K]}\right)\\
             & \qquad +\tilde{O}\left( \frac{2H}{\sqrt{\overline{\tau^\star}}}+ \frac{3H\tau^\star}{K} (\Delta_{[H],[K]}^P+\Delta_{[H],[K]}^{\pi})\right)\\
			& \qquad \overset{\RM{3}}{\leq}
			\tilde{O}\left({H^{\frac{11}{6}}d^{\frac{5}{6}}A^{\frac{1}{2}}}\left(A+d^2\right)^{\frac{1}{2}}K^{-\frac{1}{6}}\left(\Delta^{\sqrt{P}}+\Delta^{{\phi}}+1\right)^{\frac{1}{6}}\right)+ \tilde{O}\left(2HK^{-\frac{1}{3}}(\Delta^P+\Delta^\pi+1)^{\frac{1}{3}}\right),
		\end{align*}
  where $\RM{1}$ follows from \Cref{Eq: Thm1-Final Bound}, $\RM{2}$ follows from the definition of $\overline{W}$ and $\RM{3}$ follows from the definition of $W^\star$ at the beginning of the proof.
  
		\textbf{Term (II):} We derive the following bound:
		\begin{align*}
			&\frac{1}{K}\sum_{i=1}^{\lceil K/M\rceil} \Eb\left[R_i(\overline{W},\overline{\tau})\right] -\sum_{i=1}^{\lceil K/M\rceil}  \Eb\left[R_i(W_i,\tau_i)\right]\\
			& \qquad \overset{\RM{1}}{\leq} \frac{1}{K}\sum_{i=1}^{\lceil K/M\rceil} \Eb\left[R_i(W^\dagger,\tau^\dagger)\right] -\sum_{i=1}^{\lceil K/M\rceil} \Eb\left[R_i(W_i,\tau_i)\right]\\
			& \qquad \overset{\RM{2}}{\leq} \tilde{O}(M\sqrt{J\lceil K/M\rceil}/K)\\
			& \qquad =  \tilde{O}(\sqrt{J KM})=\tilde{O}({ H^{\frac{1}{6}}d^{\frac{1}{6}}K^{-\frac{1}{6}}}),
		\end{align*}
		where $\RM{1}$ follows from the definition of $W^\dagger$ and $\RM{2}$ follows from Theorem 3.3 in \cite{DBLP:journals/ftml/BubeckC12} with the adaptation that reward $R_i \leq M$ and the number of iteration is $\lceil K/M\rceil$.
		Then combining the bounds on terms $(I)$ and $(II)$, we have
		\begin{align*}
			\mathrm{Gap_{Ave}}(K) \leq \tilde{O}\left({H^{\frac{11}{6}}d^{\frac{5}{6}}A^{\frac{1}{2}}}\left(A+d^2\right)^{\frac{1}{2}}K^{-\frac{1}{6}}\left(\Delta^{\sqrt{P}}+\Delta^{{\phi}}+1\right)^{\frac{1}{6}}\right)+ \tilde{O}\left(2HK^{-\frac{1}{3}}(\Delta^P+\Delta^\pi+1)^{\frac{1}{3}}\right).
		\end{align*}  
	\end{proof}
 \subsection{Detailed Comparison with \citet{DBLP:conf/colt/WeiL21}}\label{Appd: B.2}
Based on the theoretical results \Cref{Thm1: average dynamic suboptimality gap with known variation} and taking \Cref{Alg: DPO} PORTAL as a base algorithm, we can also use the black-box techniques called MASTER in \citet{DBLP:conf/colt/WeiL21} to handle the unknown variation budgets. But such an approach of MASTER+PORTAL turns out to have a larger average dynamic suboptimality gap than \Cref{Alg: ADPO}. Denote $\Delta={\Delta^{{\phi}}}+\Delta^{\sqrt{P}}+\Delta^{\pi}$. 

To explain, first choose $\tau=k$ and $W=K$ in \Cref{Alg: DPO}. Then \Cref{Alg: DPO} reduces to a base algorithm with the suboptimality gap of $\tilde{O}(\sqrt{{H^4d^2A}(A+d^2)/{K}}+\sqrt{{H^3dA}{K}(A+d^2)}\Delta)$, which satisfies Assumption 1 in \citet{DBLP:conf/colt/WeiL21} with $\rho(t)=\sqrt{{H^4d^2A}(A+d^2)/{t}}$ and $\Delta_{[1,t]}=\sqrt{{H^3dA}{t}\left(A+d^2\right)}\Delta$. Then by Theorem 2 in \citet{DBLP:conf/colt/WeiL21}, the dynamic regret using MASTER+PORTAL can be upper-bounded by $\tilde{O}({H^{\frac{11}{6}}d^{\frac{5}{6}}A^{\frac{1}{2}}}\left(A+d^2\right)^{\frac{1}{2}}K^{\frac{2}{3}}\Delta^{\frac{1}{3}})=\tilde{O}(K^{\frac{5}{6}}\Delta^{\frac{1}{3}})$. Here, we find the order dependency on $d,H,A$ is the same as \Cref{Thm2: Ada-PORTAL} and hence mainly focus on the order dependency on $K$ and $\Delta$ in the following comparison. Such a bound on dynamic regret can be converted to the average dynamic suboptimality gap as $\tilde{O}(K^{-\frac{1}{6}}\Delta^{\frac{1}{3}})$. Then it can obersed that for not too small variation budgets, i.e., $\Delta \geq \tilde{O}(1)$, the order dependency on $\Delta$ is higher than that of \Cref{Alg: ADPO}.
 
Specifically, if we consider the special case when the representation stays identical and denote $\tilde{\Delta}=\Delta^{\sqrt{P}}+\Delta^{\pi}$, then the average dynamic suboptimality gap of MASTER+PORTAL is still $\tilde{O}(K^{-\frac{1}{6}}\tilde{\Delta}^{\frac{1}{3}})$. With small modifications on the parameters, by following the analysis similar to that of \Cref{Thm2: Ada-PORTAL} and \Cref{Appd A.5.: remark1}, we can show that the average dynamic suboptimality gap of \Cref{Alg: ADPO} satisfies $\tilde{O}(K^{-\frac{1}{4}}\tilde{\Delta}^{\frac{1}{4}})$,
which is smaller than MASTER+PORTAL. 

 \section{Auxiliary Lemmas}
In this section, we provide two lemmas that are commonly used for the analysis of MDP problems.

The following lemma \citep{DBLP:conf/nips/DannLB17} is useful to measure the difference between two value functions under two MDPs and reward functions. 
 	\begin{lemma} \label{lemma: Simulation}{\rm(Simulation Lemma).}
 		Suppose $P_1$ and $P_2$ are two MDPs and $r_1$, $r_2$ are the corresponding reward functions. Given a policy $\pi$, we have,  
 		\begin{align*}
 		V_{h,P_1,r_1}^{\pi}&(s_h) - V_{h,P_2,r_2}^{\pi}(s_h)\\
 		&= \sum_{\ph=h}^H  \mathop{\Eb}_{s_\ph \sim (P_2,\pi) \atop a_\ph \sim \pi}\left[r_1(s_\ph,a_\ph) - r_2(s_\ph,a_\ph) + (P_{1,\ph} - P_{2,\ph})V^{\pi}_{\ph+1,P_1,r}(s_\ph,a_\ph)|s_h\right]\\
 		& = \sum_{\ph=h}^H  \mathop{\Eb}_{s_\ph \sim (P_1,\pi) \atop a_\ph \sim \pi}\left[r_1(s_\ph,a_\ph) - r_2(s_\ph,a_\ph) + (P_{1,\ph} - P_{2,\ph})V^{\pi}_{\ph+1,P_2,r}(s_\ph,a_\ph)|s_h\right].
 		\end{align*}
 	\end{lemma}
  The following lemma is a standard inequality in the regret analysis for linear MDPs in reinforcement learning (see Lemma G.2 in \citet{DBLP:conf/nips/AgarwalKKS20} and Lemma 10 in \citet{DBLP:conf/iclr/UeharaZS22}).	
	\begin{lemma}[Elliptical Potential Lemma] \label{Lemma: Elliptical_potential} 
		Consider a sequence of $d \times d$ positive semidefinite matrices $X_1, \dots, X_N$ with ${\rm tr}(X_n) \leq 1$ for all $n \in [N]$. Define $M_0=\lambda_0 I$ and $M_n=M_{n-1}+X_n$. Then the following bound holds:
		\begin{equation*}
		\sum_{n=1}^N {\rm tr}\left(X_nM_{n-1}^{-1}\right) \leq 2\log \det(M_N)- 2\log \det(M_0) \leq 2d\log\left(1+\frac{N}{d\lambda_0}\right).
		\end{equation*}
	\end{lemma}


\end{document}